%% file: neurips_2026.tex
\documentclass{article}

 \usepackage[preprint]{neurips_2026}

\usepackage[utf8]{inputenc} 
\usepackage[T1]{fontenc}    
\usepackage{hyperref}       
\usepackage{url}            
\usepackage{booktabs}       
\usepackage{amsfonts}       
\usepackage{nicefrac}       
\usepackage{microtype}      
\usepackage{xcolor}         

\usepackage{caption}
\usepackage{algorithm}
\usepackage{algpseudocode}
\usepackage{booktabs}
\usepackage{amssymb}            
\usepackage{mathtools}          
\usepackage{mathrsfs}           
\usepackage{graphicx}           
\usepackage{subcaption}         
\usepackage[space]{grffile}     
\usepackage{url}                
\usepackage{lipsum}             
\usepackage{amsthm}
\usepackage{enumitem}
\usepackage{our_macros}           
\usepackage{multirow} 

\usepackage[dvipsnames]{xcolor}
\usepackage{wrapfig} 
\usepackage{makecell}
\newif\ifrlc
\rlcfalse

\title{Switching Successor Measures for\\ Hierarchical Zero-shot Reinforcement Learning}

\author{%
  Stefan Stojanovic \\
  KTH, Stockholm, Sweden\\
  \texttt{stesto@kth.se} \\
  \And
  Alexandre Proutiere \\
  KTH, Digital Futures, Stockholm, Sweden\\
  \texttt{alepro@kth.se} \\
}

\begin{document}

\maketitle

\begin{abstract}

Hierarchical reinforcement learning can improve generalization by decomposing long-horizon decision-making into simpler subproblems. However, existing approaches often rely on restrictive design choices, such as fixed temporal abstractions or goal-conditioned objectives, which largely confine them to goal-reaching tasks and limit their applicability to general reward functions. In this paper, we introduce \emph{switching successor measures}, an extension of successor measures that enables hierarchical control in zero-shot reinforcement learning without additional supervision, fixed horizons, or manually designed subgoals. We show that switching successor measures arise naturally from classical successor measures while preserving their underlying structure. Building on this result, we propose FB \(\pi\)-Switch, an algorithm that extracts both a high-level subgoal-selection policy and a low-level control policy directly from forward-backward (FB) representations, allowing hierarchical behavior to emerge from a single learned representation. Experiments on both goal-conditioned and general reward-based tasks show that FB \(\pi\)-Switch improves over non-hierarchical baselines and matches state-of-the-art hierarchical methods in goal-conditioned settings. These results demonstrate that structured successor representations provide a flexible foundation for hierarchical zero-shot reinforcement learning beyond goal-reaching tasks. Our project website is available at: \url{https://stestokth.github.io/switching-successors/}.

\end{abstract}

\input{sections/1.introduction}

\input{sections/2_related_work}

\input{sections/3_preliminaries}
\input{sections/4_switching}
\input{sections/5_FB_pi-Switch}

\input{sections/6_experiments}

\input{sections/7_conclusion}

\newpage

\section*{Acknowledgments}
We thank Bastien Dubail for providing helpful feedback on the proof of the main theorem. This research was supported by the Wallenberg AI, Autonomous Systems and Software Program (WASP) funded by the Knut and Alice Wallenberg Foundation, the Swedish Research Council (VR), and Digital Futures. The computations were enabled by resources provided by the National Academic Infrastructure for Supercomputing in Sweden (NAISS), partially funded by the Swedish Research Council through grant agreement no. 2022-06725. 

\bibliographystyle{plain} 
\bibliography{main}

\newpage
\appendix

\input{sections/8_appA}
\clearpage
\newpage

\input{sections/9_appB}

\clearpage
\newpage

\input{sections/10_appC}
\clearpage
\newpage

\input{sections/11_appD}
\clearpage
\newpage



\end{document}

%% file: sections/1.introduction.tex
\section{Introduction}\label{sec:introduction}

A central challenge in reinforcement learning (RL) is to design methods that are flexible and generalize well across a wide range of reward functions, enabling agents to adapt to new tasks without retraining from scratch. Goal-conditioned reinforcement learning (GCRL) has shown that structuring tasks around goals can significantly improve generalization \citep{chane2021goal, eysenbach2022contrastive, wang20251000}. Hierarchical approaches further extend this idea by decomposing long-horizon problems into sequences of shorter subproblems. For instance, methods such as HIQL \citep{park2023hiql} demonstrate that conditioning on intermediate subgoals, rather than only on global task specifications, can substantially improve performance. However, these approaches rely on strong design choices. In particular, HIQL enforces locality through fixed subgoal horizons, rather than allowing such locality to emerge from the learning process itself.

Moreover, these methods are primarily studied in goal-conditioned settings, where tasks reduce to reaching specific states. In contrast, general RL involves arbitrary reward functions that cannot always be expressed as goal-reaching objectives. A prominent line of work addresses this setting through forward-backward (FB) representations of successor measures, which aim to represent value functions and policies in a shared embedding space \citep{touati2021learning,touati2022does,tirinzoni2025zero}. Despite their promise, these approaches rely on a strong global factorization assumption of the form \(F(s,a,z)^\top B(g)\) for successor measures, which must simultaneously capture the dynamics of the environment and all downstream reward functions. Recent works have shown that such global low-dimensional structure can be difficult to learn and may be too restrictive in complex environments \citep{cetinfiner,agarwal2024proto,frans2024unsupervised,dubail2025shift}. This suggests that successor representations may be most reliable locally: they can capture short-range transitions effectively, while their accuracy may deteriorate as the effective planning horizon increases. Supporting this view, recent work on compositional planning \citep{farebrother2026compositional} shows that combining local behaviors can improve generalization, although it requires an additional generative model. Finally, several recent methods treat representation learning as a first stage in zero-shot RL pipelines, with policy learning performed separately using standard offline algorithms \citep{ghosh2023reinforcement, park2024foundation, frans2024unsupervised}. While effective, this separation does not fully exploit the structure encoded in successor representations for joint planning and control.

Together, these observations highlight an important gap: there is currently no unified approach that leverages successor representations to derive hierarchical policies directly from the learned representation, while retaining the ability to generalize across arbitrary reward functions. This motivates our central question: \emph{Can successor representations serve as a unified foundation from which both reward embeddings and hierarchical control emerge in zero-shot reinforcement learning?}

We address this question by introducing \emph{switching successor measures} that allow us to derive hierarchical policies directly from successor representations, rather than introducing hierarchy as a separate supervised object. Concretely, in each state \(s\), the agent first selects a latent subgoal \(w\) using a high-level policy \(\pi^h(\cdot |s,r)\) conditioned on the reward function \(r\), and then acts according to a low-level policy \(\pi^\ell(\cdot |s,w)\) conditioned on the selected subgoal. Our main contributions are:

(a) We introduce \emph{switching successor measures} as a principled framework for extracting hierarchical structure from successor-based representations. A key insight is that the switching successor measures needed for high-level planning can be derived directly from a single classical successor measure. We establish this remarkable connection in Theorem~\ref{thm:SM_pi_to_pi}, showing that hierarchy is already implicitly encoded in standard successor representations and can be recovered without additional learning.

(b) We propose \emph{FB \(\pi\)-Switch}, a hierarchical zero-shot RL algorithm in which both the high-level subgoal-selection policy and the low-level control policy are derived from FB successor representations. The method follows a three-stage training procedure; in particular, the high-level policy learning stage can be incorporated into existing FB algorithms.

(c) We empirically evaluate FB \(\pi\)-Switch on both goal-conditioned and more general reward-based tasks. The method achieves strong performance without explicit locality-inducing hyperparameters, and consistently outperforms other baselines, including FB with an induced high-level policy. 

%% file: sections/2_related_work.tex
\section{Related work}
\label{sec:related_work}

\textbf{Successor representations and skill-based methods.}
Successor features \citep{barreto2017successor} provide a temporal-difference framework for zero-shot learning by decoupling dynamics from rewards. This idea has been extended to successor measures via FB representations, enabling zero-shot adaptation to new tasks \citep{touati2022does}. Subsequent work improves these models through online embedding adaptation \citep{sikchi2025fast} and more stable training objectives \citep{zheng2025towards}. Alternative parameterizations have also been explored, e.g., intention-conditioned value functions (ICVFs) \citep{ghosh2023reinforcement} and proto successor measures \citep{agarwal2024proto}, which relax the linear reward assumption but require a discrete policy codebook. 
A related line of work learns reusable skills without explicit successor structure, including unsupervised skill discovery in online \citep{eysenbach2018diversity} and offline settings \citep{ajay2020opal}, as well as methods that improve skill transfer \citep{frans2024unsupervised}. More broadly, GCRL methods such as UVFA \citep{schaul2015universal} and contrastive approaches \citep{eysenbach2022contrastive} learn policies conditioned on goals and can be viewed as a special case of successor-based formulations. However, these approaches typically rely on conditioning on fixed goals and do not explicitly capture hierarchical composition.

\textbf{Hierarchical reinforcement learning.}
Hierarchical RL studies temporal abstraction, such as options, skills, and goals. The option framework \citep{sutton1999between} provides a principled formulation for hierarchical RL, with extensions including gradient-based methods \citep{comanici2010optimal} and actor-critic architectures \citep{bacon2017option}. Composing pretrained policies has also been explored in prior work \citep{liaw2017composing, peng2017deeploco}, while subgoal-based methods such as HIRO \citep{nachum2018data} delegate goal generation to a high-level policy under a fixed reward function. 
Subgoal-based approaches have been widely studied in both online \citep{huang2019mapping, chane2021goal, gurtler2021hierarchical} and offline goal-conditioned RL. In particular, HIQL \citep{park2023hiql} extends hierarchical RL to the offline regime, with subsequent work improving subgoal generation \citep{huang2024goal, chen2024plandq, park2025temporal} and value estimation \citep{keconservative, giammarino2025physics}. 
Hierarchical structure can also be induced in skill-based frameworks, e.g., by optimizing for midpoints in latent space \citep{park2024foundation}, or by combining successor representations with generative models and planning \citep{farebrother2026compositional}. 

%% file: sections/3_preliminaries.tex
\section{Preliminaries} 
\label{sec:preliminaries}

\textbf{Problem setting.}
We consider an infinite-horizon Markov decision process $\mathcal{M}=(\mathcal{S}, \mathcal{A}, P, \gamma)$, where $\mathcal{S}$ and $\mathcal{A}$ denote the state and action spaces, $P$ is the transition kernel, and $\gamma \in (0,1)$ is the discount factor. Given a reward function $r: \mathcal{S} \to \mathbb{R}$, we define the Q-function of a policy, represented by a distribution over actions $\pi(\cdot |s)$ in each state $s$, by $Q^\pi(s,a;r) = \mathbb{E}_\pi \left[ \sum_{t=0}^{\infty} \gamma^t r(s_t) \vert s_0=s, a_0 = a \right]$. A policy $\pi$ is optimal for reward $r$ if it maximizes the corresponding Q-function for all states and actions. We focus on the unsupervised setting, where the reward function is not specified a priori, and the objective is to learn optimal policies for arbitrary downstream reward functions. We assume access to an offline dataset $\mathcal{D}$ of transitions $(s_t,a_t,s_{t+1})$, collected under an unknown behavior policy. The states in $\mathcal{D}$ are distributed according to an unknown marginal $\rho$. The dimensionality of all learned representations is denoted by $d$. Finally, $\mathcal{Z}\subseteq \mathbb{R}^d$ denotes the latent intent space, from which $z$ is sampled; the sampling distributions may vary across different stages of our algorithm.

\textbf{Successor measure.} Define the \emph{successor measure} under policy $\pi$ \citep{blier2021learning} and the corresponding \emph{state-successor measure}\footnote{$M^\pi_s(s')$ is called an intention-conditioned value function in \cite{ghosh2023reinforcement}; we avoid this terminology since it corresponds to a value function only under indicator rewards supported solely at the goal state.} as:
\begin{align*}
    M^\pi_{s,a}(s') = \mathbb{E}_{\pi} \left[ \sum_{t=0}^\infty \gamma^t \mathbf{1}\{s_t = s' \} \vert s_0 = s, a_0 = a\right],\quad M^\pi_s(s')  = \mathbb{E}_{a\sim \pi(\cdot \vert s)} M_{s,a}^\pi(s').
\end{align*}
For a state-dependent reward function $r$, the $(r,\pi)$-value functions can be written as $Q^\pi(s,a; r) = \langle M^\pi_{s,a}, r \rangle$ and $V^\pi(s; r) = \langle M^\pi_s, r\rangle$, where the inner product $\langle \cdot , \cdot \rangle$ is taken over states. In particular, for indicator rewards $r_{g}$, we have $V^\pi(s; r_{g}) = M^\pi_s(g)$ for all $s$. For brevity, we sometimes use a single index $g$ to denote the indicator reward $r_g$; in this case, we write $V^\pi(s;g)$ in place of $V^\pi(s;r_g)$, and use $V^\star(s;g)$ as shorthand for $V^{\pi_g}(s;r_g)$. Finally, let $H^\pi_s(w)$ denote the hitting time of state $w$ under policy $\pi$ and starting from $s$, defined by $H^\pi_s(w) = \min \{ t \ge 0 : s_t=w \vert s_0 =s, \pi\}$.

\textbf{Learning successor measures.}
A standard approach to learning successor representations is the FB factorization of \cite{touati2022does}, which models the successor measures of a parametrized family $(\pi_z)_{z\in\cal Z}$ of policies via the factorization $M^{\pi_z}_{s,a}(s') = F(s,a,z)^\top B(s')\rho(s')$. Such representations have the key property that, if the policy $\pi_z$ satisfies for every state $s$, $\pi_z(s) \in \argmax_a F(s,a,z)^\top z$, then $\pi_z$ is optimal for the reward function $r_z(s) = \langle B(s), z\rangle$. The representations are obtained by minimizing a squared Bellman error with respect to the measure $\rho$:
\begin{align}
    \min_{F,B}\mathbb{E}_{\substack{
    (s_t,a_t,s_{t+1})\sim \mathcal{D} \\
    s'\sim \rho, z \sim \mathcal{Z}
    }} \left[ \frac{ \mathbf{1}\{ s_{t}=s' \} }{\rho(s')} + \gamma  \overline{F}(s_{t+1},\pi_z(s_{t+1}), z)^\top \overline{B}(s')  - F(s_t,a_t,z)^\top B(s') \right]^2,
    \label{eq:FB_standard_loss}
\end{align}
where we use $\overline{F}, \overline{B}$ to denote that gradients are not propagated through the parameters of these networks (they are target networks). Thus, each training iteration of the FB algorithm consists of two intertwined steps: updating the networks $F,B$, and updating the policy $\pi_z$. In practice, an additional orthonormalization loss is imposed on $B$ to fix the scale ambiguity. Importantly, the objective can be evaluated without explicit estimation of the marginal $\rho(s')$ (Appendix A.2).

An alternative approach learns successor-related value functions via offline regression using a direction-weighted expectile loss. \cite{ghosh2023reinforcement} extends the IQL objective of \cite{kostrikov2021offline} to the intention-conditioned setting by minimizing over latent intentions $\mathcal{Z}$:
\begin{align}
    \min_\theta \; \mathbb{E}_{\substack{
    (s_t,s_{t+1})\sim \mathcal{D} \\
    s'\sim \rho, z \sim \mathcal{Z}
    }} \Big[
    \left| \tau - \mathbf{1}\{\Delta < 0\} \right|
    \left(
    \mathbf{1}\{ s_{t}=s' \}
    + \gamma\, V^{\pi_z}_{\overline{\theta}} (s_{t+1}; s')
    - V^{\pi_z}_\theta(s_t; s')
    \right)^2
    \Big],
    \label{eq:IQL_intention}
\end{align}
where $\tau \in [\tfrac{1}{2},1]$ is a fixed expectile parameter and $\Delta = r_z(s_t) + \gamma V^{\pi_z}_\theta(s_{t+1}; r_z) - V^{\pi_z}_\theta(s_t; r_z)$. Here, $r_z$ is an intent-specific reward function and $\pi_z(s) \in \argmax_a r_z(s) + \gamma \mathbb{E}_{s' \sim P(\cdot \vert s,a)} V_\theta^{\pi_z}(s'; r_z)$ the corresponding policy. The weighting term depends on the sign of $\Delta$, emphasizing transitions with positive temporal-difference residuals under $\pi_z$; the squared term corresponds to the standard TD error for policy $\pi_z$ and a target state $s'$. Note that \cite{ghosh2023reinforcement} does not use these value estimates for policy learning; instead, they serve purely for representation learning, after which a separate offline RL algorithm is applied (see Appendix A.4 for more details).

\textbf{Policy extraction.} We use advantage-weighted regression (AWR) \citep{peng2019advantage} to train all policies, following \cite{park2023hiql}. Given an advantage function $A(s,a)$, AWR learns a policy by maximizing:
\begin{align}
    \max_\pi \mathbb{E}_{(s,a)\sim \mathcal{D}} \left[ \exp\big(\beta A(s,a)\big)\,\log \pi(a \mid s) \right],    
    \label{eq:AWR_loss}
\end{align}
where $\beta > 0$ is an inverse temperature parameter controlling the sharpness of the weighting.
In our setting, the advantage function is defined in terms of the learned successor-based value functions, and the policy may condition on $(s,z)$ and output either primitive actions or latent variables.

%% file: sections/4_switching.tex
\section{Switching successor measures}
\label{sec:switching_successor_measures}

We introduce and study {\it switching successor measures}, a central object for hierarchical zero-shot reinforcement learning. Our starting point is the hierarchical goal-conditioned framework introduced in HIQL \citep{park2023hiql}. In the hierarchical stage of HIQL, the high-level policy is learned by maximizing a subgoal-based advantage objective. The latter quantifies the benefit of moving to state $w$ (the subgoal) within the first $k$ time steps when starting at state $s$ and under reward function $r_g$: 
\begin{align}
    {A}(s, w, g) = \textcolor{orange}{\gamma^{k}} V^\star(w; g) \textcolor{orange}{ + \sum_{t= 0}^{k - 1} \gamma^{t} r_g(s_{t}) } - V^\star(s; g),
    \label{eq:Ah_HIQL}
\end{align}
where the trajectory $(s_0=s,s_1,\ldots,s_k=w)$ is drawn from the offline dataset. The terms highlighted in \textcolor{orange}{orange} are omitted in HIQL under the goal structure, as discussed in \cite{park2023hiql}. The framework relies on the equivalence $V^\pi(s; r_{g}) = M^\pi_s(g)$, valid in the goal-conditioned setting. 

\newcommand{\biaskfigure}{
   \begin{wrapfigure}{r}{0.24\textwidth}
        \centering
        \vspace{-10pt}
        \includegraphics[width=0.24\textwidth]{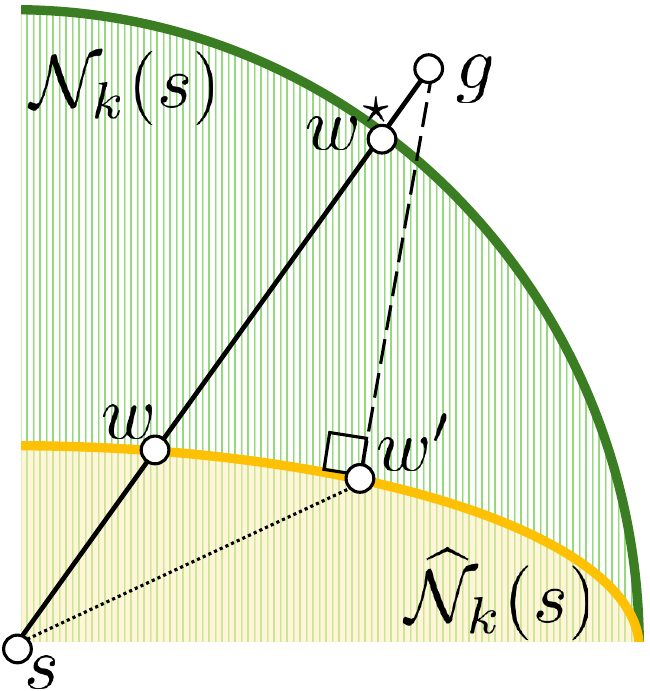}
        \caption{For $w,w'\in \widehat{\mathcal{N}}_k(s)$, it may hold that $V^\star(w';g)>V^\star(w;g)$ even though $w$ is the optimal subgoal.} 
        \vspace{-20pt}
        \label{fig:bias_k}
    \end{wrapfigure}
}

\ifrlc
\else
   \biaskfigure
\fi

First we highlight two issues with the objective in \eqref{eq:Ah_HIQL}:

(i) \textbf{Short-term return.}  In the goal-conditioned framework of HIQL, a constant per-step reward, e.g., $0$ or $-1$, allows the short-term return term to be omitted. More generally, however, the short-term return $\sum_{t=0}^{k - 1} \gamma^{t} r_g(s_{t})$ depends on both the reward and the states, and cannot be treated as a constant when $r_g$ is replaced with a general reward function.

(ii) \textbf{Bias induced by the subgoal horizon $k$.} Defining the subgoal \(w\) via the advantage \(A(s,w,g)\) implicitly assumes that it is reached in exactly \(k\) steps under the optimal policy. In practice, however, offline datasets rarely contain sufficiently many optimal trajectories. This creates a mismatch between the true \(k\)-step reachable set and the one observed in the data, which introduces bias into value estimation. This issue is further exacerbated in anisotropic environments, where progress depends strongly on direction. In such settings, a fixed horizon \(k\) may correspond to subgoals that are unevenly distant in terms of true task progress, even when they are all \(k\)-step reachable. 
Figure~\ref{fig:bias_k} illustrates this phenomenon. There, \({\cal N}_k(s)\) denotes the set of states reachable from \(s\) in \(k\) steps, while \(\widehat{\cal N}_k(s)\) denotes the corresponding set inferred from the dataset. The subgoal \(w' \in \widehat{\cal N}_k(s)\), selected under HIQL, is suboptimal, as it does not lie on the optimal trajectory from \(s\) to the goal \(g\). A better choice would be \(w\), which lies along the optimal path. If the dataset were exhaustive and contained all optimal trajectories, the optimal subgoal would instead be \(w^\star\).

Our framework resolves both issues by (i) accounting for the short-term return via our {switching} successor measures, and (ii) replacing fixed-horizon subgoal selection with hitting-time switching.

\textbf{Addressing short-term return.}
To design high-level policies, we aim to capture the effect of temporarily following a subgoal-conditioned policy and then switching back to a globally efficient policy. Concretely, starting from state $s$ and given an intermediate subgoal $w$, the agent follows $\pi_w$ until time $k$, and thereafter follows $\pi$, meant to be an efficient policy for the given reward function $r$. To optimize subgoal selection, a natural objective is the $k$-step advantage function:
\begin{align}
    A_{s}^{\pi_w|_k \to \pi} (r) := \mathbb{E}_{\pi_w} \Big[ \gamma^{k} V^{\pi}(s_{k}; r) + \sum_{t = 0}^{k - 1} \gamma^{t} r(s_{t})\ \vert s_0=s \Big] -  V^{\pi}(s; r).
    \label{eq:Ak_short_term}
\end{align}
Instead of working with intermediate rewards and the state $s_{k}$ along the trajectory, we can express the $k$-step advantage using successor measures. Define the $k$-step switching successor measure as:
\begin{align}
    M^{\pi_{w}|_k \to \pi}_s(s') &:= \sum_{t=0}^{k-1} \gamma^{t} \mathbb{P}_{\pi_{w}}(s_t=s'\vert s_0=s) + \gamma^{k} \sum_{\tilde{s} \in \mathcal{S}} \mathbb{P}_{\pi_{w}} (s_{k}=\tilde{s} \vert s_0=s) M^{\pi}_{\tilde{s}} (s').
    \label{eq:M_change_k}
\end{align}
When $\pi_{w} = \pi$, this reduces to the standard successor measure, $M^{\pi|_k \to \pi}(\cdot) \equiv M^{\pi}(\cdot)$. The first term in~\eqref{eq:M_change_k} corresponds to the $k$-step truncated successor measure under $\pi_w$, which we denote by $M^{\pi_{w}\vert_k}_s(s') $. Moreover, $A_{s}^{\pi_w|_k \to \pi}(r)=\langle M^{\pi_{w}|_k \to \pi}_s -M^{\pi}_s, r\rangle$, and in the goal-conditioned setting, this reduces to $A_{s}^{\pi_w|_k \to \pi}(r_g) =   M^{\pi_{w}|_k \to \pi}_s(g) - M^{\pi}_s (g)$.

\newcommand{\bigfiguremain}{
    \begin{figure}[t!]
        \centering
        \begin{subfigure}[t]{0.24\textwidth}
            \centering
            \includegraphics[height=1\linewidth]{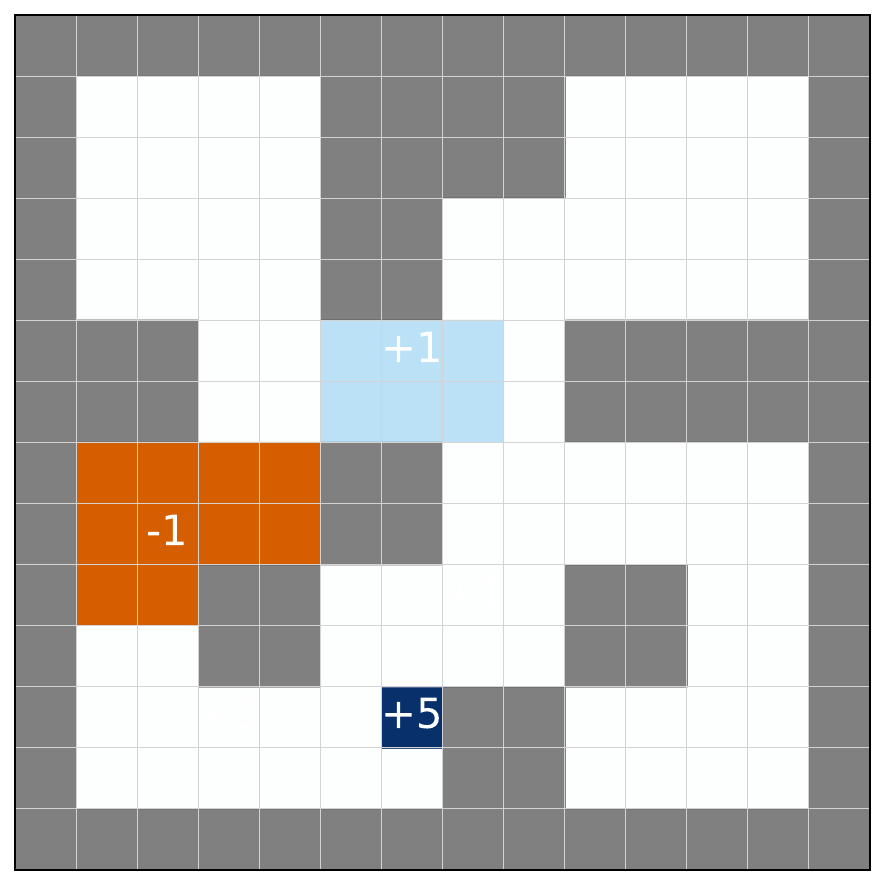}
            \caption{$r(\textcolor{magenta}{s})$}
        \end{subfigure}
        \begin{subfigure}[t]{0.24\textwidth}
            \centering
            \includegraphics[height=1\linewidth]{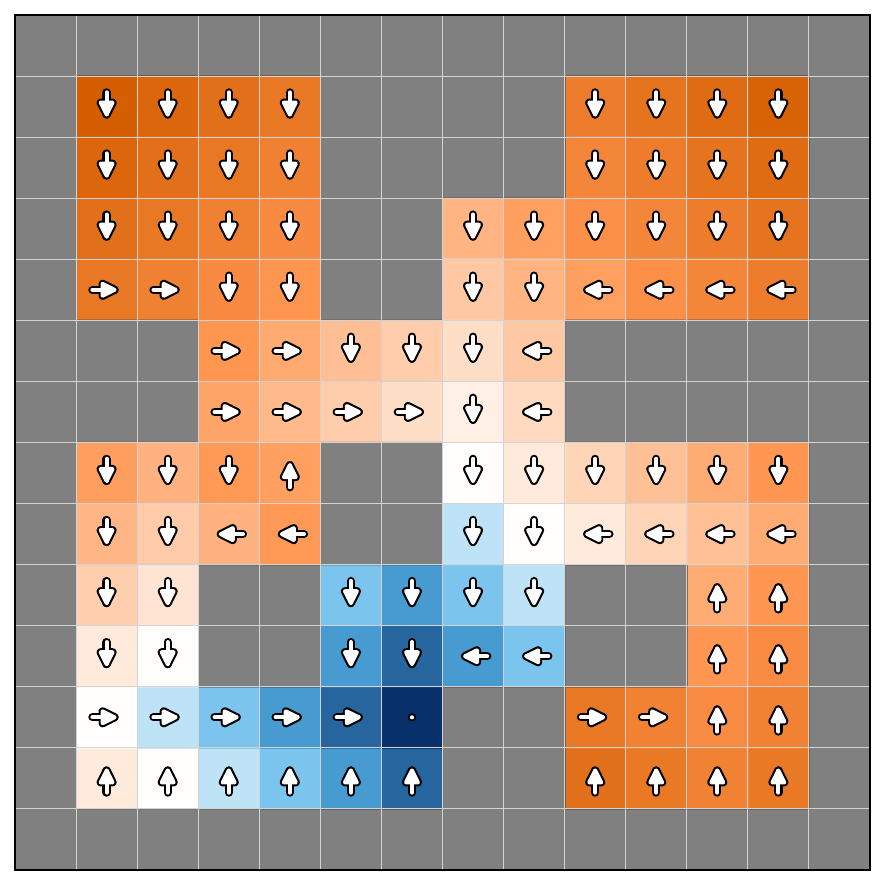}
            \caption{$V^\star(\textcolor{magenta}{s};r)$}
        \end{subfigure}
        \begin{subfigure}[t]{0.24\textwidth}
            \centering
            \includegraphics[height=1\linewidth]{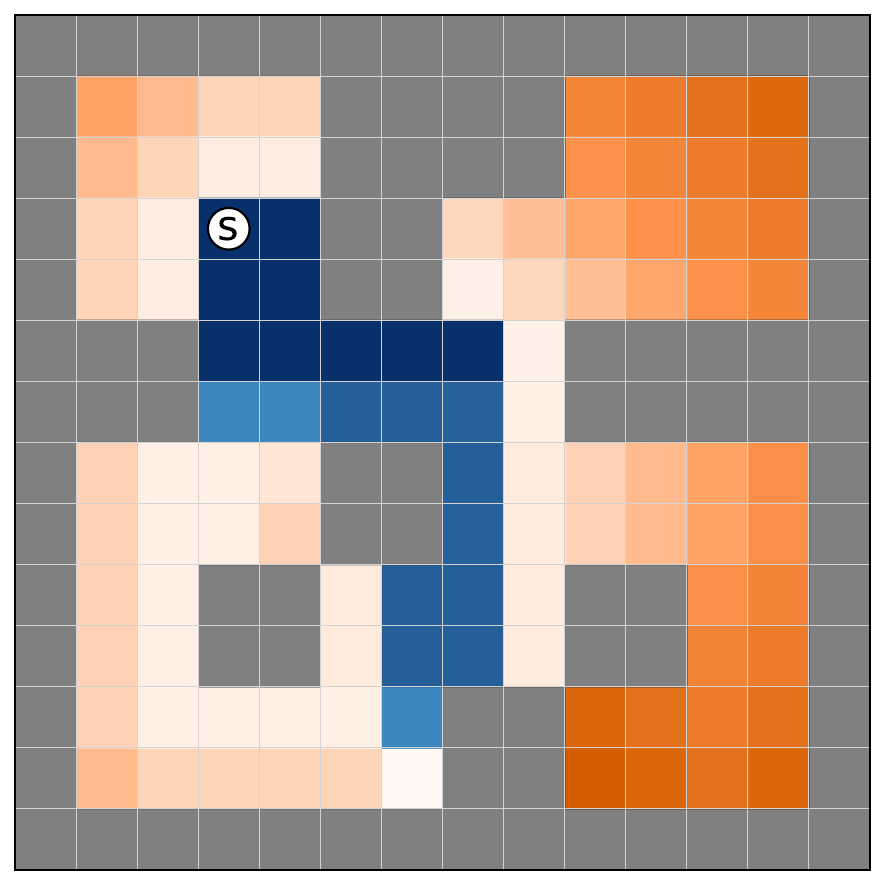}
            \caption{$A_s^{\pi_{\textcolor{magenta}{w}} \to \pi^\star}(r)$  }
        \end{subfigure}
        \begin{subfigure}[t]{0.24\textwidth}
            \centering
            \includegraphics[height=1\linewidth]{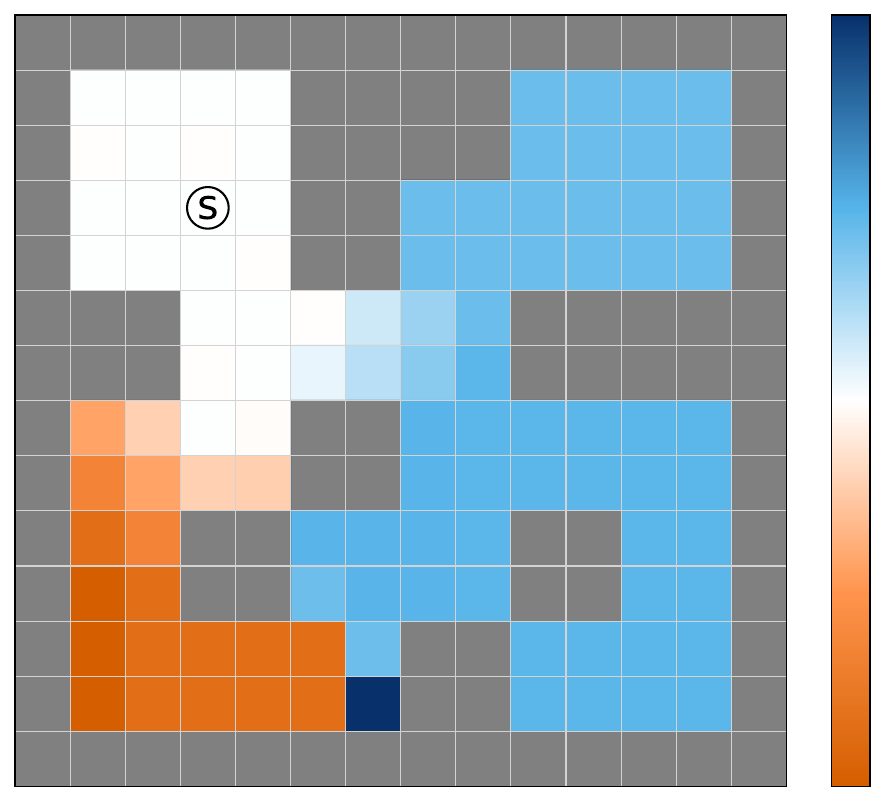}
            \captionsetup{width=1.26\linewidth}{\vspace{-0.47cm}}
             \caption{$ \mathbb{E}   A_s^{\pi_{\textcolor{magenta}{w}}\vert_k \to \pi^\star}(r) \vert_{k< H_s^{\pi_{\textcolor{magenta}{w}}}  (\textcolor{magenta}{w})}   $ }
        \end{subfigure}
        \caption{Depicted quantities as functions of the \textcolor{magenta}{magenta} variables: (a) reward function with positive and negative regions, and a goal state with reward $+5$; (b) optimal value function corresponding to this reward, shown with one choice of optimal actions. Note that after passing through reward-bearing regions, the value of downstream states depends only on the final reward, not on the rewards encountered earlier; (c) switching advantage function defined in~\eqref{eq:def_A_random_switch}, which highlights the states on the optimal trajectory induced by the reward; (d) contributions to the switching advantage arising from rewards collected \textit{before} switching (reaching $w$). 
        }
        \label{fig:bigfigure}
        \vspace{-13pt}
    \end{figure}
}

\ifrlc
\else
    \bigfiguremain
\fi

\textbf{Addressing subgoal horizon.} Our notion of $k$-step advantage considers a fixed number of steps and there is no reason why, starting from $s_0=s$ under $\pi_w$, we would reach $w$ after $k$ steps. To build intuition, however, consider the idealized case where $s_{k}=w$ deterministically under $\pi_w$ when $s_0=s$. In this very specific case, we have $M^{\pi_{w}\vert_k}_s(s')  = M^{\pi_{w}}_s(s') - \gamma^k M^{\pi_w}_w(s')$, which enables us to rewrite~\eqref{eq:Ak_short_term} purely in terms of value functions:
\begin{align*}
      A_s^{\pi_w|_k \to \pi}(r) =  V^{\pi_w}(s;r) + \gamma^k \big( V^\pi(w; r) - V^{\pi_w}(w;r) \big)  -  V^\pi(s; r).
\end{align*}
Next, we relax the assumption $s_{k}=w$ and study policies that switch upon reaching the subgoal $w$. More precisely, consider a policy that first follows $\pi_w$ until reaching state $w$, and then switches to $\pi$. Denote by $M^{\pi_w \to \pi}$ the successor measure obtained by plugging\footnote{Definition \eqref{eq:SM_switch_agnostic} is equivalent to $M^{\pi_{w} \to \pi}_s(s') =  \mathbb{E} M^{\pi_{w}|_{H_s^{\pi_w}(w)} \to \pi}_s(s')$, but we use the former for clarity.} $k=H_s^{\pi_w}(w)$ into~\eqref{eq:M_change_k}:
\begin{align}
    M^{\pi_{w} \to \pi}_s(s') &:=  \mathbb{E}\left[ M^{\pi_{w}|_k \to \pi}_s(s') \mathbf{1}\{ k=H_s^{\pi_w}(w)\} \right]. 
    \label{eq:SM_switch_agnostic}
\end{align}
The following theorem relates switching and standard successor measures:
\begin{theorem}
    For any $s,w,s'\in\mathcal{S}$ and policy $\pi$, the following identity holds:\\[0.2cm]
    \hspace*{0.18\textwidth}
    $ \displaystyle  M^{\pi_{w} \to \pi}_s(s') = M^{\pi_w}_s(s') + \frac{M^{\pi_w}_s(w)}{M^{\pi_w}_w(w)} \Big( M^{\pi}_w(s') -  M^{\pi_w}_w(s') \Big).
    $
    \label{thm:SM_pi_to_pi}
\end{theorem}
The proof is deferred to Appendix B.1. We next formalize the advantage of switching from $\pi_w$ to $\pi$ at $w$. For a given reward $r$, the \emph{switching advantage function} is defined as:
\begin{align}
    A^{\pi_{w} \to \pi}_s(r) &:=  \mathbb{E}\left[ A^{\pi_{w}|_k \to \pi}_s(r) \mathbf{1}\{ k=H_s^{\pi_w}(w)\} \right].
    \label{eq:def_A_random_switch}
\end{align}

\begin{corollary}
\ifrlc
\else
For any $s,w\in\mathcal{S}$, any policy $\pi$, and any reward function $r$, the following identity holds:\\[0.2cm]
\fi
    \hspace*{0.1\textwidth}
    $ \displaystyle A_s^{\pi_w \to \pi}(r) =  V^{\pi_{w}}(s;r) +  \frac{M^{\pi_{w}}_s(w)}{M^{\pi_{w}}_w(w)} \Big( V^\pi(w;r) - V^{\pi_{w}}(w;r)\Big) - V^\pi(s;r).
    $
    \label{cor:main_hierarchical_SM}
\end{corollary}
The switching advantage function serves as the objective for designing the high-level policy. Given a reward function $r$ and an efficient base policy $\pi$, we define a high-level policy that selects subgoals $w$ by maximizing the switching advantage function $A_s^{\pi_w \to \pi}(r)$ over $w$. The main quantities in this section are illustrated in Figure~\ref{fig:bigfigure} for a discrete maze environment; see Appendix C for details.

%% file: sections/5_FB_pi-Switch.tex
\section{FB \texorpdfstring{$\pi$}{pi}-Switch: Hierarchical learning with switching successor measures}
\label{sec:hil-switch}
In this section, we present a practical offline learning algorithm that leverages the hierarchical successor measures introduced in the previous section. Training of \textbf{FB $\pi$-Switch}, illustrated in Figure~\ref{fig:framework}, consists of three stages: (1) jointly learning state-successor measure and representations, followed by (2) high-level policy learning and (3) low-level policy learning.

\begin{figure}[h]
    \centering
    \includegraphics[width=0.82\linewidth]{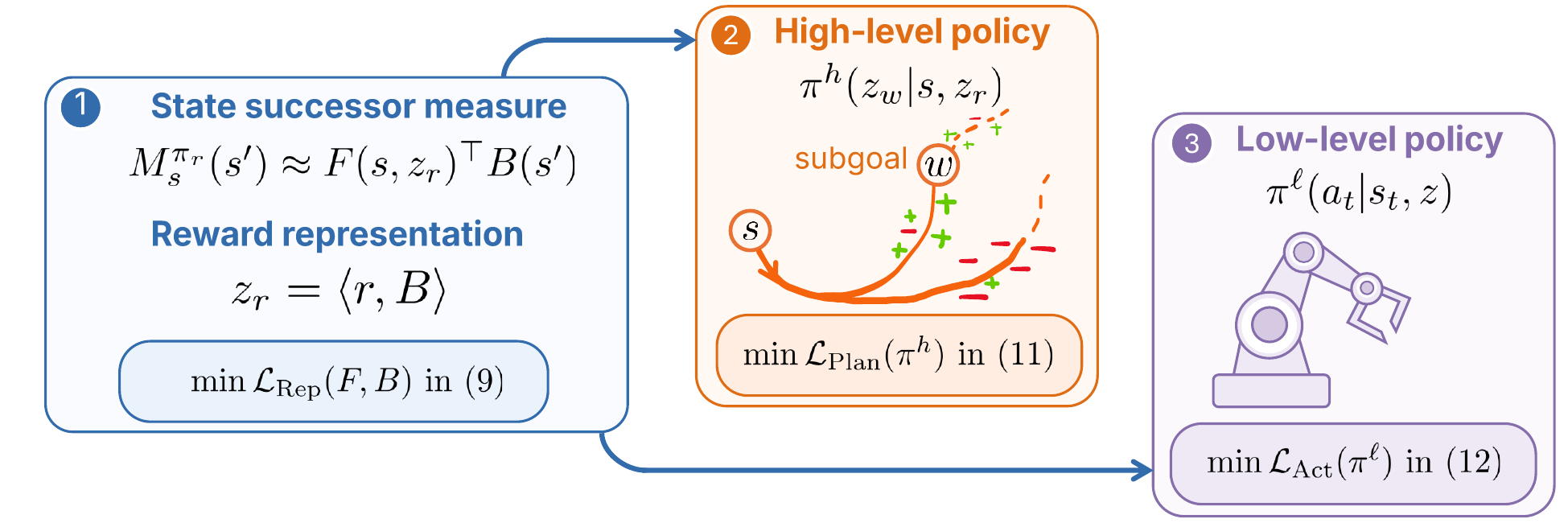}
    \caption{FB $\pi$-Switch training pipeline: learning state successor measures, high- and low-level control.}
    \label{fig:framework}
\end{figure}

During testing, given a task defined by reward $r$ and its embedding $z_r$, at each time step we first sample a subtask $z \sim \pi^h(\cdot \mid s, z_r)$ and then execute an action $a \sim \pi^\ell(\cdot \mid s_t, z)$. In the following, we describe the three training stages in detail and define their corresponding loss functions.

\textbf{\textcolor{RoyalBlue}{Step 1: Learning state-successor representations.}} We begin by learning action-free, reward-conditioned state-successor representations that can later be reused for hierarchical planning and control. Compared to the original FB objective~\eqref{eq:FB_standard_loss}, our formulation differs in two aspects: (i) it removes dependence on actions and policy-dependent bootstrapping, and (ii) it enables a single-step learning procedure, in contrast to the two coupled optimization steps required in FB framework (see Section~\ref{sec:preliminaries}). To achieve this, we marginalize over actions and replace the squared Bellman error with an expectile regression objective, following the idea of~\eqref{eq:IQL_intention}. This requires addressing two questions:

(a) \textit{What does action-free FB represent?} If $M^{\pi_z}_{s,a}(s') = F(s,a,z) ^\top B(s')$ holds, then the corresponding state-successor measures also admit a similar FB-representation. In particular,
\begin{align*}
    M^{\pi_z}_s(s') = \mathbb{E}_{a\sim \pi_z(\cdot\mid s)} M^{\pi_z}_{s,a}(s') = \mathbb{E}_{a\sim \pi_z(\cdot |s)} F(s,a,z) ^\top B(s') = F(s,z)^\top B(s').    
\end{align*}
where we define the action-marginalized successor features $F(s,z):= \mathbb{E}_{a\sim \pi_z(\cdot\mid s)} F(s,a,z)$. This allows us to directly optimize $(F,B)$ for the state-successor measure $M^{\pi_z}_s(s')$. Furthermore, the induced value function admits a linear representation:
\begin{align*}
    V^{\pi_z}(s; r) = \sum_{s'} M_s^{\pi_z} (s') r(s')  = \sum_{s'} F(s,z)^\top B(s') r(s') = F(s,z)^\top z_r
\end{align*}
where $z_r = \langle B, r\rangle = \mathbb{E}_{s'\sim \rho} \frac{1}{\rho(s')} B(s')r(s')$. Unlike the original FB formulation, the weighting factor $1/\rho(s')$ appears explicitly here. This difference arises from the fact that FB assumes the approximation $M^{\pi_z}_{s,a}(s') \approx F(s,a, z)^\top B(s') \rho(s')$, which absorbs the dependence on $\rho$ into the representation. In contrast, when using non-quadratic losses, removing this dependence on $\rho$ from the objective is no longer straightforward (see Appendix~A.2 for more details).

(b) \textit{How to learn without explicit policy feedback?}
In the original FB algorithm, the bootstrap target depends explicitly on the policy via $F(s_{t+1}, \pi_z(s_{t+1}), z)$, coupling representation and policy learning. Since we aim for a one-step, action-free objective, this form of bootstrapping is unavailable. We instead replace explicit policy feedback with an implicit greedy backup via intention-conditioned expectile regression, akin to~\eqref{eq:IQL_intention}. To learn the representation of $M^{\pi_z}_{s_t}(s')$, we minimize:
    \begin{align}
    \label{eq:L_rep_FB}
        \mathcal{L}_{\mathrm{Rep}}(F,B) &= 
        \mathbb{E}_{(s_t,s_{t+1})\sim \mathcal{D}, s'\sim \rho, z \sim \mathcal{Z}} 
        \big\vert \tau -  \mathbf{1}\{\Delta(F,B) <0\} \big\vert  \delta(F,B)^2  \\
        \mathrm{with}\quad \Delta(F,B) :&= 
        r_z(s_{t}) + \gamma F (s_{t+1},z)^\top z - F (s_t,z)^\top z, \nonumber\\
        \quad \delta(F,B)  :&= 
         \mathbf{1}\{s_{t}=s'\} + \gamma \overline{F} (s_{t+1},z)^\top \overline{B}(s') - F (s_t,z)^\top B(s'). \nonumber
    \end{align}
Here, $r_z(s)$ is the reward associated with embedding $z$, defined as $r_z(s)= B(s)^\top (\mathbb{E}_\rho BB^\top/\rho)^{-1} z$.
In practice, we enhance this loss with standard techniques such as double target networks, orthonormalization regularization, and improved sampling strategies for $z$ and $s'$, described in Appendix~C.

\textbf{\textcolor{Orange}{Step 2: Learning high-level policy.}} Given successor representations, we learn a policy that plans over subgoals rather than actions. Using Corollary~\ref{cor:main_hierarchical_SM}, we define an FB-approximation of $A_s^{\pi_w \to \pi_z}(r_z)$:
\begin{align}
    A_{\mathrm{FB}} (s,w,z) := F(s,z_w)^\top z +  \frac{F(s,z_w)^\top z_w }{F(w,z_w)^\top z_w} \Big( F(w, z) - F(w,z_w) \Big)^\top z - F(s,z)^\top z
    \label{eq:def_A_FB}
\end{align}
where $z_w=B(w)$ denotes the embedding of the subgoal $w$, and $z$ is the reward embedding. We learn a high-level policy $\pi^h(z_w \mid s, z_r)$ by maximizing $A_{\mathrm{FB}}$ using the AWR objective from~\eqref{eq:AWR_loss}:
    \begin{align}
        \mathcal{L}_{\mathrm{Plan}}(\pi^h) &= - 
        \mathbb{E}_{s,w\sim \mathcal{D}, z\sim \mathcal{Z}} \Big[
        \exp \Big(\beta A_{\mathrm{FB}}(s,w,z) \Big) \log \pi^h(z_w \mid s, z) \Big]. 
        \label{eq:L_plan}
    \end{align}


\textbf{\textcolor{Purple}{Step 3: Learning low-level policy.}} Finally, we learn a low-level policy operating in the action space. As in Step 2, we adopt the AWR objective~\eqref{eq:AWR_loss} and train the policy using one-step transitions. Concretely, the low-level policy $\pi^\ell (a_t \mid s_t, z)$ is learned by minimizing\footnote{Note that we omit the constant reward term from the objective and absorb the discount factor into parameter $\beta$.}: 
    \begin{align}
         \mathcal{L}_{\mathrm{Act}}(\pi^\ell) = - \mathbb{E}_{\substack{
       (s_t,a_t,s_{t+1})\sim \mathcal{D} \\
        z\sim \mathcal{Z}
        }} \Big[ \exp \Big( \beta  ( F(s_{t+1},z)^\top z - F(s_t, z)^\top z ) \Big) \log \pi^\ell (a_t \mid s_t, z) \Big]. 
        \label{eq:L_act}
    \end{align}

The final FB $\pi$-Switch objective is the sum of representation, planning, and control losses: $\mathcal{L}_{\mathrm{Rep}}(F,B) + \mathcal{L}_{\mathrm{Plan}}(\pi^h) + \mathcal{L}_{\mathrm{Act}}(\pi^\ell)$. Further training details are provided in Appendix~C.

\subsection{Hierarchical post-tuning for pretrained FB models}
\label{subsec:post_tuning_FB}
Although we present a full hierarchical algorithm, the hierarchical component can also be added on top of pretrained behavior foundational models (BFMs). Suppose we are given a pretrained successor representation $M_{s,a}^{\pi_z}(s') \approx F(s,a,z)^\top B(s')\rho(s')$ together with a latent-conditioned policy $\pi(a\vert s,z)$, as in prior work \citep{touati2022does,pirotta2023fast,tirinzoni2025zero}. In this case, hierarchy can be introduced by applying only Step 2 of our algorithm. Step 1 is implicitly satisfied since $M_s^{\pi_z} (s') = \mathbb{E}_{a\sim \pi(\cdot \mid s,z)} M_{s,a}^{\pi_z}(s')$, and Step 3 requires no additional training, as the pretrained policy $\pi(a\vert s,z)$ can be directly used as the low-level controller. However, pretraining a BFM is generally more challenging, since the representations are learned directly in the state-action space. A more detailed description is provided in Appendix C.3.

%% file: sections/6_experiments.tex
\section{Experiments}
\label{sec:experiments}
We evaluate our hierarchical learning methods in three settings: discrete control, standard goal-conditioned continuous control, and continuous control with general reward functions.

\textbf{An illustrative example: learning in a discrete maze.} 
We consider the environment from Section~\ref{sec:switching_successor_measures} and learn representations by minimizing the representation loss~\eqref{eq:L_rep_FB} (check Appendix C for details). Figure~\ref{fig:bigfigure_learning_based} shows the learned counterparts of the quantities illustrated in Figure~\ref{fig:bigfigure}.
\begin{figure}[h]
    \vspace{-5pt}
    \centering
    \begin{subfigure}[t]{0.24\textwidth}
        \centering
        \includegraphics[height=1\linewidth]{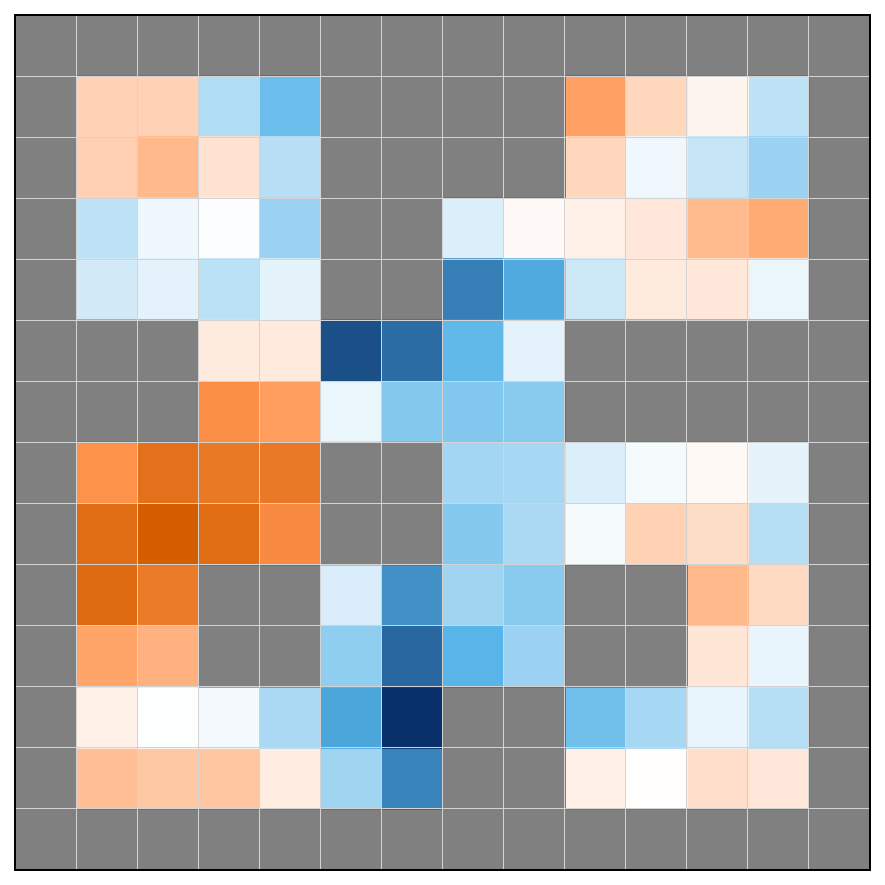}
        \caption{ $\mathrm{Proj}_B r( \textcolor{magenta}{s})$ }
    \end{subfigure}
    \begin{subfigure}[t]{0.24\textwidth}
        \centering
        \includegraphics[height=1\linewidth]{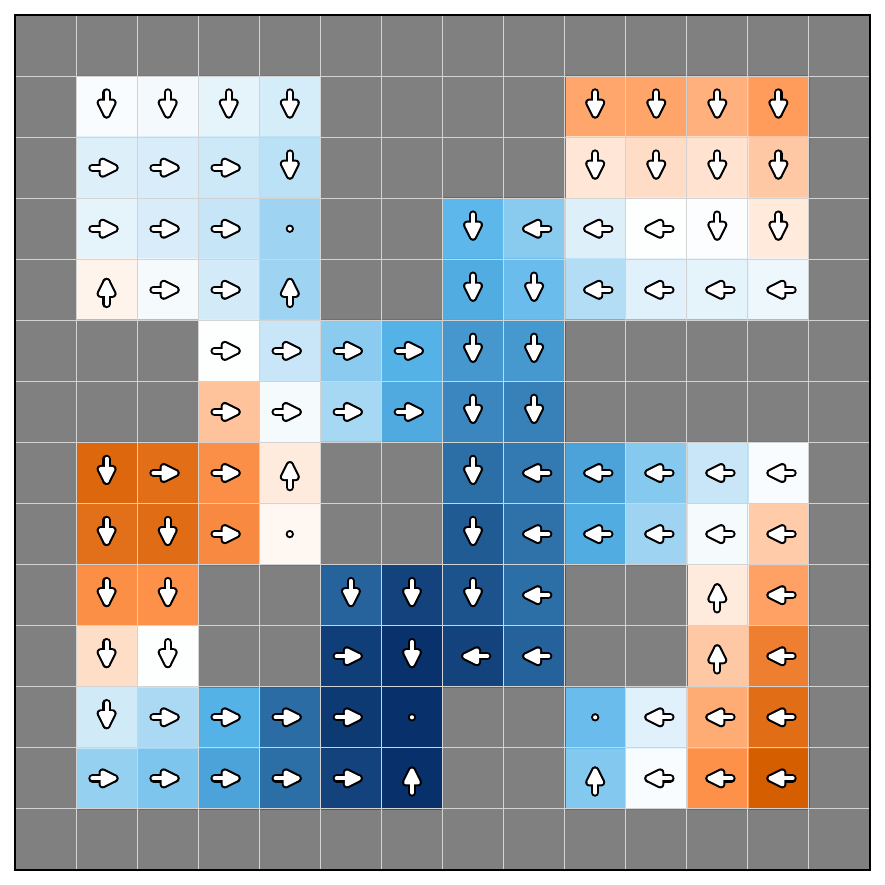}
        \caption{$F(\textcolor{magenta}{s},z_r)^\top z_r$}
    \end{subfigure}
    \begin{subfigure}[t]{0.24\textwidth}
        \centering
        \includegraphics[height=1\linewidth]{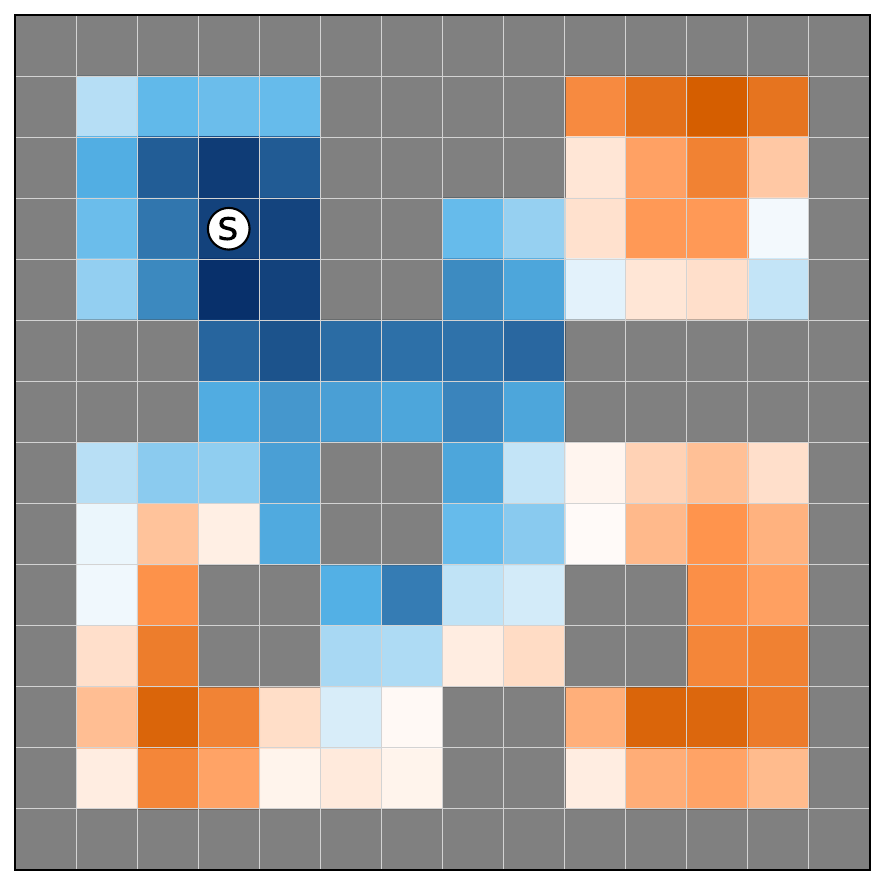}
        \caption{$A_{\mathrm{FB}}(s,\textcolor{magenta}{w},r)$  }
    \end{subfigure}
    \begin{subfigure}[t]{0.24\textwidth}
        \centering
        \includegraphics[height=1\linewidth]{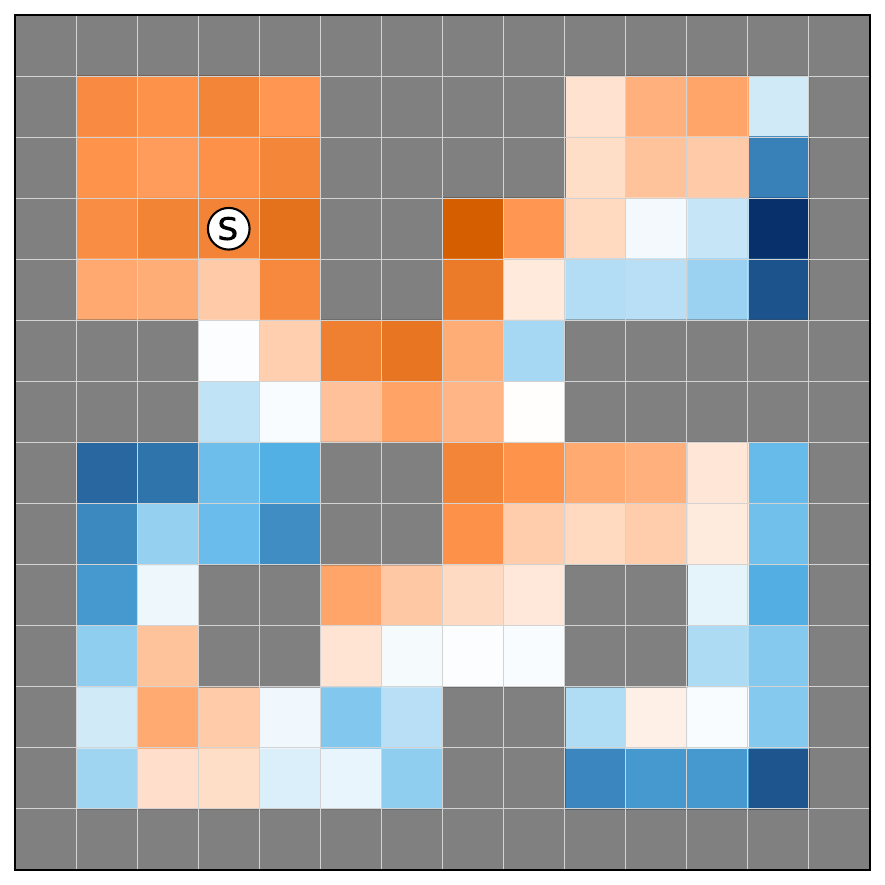}
        \caption{$A_{\mathrm{FB}}(s,\textcolor{magenta}{w},r) \vert_{< w}$ }
    \end{subfigure}
    \ifrlc
\else
        \caption{Depicted quantities as functions of the \textcolor{magenta}{magenta} variables: (a) Reward projected onto learned features: $\mathrm{Proj}_B r(s) = \langle r, B^\top (BB^\top)^{-1} B(s) \rangle$. (b) Learned value function for task embedding $z_r = \langle r, B \rangle$. (c) Learned switching advantage~\eqref{eq:def_A_FB}. (d) Learned analogue of Figure~\ref{fig:bigfigure}(d), corresponding to the pre-hitting terms in \eqref{eq:def_A_FB}.}
    \fi
    \label{fig:bigfigure_learning_based}
\end{figure}\\
This example highlights an important observation. As shown in Figure~\ref{fig:bigfigure_learning_based}, the successor representations can be successfully learned from data and captures meaningful structure across rewards~(a), values~(b), and advantage-related quantities~(c). However, the same example also reveals a challenge: as shown in Figure~\ref{fig:bigfigure_learning_based}~(d), learning the local component of the advantage function is difficult, as it requires subtracting a term supported only on states with nonzero reward. To address this in practice, we use a proxy, $\widehat{A}_{\mathrm{FB}}$, obtained by dropping the term $- \frac{F(s,z_w)^\top z_w }{F(w,z_w)^\top z_w} F(w,z_w)^\top z_r$ from~\eqref{eq:def_A_FB}. In addition, we omit the $\rho^{-1}(s)$ scaling when constructing reward embeddings, using $z_r = \mathbb{E}_{s\sim \rho} [r(s)B(s)]$ instead\footnote{Note that all our loss functions are invariant to scaling of $z$.}. Additional results, including a comparison between $A_{\mathrm{FB}}$ and $\widehat{A}_{\mathrm{FB}}$, are provided in Appendix~D.1. 

Overall, this example demonstrates that FB representations can be learned in practice, while also highlighting the difficulty of accurately estimating the local component of the advantage function. In subsequent experiments, we therefore omit the subtracted term when training the high-level policy.

\newcommand{\figPiH}{
    \begin{wrapfigure}{r}{0.23\textwidth}
        \centering
        \vspace{-15pt}
        \begin{subfigure}[t]{0.23\textwidth}
            \centering
            \includegraphics[width=0.9\linewidth]{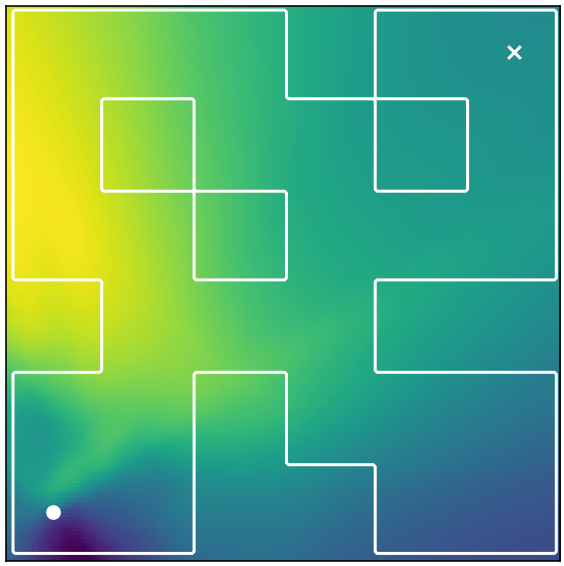}
            \caption{HIQL}
        \end{subfigure}
        \begin{subfigure}[t]{0.23\textwidth}
            \centering
            \includegraphics[width=0.9\linewidth]{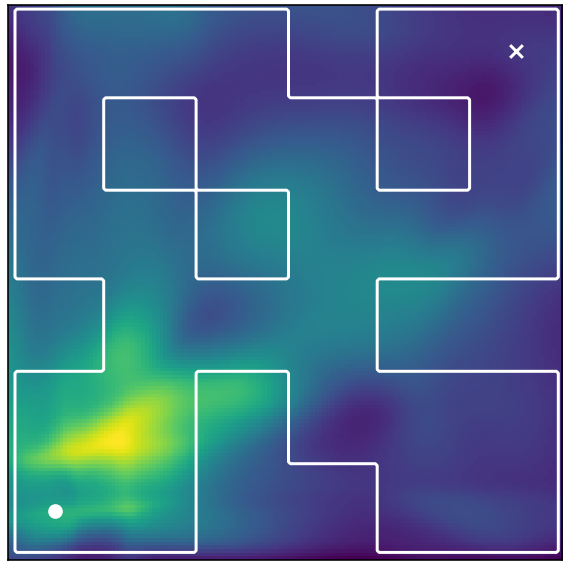}
            \caption{FB $\pi$-Switch}
        \end{subfigure}
        \caption{Comparison of $\pi^h$: our method induces optimal path subgoals.}
        \label{fig:hierarchies_qualitative_difference}
        \vspace{-25pt}
    \end{wrapfigure}
    }

\textbf{Continuous control for goal-reaching tasks.}
Next, we evaluate our method on AntMaze from OGBench \citep{park2024ogbench}, a standard GCRL benchmark with 29-dimensional states and 8-dimensional actions. We note that this is not purely a navigation problem; the agent must also learn a complex locomotion policy in addition to planning paths through the maze. Although our framework can be applied to more general tasks than goal-reaching, this setting provides a controlled and widely used environment for comparison.
We compare against HIQL, ICVF, FB, and One-Step FB \citep{zheng2026can}. For FB $\pi$-Switch, we evaluate both hierarchical and non-hierarchical variants (without the high-level policy $\pi^h$). We also include a hierarchical variant of FB (Section~\ref{subsec:post_tuning_FB}), denoted FB $+\pi^h$. Additional methodological comparisons are provided in Appendix A.4.

\ifrlc
\else
    \begin{table}[h]
    \centering
    {\small
    \begin{tabular}{l c c c c c c c c c c c}
    \toprule
    \textbf{AntMaze} 
    & \multicolumn{2}{c}{\textbf{HIQL}} 
    & \textbf{ICVF}
    & \multicolumn{1}{c}{\raisebox{-12pt}[0pt][0pt]{\shortstack{\textbf{One-Step} \\ \textbf{FB}}}}
    & \multicolumn{2}{c}{\textbf{FB}} 
    & \multicolumn{2}{c}{\textbf{FB $\pi$-Switch}} \\
    \cmidrule(lr){2-3}   \cmidrule(lr){6-7} \cmidrule(lr){8-9}
    
    & vanilla & $-\pi^h$ &  &  & vanilla & $+\pi^h$ & vanilla & $-\pi^h$ \\
    \midrule
    
    Medium   
    & {93 \scriptsize $\pm$ 1} 
    & {73 \scriptsize $\pm$ 3} 
    & {46 \scriptsize $\pm$ 37} 
    & {35 \scriptsize $\pm$ 7} 
    & {47 \scriptsize $\pm$ 1} 
    & {50 \scriptsize $\pm$ 12} 
    & {87 \scriptsize $\pm$ 6} 
    & {70 \scriptsize $\pm$ 4} \\
    
    Large    
    & {73 \scriptsize $\pm$ 6} 
    & {23 \scriptsize $\pm$ 4} 
    & {28 \scriptsize $\pm$ 3} 
    & {23 \scriptsize $\pm$ 18} 
    & {21 \scriptsize $\pm$ 5} 
    & {20 \scriptsize $\pm$ 5} 
    & {66 \scriptsize $\pm$ 9} 
    & {37 \scriptsize $\pm$ 7} \\
    
    Giant    
    & {5 \scriptsize $\pm$ 3} 
    & {0 \scriptsize $\pm$ 0} 
    & {0 \scriptsize $\pm$ 0} 
    & {0 \scriptsize $\pm$ 0} 
    & {0 \scriptsize $\pm$ 0} 
    & {0 \scriptsize $\pm$ 0} 
    & {1 \scriptsize $\pm$ 1} 
    & {1 \scriptsize $\pm$ 1} \\
    
    Teleport 
    & {42 \scriptsize $\pm$ 4} 
    & {32 \scriptsize $\pm$ 6} 
    & {23 \scriptsize $\pm$ 6} 
    & {8 \scriptsize $\pm$ 2} 
    & {25 \scriptsize $\pm$ 7} 
    & {13 \scriptsize $\pm$ 7} 
    & {40 \scriptsize $\pm$ 6} 
    & {29 \scriptsize $\pm$ 4} \\
    \bottomrule  \\
    \end{tabular}
    }
        \caption{Average (binary) success rate ($\%$) across the five test-time goals with \texttt{navigate-v0} datasets. The results are averaged over 50 episodes and 5 seeds, and the standard deviations are reported after $\pm$ sign.}
        \label{tab:results_antmaze}
        \vspace{-15pt}
    \end{table}
\fi

All methods are trained under a common setup with sparse rewards and no terminal masking. Models use comparable architectures and are trained for 1M steps. Hierarchical variants are initialized from pretrained non-hierarchical models and trained for an additional 500k steps, updating only the high-level policy. This enables direct comparison between variants sharing the same representation. Evaluation is performed on five standard tasks per maze, with episodes of length 1000 (App. C).

\figPiH

Table~\ref{tab:results_antmaze} reports success rates across four environments. FB $\pi$-Switch achieves performance comparable to HIQL, despite not being specifically designed for GCRL. Adding a high-level policy consistently improves performance, highlighting the benefit of hierarchical structure. Notably, even without hierarchy, FB $\pi$-Switch outperforms other non-hierarchical baselines. All methods struggle in the Giant setting due to long horizons and sparse rewards. Since performance depends on learned representations, sufficiently strong representations are essential.

To better understand these results, we ablate HIQL by (i) removing the high-level policy, (ii) replacing its critic with a bilinear form, and (iii) using intention-based values $V^{\pi_g}(s; s')$. These modifications progressively align HIQL with FB $\pi$-Switch without hierarchy, yielding comparable performance (Appendix D.2). This suggests that these adaptations do not significantly affect performance in GCRL settings.

We also observe qualitative differences in learned high-level policies (Figure~\ref{fig:hierarchies_qualitative_difference}). HIQL uses representations conditioned on both state and goal, learning subgoals that lead to similar immediate actions. In contrast, our method uses goal-only representations, yielding subgoals that lie along coherent trajectories toward the goal. This results in more structured, path-consistent behavior (see Appendix D.2 for more visualizations).

\ifrlc
\else
    \begin{wraptable}{r}{0.53\textwidth}
        \centering
        \setlength{\tabcolsep}{3pt}
        {\small
        \begin{tabular}{llccc}
        \toprule
        $\pi^h$ & $\pi^\ell$ & Medium & Large & Teleport \\
        \midrule
        FB$+\pi^h$ & FB $\pi$-Switch 
        & $+30\%$ 
        & $+90\%$ 
        & $+74\%$  \\
        \midrule
        FB $\pi$-Switch & FB$+\pi^h$
        & $+19\%$ 
        & $-23\%$ 
        & $-14\%$ \\
        \bottomrule
        \end{tabular}
        }
        \caption{Relative improvement in success rate (\%) over the base FB $+\pi^h$ model when combined with high- or low-level policies of FB $\pi$-Switch.}
        \label{tab:explanation_HFB_GCRL}
        \vspace{-10pt}
    \end{wraptable}
\fi

Finally, we analyze the role of hierarchy in FB-based methods. While adding a high-level policy to standard FB yields only marginal improvements, we find that this is not due to a lack of meaningful subgoals. To isolate the effect, we construct hybrid agents combining high- and low-level policies from FB and FB $\pi$-Switch. Results in Table~\ref{tab:explanation_HFB_GCRL} show that hierarchical FB can learn useful subgoals, and that the limited gains are primarily due to the quality of the low-level policy rather than the hierarchical decomposition itself.

\ifrlc
\else
    \begin{wrapfigure}{r}{0.48\textwidth}
        \centering
        \vspace{-15pt}
        \begin{subfigure}[t]{0.48\textwidth}
            \centering
            \includegraphics[width=\linewidth]{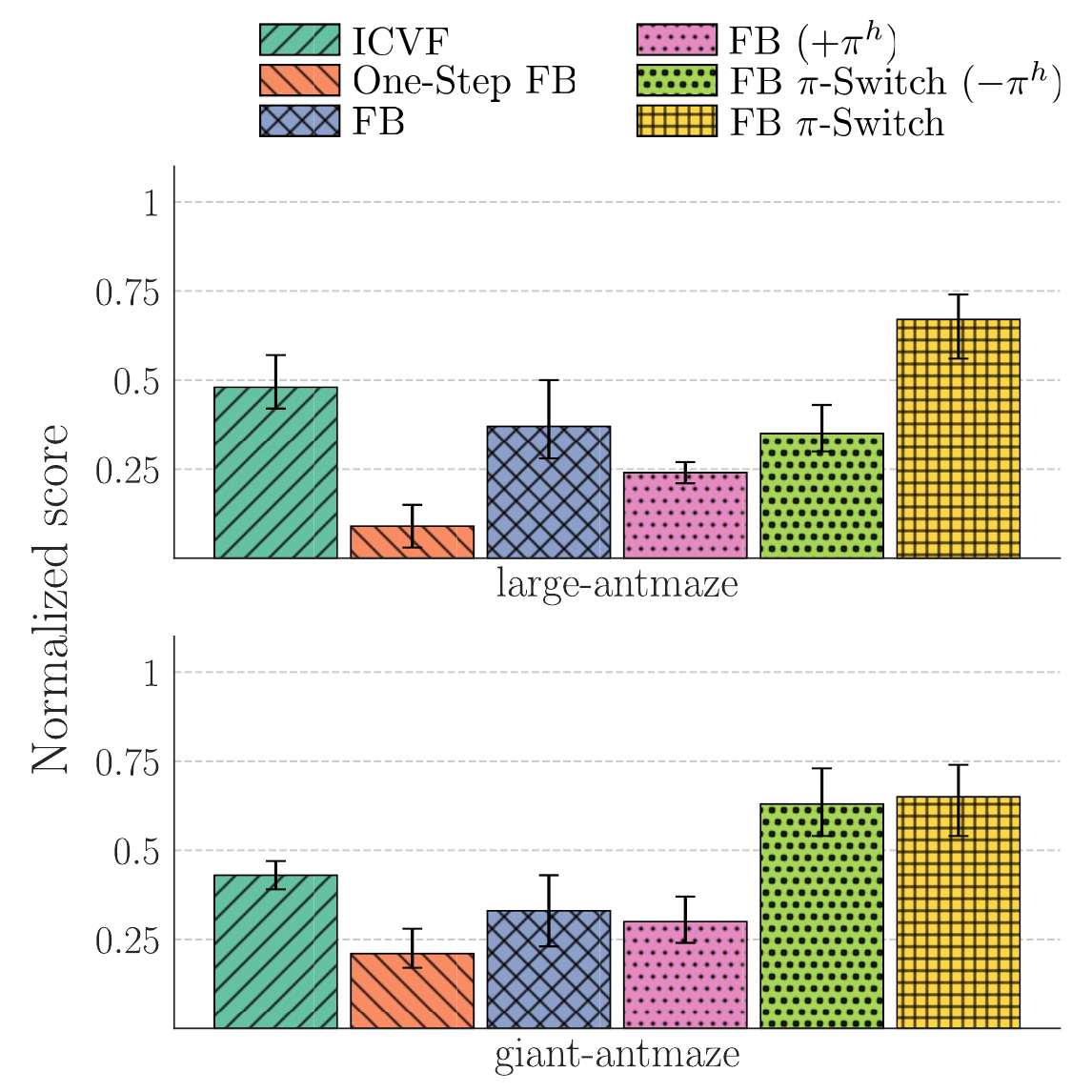}
        \end{subfigure}
        \caption{Normalized returns across the five tasks in Figure~\ref{fig:rewards_antmaze}, evaluated over 50 episodes (1000 steps) and 5 seeds per task. We report IQM with stratified bootstrap 95\% confidence intervals.}
        \label{fig:regions_main}
        \vspace{-15pt}
    \end{wrapfigure}
\fi

\textbf{Continuous control for general tasks.} 
Motivated by the discrete example in Figure~\ref{fig:bigfigure}, we extend the AntMaze benchmark by defining reward functions over multiple spatial regions instead of a single goal, inducing trade-offs between local and global objectives. This yields a controlled yet more general setting with diverse reward landscapes (Figure~\ref{fig:rewards_antmaze}). We focus on the Large and Giant mazes, where long-horizon structure is critical. 
Using the models from the previous section, we normalize returns and, following \cite{agarwal2021deep}, report the interquartile mean (IQM) to aggregate performance across tasks, with stratified bootstrap $95\%$ confidence intervals in Figure~\ref{fig:regions_main}. 
\ifrlc
\else
Full results and additional comparisons are provided in Appendix D.3.
\fi

FB $\pi$-Switch achieves the best average performance. In the Giant setting, the non-hierarchical variant also performs well, likely due to strong local rewards guiding behavior. Notably, FB $\pi$-Switch remains robust across both environments, suggesting that the learned high-level policy improves robustness. We observe that not all hierarchical variants are equally effective, reinforcing that the quality of the low-level policy is crucial for realizing the benefits of hierarchy.

\ifrlc
\else
Importantly, this experiment highlights a key distinction: while goal-conditioned methods such as HIQL are tailored to single-target objectives, our framework naturally extends to distributed reward settings without modification. This enables the agent to exploit intermediate rewards and compose behaviors across regions, leading to improved performance in more general tasks.
\fi

%% file: sections/7_conclusion.tex
\section{Conclusion}
\ifrlc
\else
In this work, we introduced switching successor measures as a framework for modeling hierarchical structure in zero-shot reinforcement learning, and proposed FB $\pi$-Switch, a method that jointly learns high- and low-level policies from shared representations. Without relying on expert trajectories or hand-crafted temporal abstractions, it enables hierarchical behavior to emerge directly from the representations. Empirical results across goal-reaching and more general tasks show that FB $\pi$-Switch achieves competitive performance and consistently improves over non-hierarchical baselines.
\fi

Our findings suggest that hierarchical learning with successor representations can effectively bridge local and global decision-making, but several open questions remain. Future work could explore the dependence of hierarchical behavior on the quality of the base model, and which task properties, such as reward sparsity and long horizons, most strongly favor this hierarchical structure. Another direction is comparing this approach to generative model-based methods, which often excel in long-horizon planning but introduce additional complexity.
\ifrlc
\else
 Extending our framework to multi-subgoal planning—where agents dynamically compose sequences of subgoals—could unlock even greater adaptability, though error accumulation across chained subgoals remains a key challenge.
 \fi

%% file: sections/8_appA.tex
\section*{A Additional discussions}

\subsection*{A.1 Further related work}
In this section, we highlight additional works related to our approach. For a more comprehensive overview, we refer the reader to the survey on hierarchical RL by \cite{klissarov2025discovering}.

\textbf{Goal-conditioned RL.}
Many goal-conditioned RL methods build on hindsight experience replay \citep{andrychowicz2017hindsight}, which relabels states encountered along trajectories as goals to address sparse rewards. Successor features have also been applied in goal-conditioned RL \citep{ramesh2019successor, hoang2021successor}. Another line of work introduces metric \citep{reichlin2024learning} or quasimetric \citep{myers2024learning} structure into value representations to improve robustness. Option-based extensions of HIQL \citep{ahn2025option} reduce the effective horizon by minimizing TD-error over multi-step targets rather than one-step transitions. Other approaches focus on online adaptation of learned goal-based hierarchies \citep{wu2024planning}.

\textbf{Successor measures.} 
Generative formulations of successor measures have been studied under the name of $\gamma$-models \citep{janner2020gamma}. Empirical work has demonstrated that FB models can produce diverse behaviors in complex domains such as humanoid control \citep{tirinzoni2025zero}. Other studies investigate the impact of dataset quality on FB methods, showing sensitivity to limited or homogeneous data \citep{jeen2024zero}. Alternative representations of successor measures have been explored beyond the FB parameterization, see \cite{agarwal2025unified} for an overview.

\textbf{Hierarchical methods.}
Another line of work develops hierarchical skill models using explicit dynamics or world models \citep{shi2022skill, hansen2024hierarchical, gurtlerlong}, where high-level policies operate over predicted skill outcomes. Related directions include model-based planning \citep{sobal2025learning} and autoregressive generation of subgoal sequences \citep{choi2026chain}. Hierarchical structure has also been explored in imitation learning, for example through options \citep{krishnan2017ddco} or latent skill models that combine imitation learning with offline RL for skill composition \citep{rosete2023latent}.

\textbf{Policy switching.}
Policy switching has been studied in several contexts. \cite{thakoor2022generalised} derive a generalized policy improvement result under a geometric horizon model. Other work considers switching between agents at test time \citep{neggatu2025evaluation} or learning switching mechanisms over multiple policies \citep{zhang2023qmp}. Additional directions analyze the cost or frequency of switching in exploration \citep{bai2019provably} and offline learning \citep{ma2024switch}. However, these approaches do not study switching in the context of learned successor representations or hierarchical subgoal-based policies.

\subsection*{A.2 Further remarks on the FB representation}
Here we recall key properties of the FB representation of \cite{touati2022does}. The (state-action) successor measure is defined as $M^{\pi_z}_{s,a}(s') = F(s,a,z)^\top B(s')\rho(s')$ and its corresponding policy as $\pi_z(s) = \argmax_a F(s,a,z)^\top z$. According to Theorem 2 in \cite{touati2022does}, for any reward function $r$ and $z_r = \mathbb{E}_\rho [r(s)B(s)]$, the optimal Q-function can be written as $Q_r^\star(s,a) = F(s,a,z_r)^\top z_r$. This shows that the FB representation can simultaneously express: (i) successor measures $M^{\pi_z}_{s,a}(s')$ through $F(s,a,z)^\top B(s')\rho(s')$, and (ii) optimal Q-values for arbitrary reward functions $Q_r^\star(s,a)$ through $F(s,a,z_r)^\top z_r$. Both are based on the same linear structure in the feature space spanned by $B(\cdot)$. Our framework builds on this observation by approximating both (state) successor measure and (state) value function using the FB parameterization.

Next, we provide additional remarks regarding the loss functions of the FB algorithm. First, an orthonormalization loss is added in order to regularize backward representation:
\begin{align*}
    \mathcal{L}_{\mathrm{norm}}(B) = \Vert \mathbb{E}_{s\sim \rho} B(s) B(s)^\top - \mathrm{Id} \Vert_\mathrm{F}^2 = \mathbb{E}_{s,s'\sim \rho} [ (B(s)^\top B(s'))^2 - 2\Vert B(s) \Vert_2^2] + \mathrm{Const},
\end{align*}
where $\mathrm{Id}$ is the identity matrix. An auxiliary loss on forward representations has been proposed in \cite{touati2022does}, however we do not include it since prior work reports it has little impact.

Next, we explain derivation of the loss function in \eqref{eq:FB_standard_loss}. Given a state distribution $\rho$, define the norm $\|M^\pi \|_{\rho}^2 := \mathbb{E}_{(s,a)\sim \mathcal{D}, s'\sim \rho}  \big( M^\pi_{s,a}(s') / \rho(s') \big)^2 $. For simplicity, we assume a finite state-space, which allows using point values such as $\rho(s')$ and indicator functions $\mathbf{1}\{ s_{t}=s' \}$. As discussed in \cite{touati2022does}, this extends to continuous spaces with appropriate measure-theoretic definitions. 

\cite{touati2022does} defines the objective $\mathcal{L}(F, B) := \left\| F^\top B \rho - \left( \mathrm{Id} + \gamma P_{\pi} F^\top B \rho \right) \right\|_{\rho}^2$, which expands to the FB loss used in our work (Equation \eqref{eq:FB_standard_loss}):
\begin{align*}
    \mathbb{E}_{\substack{
    (s_t,a_t,s_{t+1})\sim \mathcal{D} \\
    s'\sim \rho, z \sim \mathcal{Z}
    }} \Big( \frac{ \mathbf{1}\{ s_{t}=s' \} }{\rho(s')} + \gamma  \overline{F}(s_{t+1},\pi_z(s_{t+1}), z)^\top \overline{B}(s')  - F(s_t,a_t,z)^\top B(s') \Big)^2
\end{align*}
Finally, under a quadratic objective, this loss function simplifies (up to a constant) to:
\begin{align*}
    \mathbb{E}_{\substack{
    (s_t,a_t,s_{t+1})\sim \mathcal{D} \\
    s'\sim \rho, z \sim \mathcal{Z}
    }} \Big[ \big( \gamma  \overline{F}(s_{t+1},\pi_z(s_{t+1}), z)^\top \overline{B}(s')  - F(s_t,a_t,z)^\top B(s') \big)^2 - 2 F(s_t, a_t, z)^\top B(s_t) \Big] 
\end{align*}

Remark: \cite{touati2022does} defines rewards and successor measures with respect to the next state $s_{t+1}$, whereas we use the more common convention based on the current state $s_t$. This leads to a slight difference in the loss: their formulation replaces the identity operator with the transition operator $P$, and introduces a term $F(s_t,a_t,z)^\top B(s_{t+1})$ instead of $F(s_t,a_t,z)^\top B(s_{t})$ in the final expression.

\subsection*{A.3 Analysis of Corollary \ref{cor:main_hierarchical_SM} from the perspective of prior work}
We begin by recalling the result from Corollary \ref{cor:main_hierarchical_SM}:
\begin{align*}
    A_s^{\pi_w \to \pi}(r) =  V^{\pi_{w}}(s;r) +  \frac{M^{\pi_{w}}_s(w)}{M^{\pi_{w}}_w(w)} \Big( V^\pi(w;r) - V^{\pi_{w}}(w;r)\Big) - V^\pi(s;r).
\end{align*}

We now relate this expression to two closely related lines of work.

\paragraph{Simplified objectives for GCRL in \cite{park2023hiql}.} As discussed earlier in Section \ref{sec:switching_successor_measures}, HIQL adopts a simplified advantage-based formulation. For comparison, consider a hypothetical variant of HIQL in which subgoal horizons are not fixed but sampled as $k\sim \mathrm{Geom}(1-\gamma)$. In this case, the HIQL objective can be interpreted as maximizing the advantage:
\begin{align*}
    A(s,w,r) = (1-\gamma) M^{\pi_{\beta}}_s(w)  V^\pi(w;r) - V^\pi(s;r),
\end{align*}
where $\pi_\beta$ denotes the behavior policy used to collect the offline dataset. Compared to Corollary \ref{cor:main_hierarchical_SM}, this formulation omits the local term $V^{\pi_{w}}(s;r) -  \frac{M^{\pi_{w}}_s(w)}{M^{\pi_{w}}_w(w)}  V^{\pi_{w}}(w;r)$, which is constant in the goal-conditioned setting and is equal to zero under sparse rewards, since only the goal state carries non-zero reward. Moreover, HIQL assumes absorbing goal states, which implies $M^{\pi_{w}}_w(w)=1/(1-\gamma)$. Finally, the term $M^{\pi_{w}}_s(\cdot )$ is effectively approximated using sampled future states $s_{t+k}\sim (1-\gamma) M^{\pi_\beta}_s(\cdot )$, corresponding to the future states of the behavior policy rather than those induced by the optimal goal-reaching policy for subgoal $w$.

\paragraph{Sampling successor measure via generative models in \cite{farebrother2026compositional}.}
We next compare our formulation with compositional policy evaluation using generative successor models. Here, we explicitly denote the dependence on the discount factor $\gamma$ using $V^{\pi\vert \gamma}(s;r)$ and $M^{\pi\vert \gamma}_s(w)$.

Consider Lemma 1 in \cite{farebrother2026compositional} (also reported in \cite{thakoor2022generalised}), for a compositional policy $\pi = \pi_{z_1}  \xrightarrow{\alpha} \pi_{z_2}$, where the first policy $\pi_{z_1}$ uses a discount factor $\gamma(1-\alpha)$ and the second uses $\gamma$. For $n=2$, an unbiased estimate of the value function is given by:
\begin{align*}
    \widehat{V}^{\pi\vert\gamma} (s;r) = \frac{1}{1-\gamma(1-\alpha)} \left[ r(s_1^+) +  \frac{\alpha\gamma}{1-\gamma} r(s_2^+) \right],    
\end{align*}
where\footnote{Note that $M_s^{\pi\vert \gamma}$ is a proper probability distribution in our framework only after scaling by $(1-\gamma)$.} $s_1^+\sim M^{\pi_{z_1}\vert \gamma(1-\alpha)}_{s}$ and $s_2^+\sim M^{\pi_{z_2}\vert \gamma}_{s_1^+}$. The expectation of this estimate is:
\begin{align}
    \frac{1}{1-\gamma(1-\alpha)}\mathbb{E}_{w \sim M^{\pi_{z_1}\vert \gamma(1-\alpha)}_{s}}  \Big[ r(w) &+   \frac{\alpha\gamma}{1-\gamma}\mathbb{E}_{s'\sim M^{\pi_{z_2}\vert \gamma}_{w}} r(s') \Big] \nonumber \\
    = V^{\pi_{z_1}\vert \gamma(1-\alpha)}(s;r) &
    + \alpha \gamma
    \sum_w M^{\pi_{z_1}\vert \gamma(1-\alpha)}_{s}(w) V^{\pi_{z_2}\vert \gamma} (w;r).
    \label{eq:farebrother_expression_appendix}
\end{align}
This expression is structurally similar to Corollary \ref{cor:main_hierarchical_SM}. However, there are several key differences. First, we follow $\pi_w$ until reaching $w$ instead of following the first policy with geometrically discounted horizon parameterized by $\gamma(1-\alpha)$. Second, since switching occurs upon reaching $w$, the sum over all subgoals $w$ in \eqref{eq:farebrother_expression_appendix} collapses to a single non-zero term. As a result, our method does not require a generative model to estimate the switching advantage function.

\subsection*{A.4 Further comparisons with prior methods}

\textbf{Intention conditioned value functions (ICVF) \citep{ghosh2023reinforcement}.} 
The first stage of our algorithm is closely related to the representation learning stage of ICVF. However, ICVF assumes a structured factorization of the form
$M^{\pi_z}_s(g) = \phi(s)^\top T(z)\psi(g)$,
which corresponds to the factorization $F(s,z) = \phi(s)^\top T(z)$ and $B(g) = \psi(g)$, where $T(z) \in \mathbb{R}^{d \times d}$ is a learned matrix. In practice, \cite{ghosh2023reinforcement} restricts intentions $z$ to goal states, effectively reducing the method to GCRL setting. Moreover, ICVF is not a zero-shot method: its representation learning stage (corresponding to our Stage 1) is used solely to learn state embeddings $\phi$, after which a separate offline RL algorithm is trained for each downstream reward function. In contrast, our approach learns a single value representation that directly supports inference of both high- and low-level policies without additional RL training. Finally, ICVF does not incorporate a hierarchical structure, as it does not model a high-level policy over intentions.

\textbf{Functional reward encodings (FRE) \citep{frans2024unsupervised}.} Another promising direction in zero-shot RL is learning compact, informative reward representations. FRE addresses this by learning reward embeddings via transformer-based variational autoencoders, trained using a variational lower bound on mutual information. Importantly, this representation learning stage uses only sampled state–reward pairs and does not involve value functions or RL objectives. In a second stage, the learned embeddings are used as conditioning variables for standard offline RL algorithms such as IQL to train policies for downstream tasks. This approach is largely orthogonal to our framework: while FRE focuses on learning expressive reward embeddings as a separate preprocessing step, we directly exploit learned successor measures to infer both reward structure and high- and low-level policies within a single framework. Moreover, our method is significantly simpler in terms of architecture, relying on standard function approximators such as MLPs, whereas FRE requires more expressive sequence models such as transformer-based autoencoders.

\textbf{Hilbert foundation policies (HILP) \citep{park2024foundation}.}
Similar to the previous methods, HILP follows a two-stage pipeline: first learning state representations used to construct reward embeddings, and then applying offline RL to train policies conditioned on these embeddings. Interestingly, HILP also employs expectile regression, but assumes that value functions can be represented as Euclidean distances in a latent Hilbert space. 
However, the learned value functions are restricted to goal-aligned settings, i.e., $V^{\pi_g}(s; r_g)$, where policy and reward correspond to the same goal. This structure enables a simple form of hierarchical planning by selecting midpoints $w$ that minimize distances $s \rightarrow w$ and $w \rightarrow g$ as subgoals. In practice, this is implemented by sampling $N$ candidate subgoals from the dataset and choosing the one closest to both the state and the goal. 
Since HILP does not explicitly leverage learned value functions or successor measures, its sample efficiency remains unclear: each iteration requires sampling a new reward embedding $z$ and training a separate policy $\pi_z$ using an offline RL algorithm. In contrast, our approach (and related FB-based methods) trains policies for all embeddings $\pi_z$ in parallel by sharing a unified successor representation.

\textbf{Regularizing FB models with value expectile regression \citep{zheng2025towards}.} 
Recent work by \cite{zheng2025towards} proposes several modifications to the FB algorithm aimed at improving robustness. In addition to architectural changes such as diffusion-based policy extraction and attention-based representation learning, they also introduce a mechanism to mitigate extrapolation error.
Specifically, they use expectile regression to learn state-value estimates $V^{\pi_z}(s; z)$ by regressing them toward $F(s,a,z)^\top z$. These learned values are then used as targets in the auxiliary FB loss, replacing the original term based on $Q^{\pi_z}(s', \pi_z(s'); z)$, which improves stability during training.
While our framework also learns state-value functions, there are key differences. First, \cite{zheng2025towards} retains an action-conditioned FB representation of the form $F(s,a,z)^\top B(g)$ and additionally learns state-value functions on top of it. In contrast, we only model state-value functions and do not explicitly parameterize Q-functions. Second, we do not employ the auxiliary FB loss; our representations are learned solely using the main objective.

\textbf{One-step FB learning \citep{zheng2026can}.}
The key idea in \cite{zheng2026can} is to avoid learning successor measures for every policy $\pi_z$ and instead focus on estimating it for the behavior policy alone. Policies conditioned on $z$ are then obtained via one-step policy improvement using the successor measure induced by the behavior policy.
This design simplifies representation learning by allowing direct use of next actions $a_{t+1}$ from the dataset, avoiding the need to bootstrap through $\pi_z(s_{t+1})$. The resulting policy is extracted as a softmax over actions with logits given by $F_\beta(s,a)^\top z$, where $\beta$ denotes the behavior policy, corresponding to $Q^{\pi_\beta}(s,a; r_z)$.
While this approach simplifies the FB pipeline, it is not directly applicable to our setting. Our framework requires estimates of $V^{\pi_w}(s; r)$ for intermediate policies, which cannot be recovered from representations restricted to the behavior policy alone.

%% file: sections/9_appB.tex
\section*{B Theoretical results}
\label{app:subsec_proofs}

We first provide a short derivation of Corollary \ref{cor:main_hierarchical_SM} from Theorem \ref{thm:SM_pi_to_pi} for completeness. Recall the definitions of the switching successor measure in \eqref{eq:SM_switch_agnostic} and the switching advantage in \eqref{eq:def_A_random_switch}:
\begin{align*}
    M^{\pi_{w} \to \pi}_s(s') =  \mathbb{E}\left[ M^{\pi_{w}|_k \to \pi}_s(s') \mathbf{1}\{ k=H_s^{\pi_w}(w)\} \right], \  A^{\pi_{w} \to \pi}_s(r)  =  \mathbb{E}\left[ A^{\pi_{w}|_k \to \pi}_s(r) \mathbf{1}\{ k=H_s^{\pi_w}(w)\} \right].
\end{align*}
Using the standard relation between value functions and successor measures, namely $A_{s}^{\pi_w|_k \to \pi}(r)=\langle M^{\pi_{w}|_k \to \pi}_s -M^{\pi}_s, r\rangle$, we obtain:
\begin{align*}
    A^{\pi_{w} \to \pi}_s(r)  &= \mathbb{E}\left[ \langle M^{\pi_{w}|_k \to \pi}_s -M^{\pi}_s, r\rangle \mathbf{1}\{ k=H_s^{\pi_w}(w)\} \right]
    \\ &= \mathbb{E}\left[ \langle M^{\pi_{w}|_k \to \pi}_s , r\rangle \mathbf{1}\{ k=H_s^{\pi_w}(w)\} \right] - \langle   M^{\pi}_s  , r\rangle
    \\ &=  \sum_{s'} \mathbb{E}\left[ M^{\pi_{w}|_k \to \pi}_s(s')  \mathbf{1}\{ k=H_s^{\pi_w}(w)\} \right] r(s') - V^\pi(s;r) \\ &=
    \sum_{s'} M^{\pi_{w} \to \pi}_s(s') r(s') - V^\pi(s;r)
\end{align*}
Applying Theorem \ref{thm:SM_pi_to_pi} to the first term yields:
\begin{align*}
     \sum_{s'} M^{\pi_{w} \to \pi}_s(s')r(s') &= \sum_{s'}M^{\pi_w}_s(s')r(s') + \frac{M^{\pi_w}_s(w)}{M^{\pi_w}_w(w)} \Big( \sum_{s'}M^{\pi}_w(s')r(s') - \sum_{s'} M^{\pi_w}_w(s')r(s') \Big)\\
     &= V^{\pi_w}(s;r) + \frac{M^{\pi_w}_s(w)}{M^{\pi_w}_w(w)} \Big( V^{\pi}(w;r) - V^{\pi_w}(w;r) \Big)
\end{align*}
which completes the proof of Corollary \ref{cor:main_hierarchical_SM}.

Before proving Theorem \ref{thm:SM_pi_to_pi}, we recall the definition of hitting time from Section \ref{sec:preliminaries}: $H^\pi_s(w) = \min \{ t \ge 0 : s_t=w \vert s_0 =s, \pi\}$. We use the following result:
\begin{lemma}
    [Lemma B.1 in \cite{myers2024learning}] We have for any $s,w \in \mathcal{S}$ and any policy $\pi$:\\[0.2cm]
    \hspace*{0.32\textwidth}
    $ \displaystyle 
        M^{\pi}_s(w) = \mathbb{E}\!\left[ \gamma^{H_s^{\pi}(w)} \right] M^{\pi}_w(w)
    $
    \label{lemma:Myers_hitting_time}
\end{lemma}

It is also useful to contrast Theorem \ref{thm:SM_pi_to_pi} with prior work. In particular, Lemma 3.3 of \cite{myers2024learning} shows that:
\begin{align*}
    M^{\pi_{w} \to \pi}_s(s') \geq   \frac{M^{\pi_w}_s(w)}{M^{\pi_w}_w(w)} M^{\pi}_w(s').
\end{align*}
This provides a lower bound on the switching successor measure, where the term on the right hand side captures discounted visitation after reaching the subgoal $w$. In contrast, Theorem \ref{thm:SM_pi_to_pi} establishes an equality, expressing the switching successor measure exactly in terms of standard successor measures.

\subsection*{B.1 Proof of Theorem \ref{thm:SM_pi_to_pi}}
\begin{proof}
    We construct a (generally non-Markovian) policy, denoted $\overline{\pi}$, that follows the subgoal policy $\pi_w$ until the state $w$ is reached, and then switches to $\pi$ afterwards:
    \begin{align*}
        \overline{\pi}(a_t\vert s_t) = \begin{cases}
            \pi(a_t \vert s_t), \quad w\in \{s_0, s_1,\dots,s_t\}\\
            \pi_w(a_t\vert s_t),\quad \mathrm{otherwise}.
        \end{cases}
    \end{align*}
    Let $H_s^{\overline{\pi}}(w)$ denote the hitting time of state $w$ under policy $\overline{\pi}$ starting from $s$. We can assume the hitting time is finite. If $w$ is never reached, the result holds trivially, since in that case $\overline{\pi}\equiv \pi_w$ and $M^{\pi_w}_s(w)=0$. By construction, $H_s^{\overline{\pi}}(w) = H_s^{\pi_w}(w)$ almost surely for all $s$. 

    Using the law of total probability and definition \eqref{eq:SM_switch_agnostic} of switching successor measure, we have:
    \begin{align}
        M^{ \overline{\pi} }_s(s') =  
        \sum_{k=0}^\infty \mathbb{P}(H_s^{\pi_w}(w)=k) \mathbb{E} \Big[   M^{\pi_{w}|_k \to \pi}_s(s')  \Big\vert H_s^{\pi_w}(w)=k \Big].
        \label{eq:M_pi_switch_proof}
    \end{align}
    We now analyze the inner expectation:
    \begin{align}
        \mathbb{E} &\Big[   M^{\pi_{w}|_k \to \pi}_s(s')  \Big\vert H_s^{\pi_w}(w)=k \Big] = \sum_{t=0}^{\infty} \gamma^{t} \mathbb{P}_{\overline{\pi}}(s_t=s'\vert s_0=s, H_s^{\pi_w}(w)=k ) 
        \nonumber \\ &=
        \sum_{t=0}^{k-1} \gamma^{t} \mathbb{P}_{\pi_w}(s_t=s'\vert s_0=s, H_s^{\pi_w}(w)=k ) +   \sum_{t=k}^{\infty} \gamma^{t} \mathbb{P}_{\pi}(s_t=s'\vert s_0=s, H_s^{\pi_w}(w)=k )
        \nonumber \\ &=
        \sum_{t=0}^{k-1} \gamma^{t} \mathbb{P}_{\pi_{w}}(s_t=s'\vert s_0=s, H_s^{\pi_w}(w)=k ) + \gamma^{k} \underbrace{\sum_{t=0}^{\infty} \gamma^{t} \mathbb{P}_{\pi}(s_{t+k}=s'\vert s_k=w)}_{M^\pi_w(s')}.
        \label{eq:decomp_Mpi_proof}
    \end{align}
    where for the first term we use that before reaching $w$ the agent follows $\pi_w$, and for the second term we use the stationary property of MDPs.
    Now, we consider only the part in \eqref{eq:M_pi_switch_proof} corresponding to the second term:
    \begin{align*}
         \sum_{k=0}^\infty \mathbb{P}(H_s^{\pi_w}(w)=k)  \gamma^{k} M_w^\pi(s') = \mathbb{E}\!\left[ \gamma^{H_s^{\pi_w}(w)} \right] \, M^{\pi}_w(s')
    \end{align*}
    This closely resembles the result in the proof of Lemma 3.3 in \cite{myers2024learning}. Next, we consider the first term corresponding to $t< k$ in \eqref{eq:decomp_Mpi_proof}. We have:
    \begin{align*}
    &\sum_{k=0}^{\infty} \mathbb{P}\!\left(H_s^{\pi_w}(w) = k\right)
    \sum_{t=0}^{k-1}
    \gamma^t
    \mathbb{P}_{\pi_w}(s_t = s' \mid s_0 = s, H_s^{\pi_w}(w) = k)
    \\ &=
    \sum_{k=0}^{\infty} 
    \sum_{t=0}^{k-1}
    \gamma^t
    \mathbb{P}_{\pi_w}(s_t = s', H_s^{\pi_w}(w)=k \mid s_0 = s)
     =
    \sum_{t=0}^{\infty}
    \gamma^t
    \mathbb{P}_{\pi_w}(s_t = s', H_s^{\pi_w}(w) > t \mid s_0 = s) = (\star ),
    \end{align*}
    where we apply Fubini's theorem. We then decompose the joint probability:
    \begin{align*}
        \mathbb{P}_{\pi_w}(s_t = s', H_s^{\pi_w}(w) > t \mid s_0 = s)  = \mathbb{P}_{\pi_w}(s_t = s' \mid s_0 = s)  - \mathbb{P}_{\pi_w}(s_t = s', H_s^{\pi_w}(w) \leq t \mid s_0 = s).
    \end{align*}
    Summing over time steps gives:
    \begin{align*}
    (\star)
    &= \sum_{t=0}^\infty \gamma^t  \mathbb{P}_{\pi_w}(s_t = s' \mid s_0 = s) - \sum_{t=0}^\infty \gamma^t \mathbb{P}_{\pi_w}(s_t = s', H_s^{\pi_w}(w) \leq t \mid s_0 = s)\\
    &= M^{\pi_w}_s(s') - \mathbb{E}\!\left[ \gamma^{H_s^{\pi_w}(w)} \right]
        \, M^{\pi_w}_w(s')
    \end{align*}
    where the second term was treated analogously as in equations leading to \eqref{eq:decomp_Mpi_proof}, since the hitting time precedes the current time. 

    Next, we invoke Lemma \ref{lemma:Myers_hitting_time} to obtain $\mathbb{E}\!\left[ \gamma^{H_s^{\pi_w}(w)} \right] = M^{\pi_w}_s(w)/ M^{\pi_w}_w(w)$ for any $w$ for which $M^{\pi_w}_w(w)>0$. Note that $M^{\pi_w}_w(w)\geq 1$ under the standard definition of successor measures including the $t=0$ term, and is therefore strictly positive. Summing the two components in \eqref{eq:decomp_Mpi_proof} gives:
    \begin{align*}
        M^{\overline{\pi}}_s(s') = M^{\pi_w}_s(s') - \frac{M^{\pi_w}_s(w)}{M^{\pi_w}_w(w)} M^{\pi_w}_w(s') + \frac{M^{\pi_w}_s(w)}{M^{\pi_w}_w(w)} M^{\pi}_w(s') 
    \end{align*}
    Finally, identifying $\overline{\pi}$ with $\pi_w\to \pi$ completes the proof.
\end{proof}

%% file: sections/10_appC.tex
\section*{C Experimental details}

\subsection*{C.1 Environment setups and datasets}

\textbf{Discrete environment.} 

We provide additional details on the environment and experiments shown in Figures \ref{fig:bigfigure} and \ref{fig:bigfigure_learning_based}. We consider a discrete variant of the Medium maze environment from OGBench, consisting of 104 states and 5 actions (stay or move in one of the four cardinal directions), with deterministic transitions and discount factor $\gamma=0.98$. The reward function used in these experiments is illustrated in Figure \ref{fig:bigfigure}(a).

In Figure \ref{fig:bigfigure}, we visualize quantities computed exactly from the known dynamics. In particular, panel (b) shows the exact value function under the optimal policy for this reward function. Since the environment is small and fully discrete, these quantities can be computed exactly. This also allows us to empirically validate the results of Theorem \ref{thm:SM_pi_to_pi} and Corollary \ref{cor:main_hierarchical_SM} in this setting.

For Figure \ref{fig:bigfigure_learning_based}, we use the same environment but replace exact computation with a learning-based approach. Specifically, we apply the representation learning component of FB $\pi$-Switch to learn successor measures and visualize the learned quantities. Evaluation in this setting is qualitative, based on comparisons between learned quantities and their exact counterparts.

Given the low-dimensional state and action spaces, we use a latent dimension of 24, which is sufficient for this setting. We construct an offline dataset of 100,000 trajectories, each with a maximum length of 100 steps, collected using a uniformly random policy. Training is performed for 250 epochs, each consisting of 1000 gradient steps, with a batch size of 32. We use the Adam optimizer with learning rate $10^{-3}$. During training, we gradually decrease the proportion of latents derived from states relative to those sampled uniformly from a sphere. These experiments are implemented in PyTorch and executed on CPU.

\textbf{Continuous environments.} 

In Section \ref{sec:experiments}, we use AntMaze environments from OGBench \citep{park2024ogbench}. Specifically, we consider the Medium, Large, Giant, and Teleport variants, which share the same 29-dimensional state space and 8-dimensional action space, with actions bounded in $[-1,1]$, but differ in maze size and task difficulty. We use raw observations without normalization.

We use the \texttt{antmaze-giant-navigate-v0} dataset from OGBench. This dataset is generated by a noisy expert policy that repeatedly navigates to randomly sampled goals, resulting in diverse trajectories that cover the state space, and is commonly used as a standard benchmark dataset \citep{park2024ogbench,zheng2026can}. For the Medium, Large, and Teleport environments, episodes have a maximum length of 1000 steps, and the dataset consists of 1000 episodes (approximately 1M transitions). For the Giant environment, we use 500 episodes of length 2000 (also approximately 1M transitions). In our experiments, episodes are truncated at $1000$ steps. In goal-conditioned evaluations, episodes terminate early upon reaching a neighborhood of the goal, defined as entering a radius of 0.5 in the (x,y) plane around the goal.

For the general reward setting, we design task-specific reward functions as shown in Figure~\ref{fig:rewards_antmaze}. These tasks are defined over a discretized grid of the maze, with rewards taking values in $\{0,+1,-1,+5,+10\}$. Typically, each task includes a high-reward region, along with additional regions of moderately positive or negative reward. A reward is assigned whenever the agent enters a grid cell with a non-zero reward. We restrict this setting to the Large and Giant environments. The Medium environment is excluded due to its small size, which would lead to overly dense rewards when multiple reward regions are introduced. The Teleport environment is also omitted, as the resulting behaviors are less interpretable.

\begin{figure}[!htbp]
    \centering
    \begin{subfigure}[t]{0.4\textwidth}
        \centering
        \includegraphics[width=0.98\linewidth]{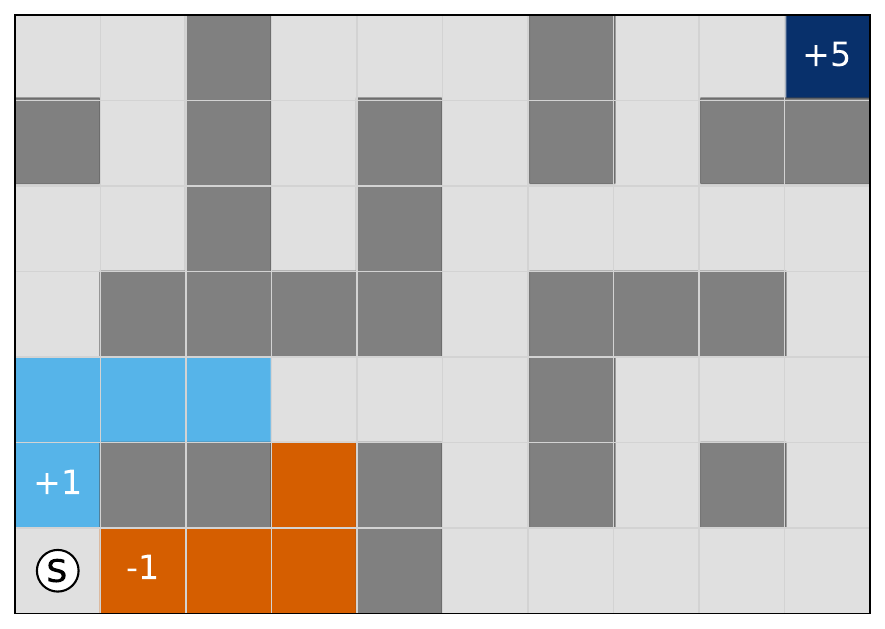}
    \end{subfigure}
    \begin{subfigure}[t]{0.58\textwidth}
        \centering
        \includegraphics[width=0.73\linewidth]{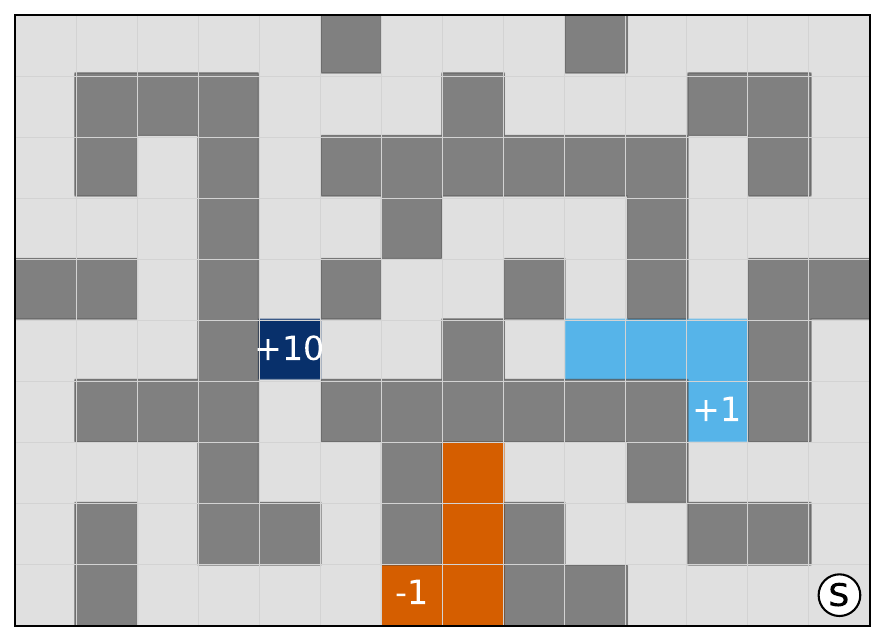}
    \end{subfigure}

    \begin{subfigure}[t]{0.4\textwidth}
        \centering
        \includegraphics[width=0.98\linewidth]{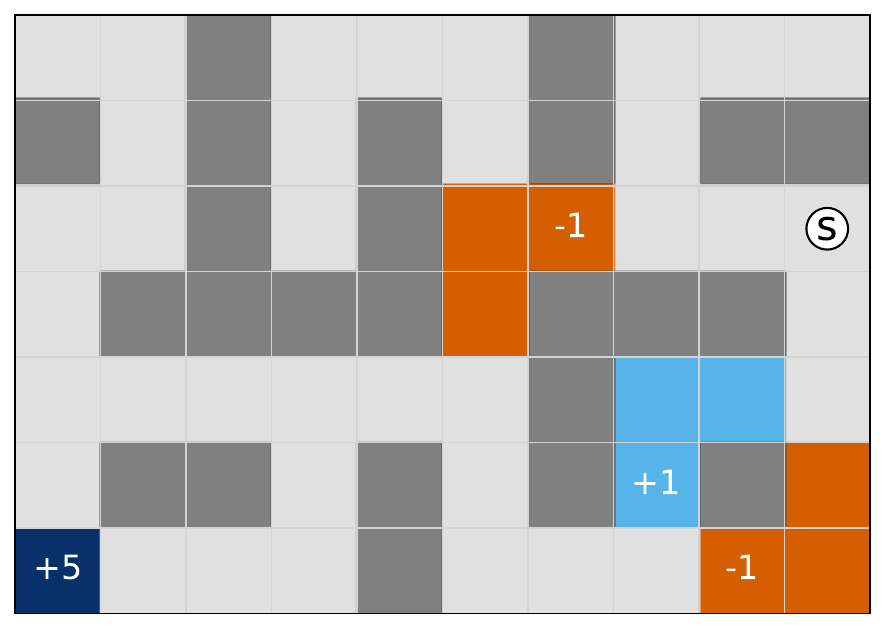}
    \end{subfigure}
    \begin{subfigure}[t]{0.58\textwidth}
        \centering
        \includegraphics[width=0.73\linewidth]{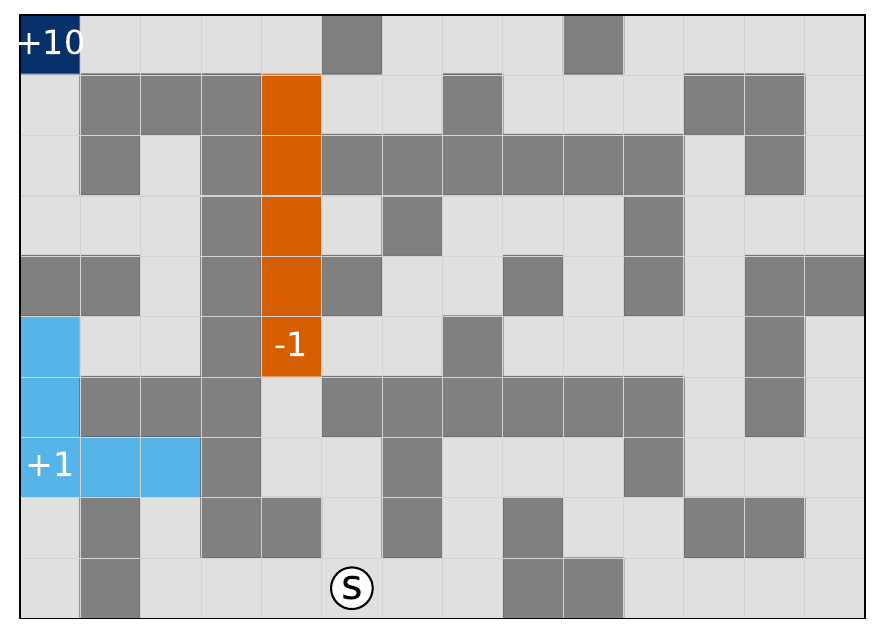}
    \end{subfigure}

    \begin{subfigure}[t]{0.4\textwidth}
        \centering
        \includegraphics[width=0.98\linewidth]{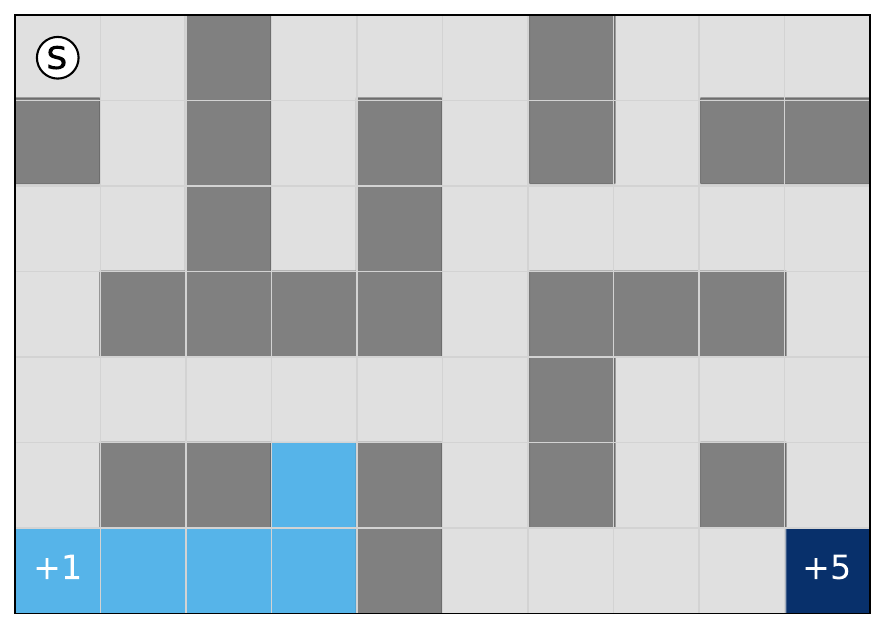}
    \end{subfigure}
    \begin{subfigure}[t]{0.58\textwidth}
        \centering
        \includegraphics[width=0.73\linewidth]{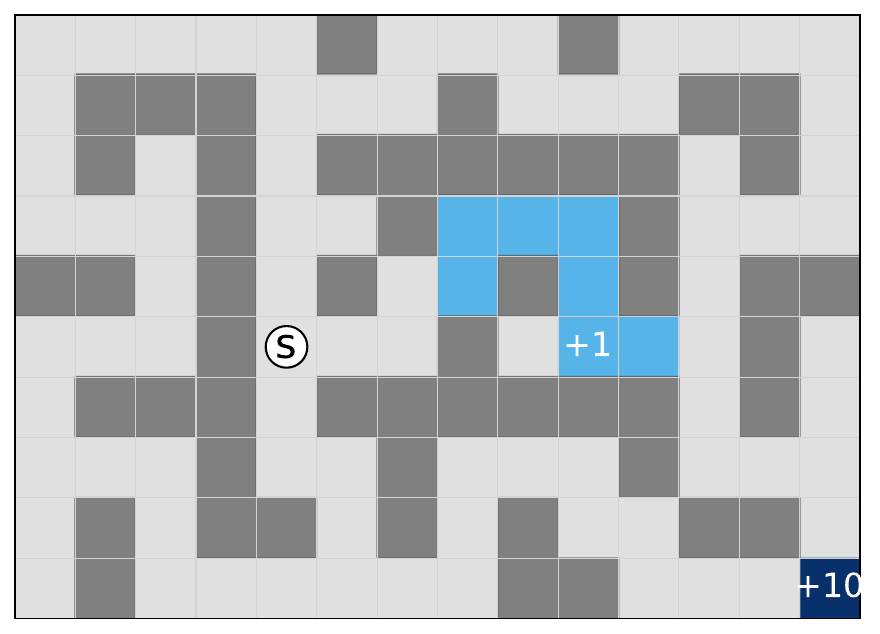}
    \end{subfigure}

    \begin{subfigure}[t]{0.4\textwidth}
        \centering
        \includegraphics[width=0.98\linewidth]{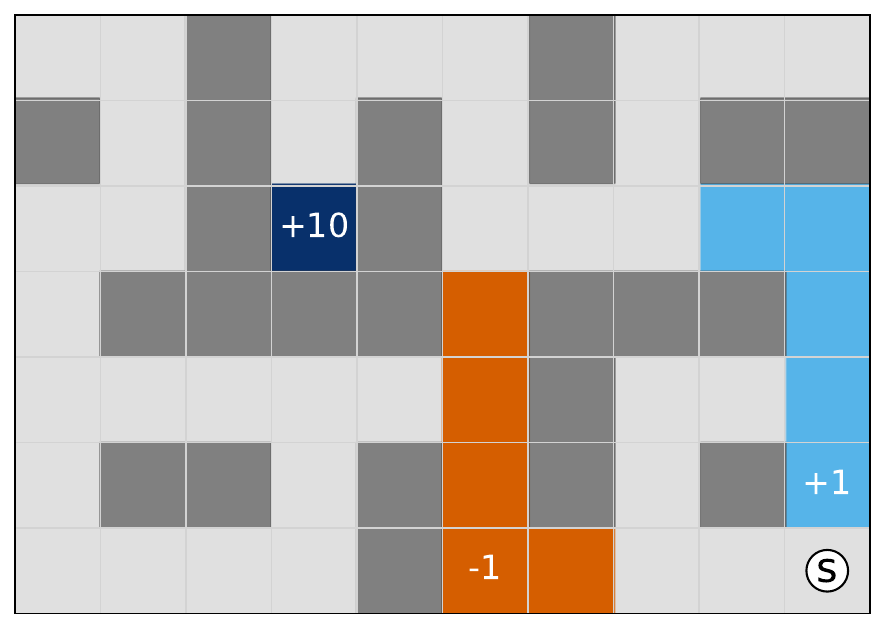}
    \end{subfigure}
    \begin{subfigure}[t]{0.58\textwidth}
        \centering
        \includegraphics[width=0.73\linewidth]{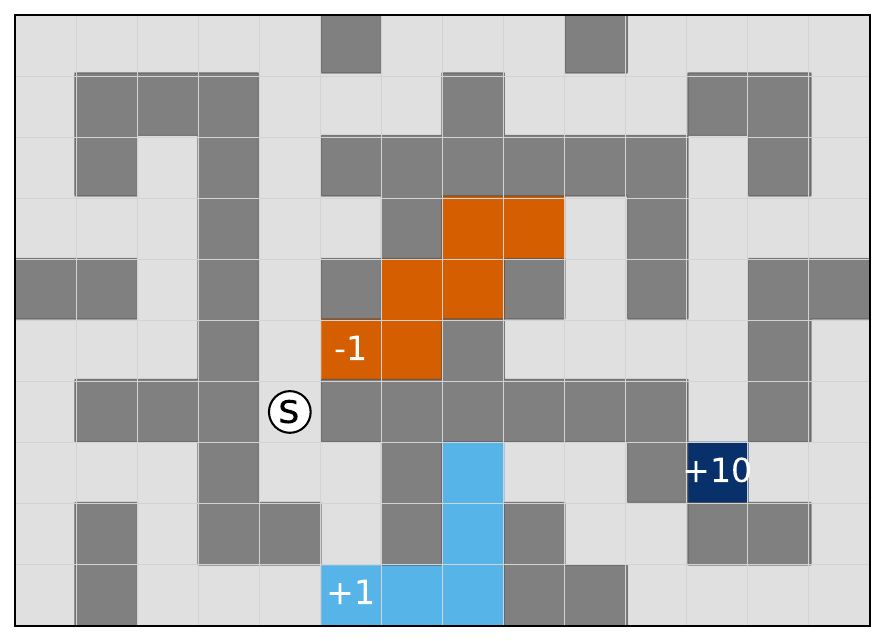}
    \end{subfigure}
    \begin{subfigure}[t]{0.4\textwidth}
        \centering
        \includegraphics[width=0.98\linewidth]{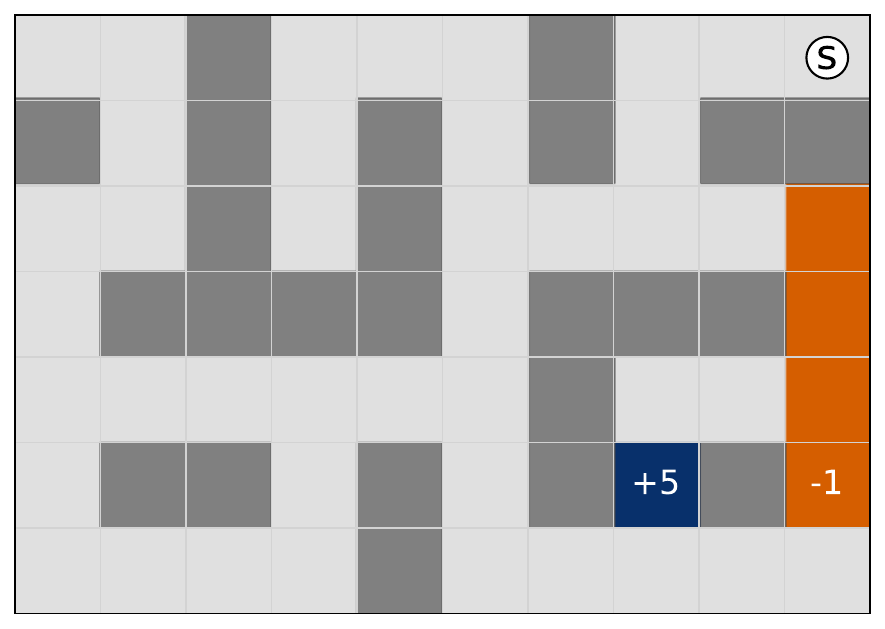}
    \end{subfigure}
    \begin{subfigure}[t]{0.58\textwidth}
        \centering
        \includegraphics[width=0.73\linewidth]{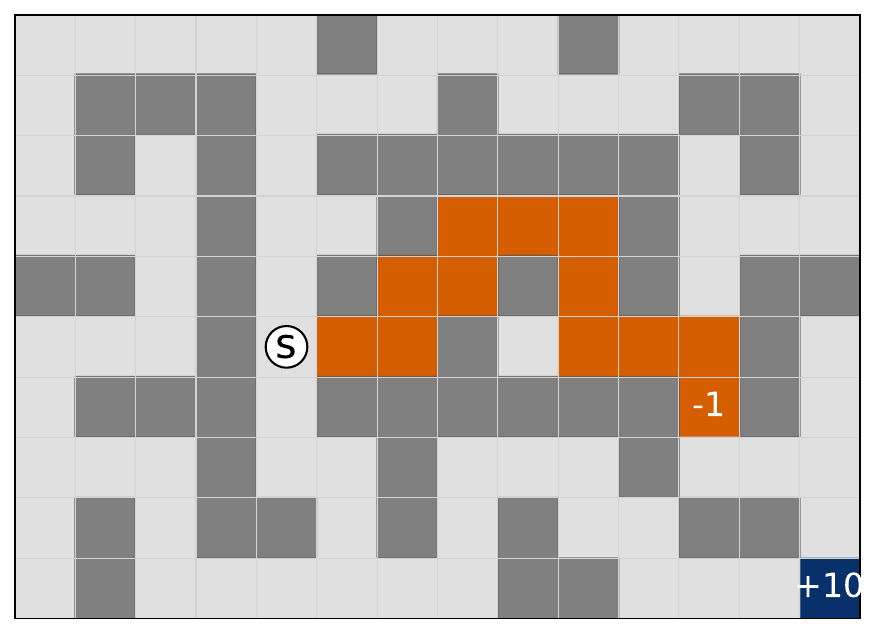}
    \end{subfigure}
    \caption{Tasks 1–5 (top to bottom) for Large AntMaze (left) and Giant AntMaze (right). Initial states are indicated by a white circle labeled “s”. Rewards take values in $\{0,+1,-1,+5,+10\}$.}
    \label{fig:rewards_antmaze}
\end{figure}

\subsection*{C.2 Model architecture}
We use fully connected neural networks for all model components, as observations in AntMaze are low-dimensional state vectors. All networks are implemented as multi-layer perceptrons (MLPs) with hidden dimensions specified per method (typically 3–4 layers of width 512) and GELU activations. Unless stated otherwise, layer normalization is not applied.

\textbf{State and goal encoding.}
We do not use a separate encoder module; instead, state and goal inputs are concatenated and fed directly into the networks. For goal-conditioned components, goals are either raw states (in GCRL baselines) or latent representations inferred via the backward representation network. Latent vectors are normalized to lie on a sphere of radius $\sqrt{d}$, where $d$ is the latent dimension.

\textbf{Representation networks.}
We use forward and backward representation networks, both parameterized as MLPs. The backward representation maps states to latent vectors, while the forward representation maps state–latent (and optionally state–action–latent) tuples to latent predictions. The forward representation uses an ensemble of size two (double-critic setup), while the backward representation uses a single network. In the hierarchical FB variant, an additional marginal forward representation network is introduced and trained to match the action-marginalized version of the forward representation.

\textbf{Actor networks.}
Policies are modeled as Gaussian distributions with diagonal covariance. The mean is produced by an MLP, while the standard deviation is state-independent and fixed (constant across dimensions). For low-level policies, actions are optionally passed through a $\tanh$ squashing function to enforce bounds in $[-1, 1]$. The high-level policy outputs latent subgoals and follows the same Gaussian parameterization, without squashing.

\textbf{Parameter sharing and freezing.}
In all hierarchical variants, the forward representation, backward representation, and low-level actor are initialized from pretrained models and kept frozen during training. Only the newly introduced components (e.g., high-level actor and marginal forward representation) are updated. This is implemented via parameter masking in the optimizer.

\textbf{Critic structure.}
The forward representation serves as a latent-space critic. We use a double critic architecture (ensemble of size two) to improve stability, analogous to double Q-learning. The outputs are vector-valued (dimension equal to the latent size), and scalar values are obtained via inner products with reward latents.

\subsection*{C.3 Training protocol}

\textbf{General training setup.}
We build on implementations of baseline algorithms from \cite{park2024ogbench, zheng2026can}. Training is performed fully offline on fixed datasets, without any environment interaction during learning. All methods use pre-collected trajectory datasets provided by OGBench, and we do not use any form of online data collection or replay buffer. Mini-batches are sampled uniformly from the dataset for gradient-based optimization. We use double target networks to alleviate maximization bias. The two value estimates are aggregated in a model-dependent manner, either by taking their mean or a minimum over the ensemble, depending on the specific algorithm.

\textbf{Terminal state masking.}
In standard goal-conditioned reinforcement learning \citep{ghosh2023reinforcement, park2023hiql}, one typically estimates quantities of the form $M_s(g)$, which implicitly correspond to $M_s^{\pi_g}(g)$. In this case, the queried state coincides with the goal $(s'=g)$, so defining rewards and terminal masking with respect to $g$ is consistent. However, in a more general setting when learning $M_s^{\pi_g}(s')$, applying masking based on either the goal $g$ or the queried state $s'$ leads to a mismatch: the policy is conditioned on $g$, while the Bellman recursion and termination are defined with respect to a different state variable. As a result, the induced recursion no longer corresponds to a consistent successor measure under $\pi_g$.

Thus, in our setting we do not use terminal-state masking. This design choice is motivated by our objective of supporting general reward functions beyond goal-conditioned settings, where termination is not tied to reaching a goal state. As a result, performance is lower than in standard goal-conditioned benchmarks; for example, OGBench reports HIQL success rates of $96$, $91$, $65$, and $42$ percent on AntMaze Medium, Large, Giant, and Teleport, respectively, compared to our HIQL agents, which achieve $93$, $73$, $5$, and $42$ percent.

\textbf{Reward transformations.}
We do not apply any reward transformations in our experiments such as replacing the sparse indicator reward at the goal with a $-1$ per-step reward and a $0$ reward at the goal. This choice is consistent with our goal embedding framework: in the goal-conditioned setting, the reward embedding of reaching a goal $g$ corresponds directly to $B(g)$, whereas dense reward formulations would require aggregating contributions over all non-goal states, which complicates the representation. Finally, we set the reward temperature to $0$, ensuring a consistent treatment of rewards and allowing negative rewards to be properly accounted for by the agents.

\textbf{Baseline implementations.}
We use baseline implementations from \cite{park2024ogbench} and \cite {zheng2026can}. All baselines are evaluated using the same offline dataset, training budget and evaluation protocol as our method.

\emph{HIQL and its variants:} We follow \cite{park2024ogbench} and keep all hyperparameters unchanged, including \texttt{subgoal\_steps}=25 and \texttt{rep\_dim}=10. Value goals are sampled with mixing proportions (current state, future state in trajectory, random state) = (0.2, 0.5, 0.3), while actor goals use (0, 1.0, 0). When training HIQL without high level policy ($-\pi^h$) in the ablation study in Table \ref{tab:full_appendix_results_pointmaze}, the actor is trained using the same goal sampling distribution as the value function, i.e., (0.2, 0.5, 0.3). For the bilinear variant, we use an inner-product value function of the form $\phi(s,g)^\top \psi(g)$ with embedding dimension 128. The embedding $\psi(g)$ is normalized to unit norm and, no orthonormality regularization is applied. For the variant with intentions, with probability of 0.5 the goal and intention are set to the same state, and we learn a value function $V(s,g,s')$. Intention goals follow the same sampling scheme as value goals. For direction error, we take the minimum over the two value estimates, while for TD-error we use their average.

\emph{ICVF:} We follow \cite{zheng2026can} and use their default hyperparameters. Intention goals are sampled using the same mixture as value goals, with distribution (0, 0.625, 0.375) together with geometric sampling for future goals. In actor training, we include an additional behavior cloning (BC) loss with coefficient $\alpha=0.3$ for Medium and Teleport, and $\alpha=3.0$ for Large and Giant.

\emph{OneStep FB:} We use BC coefficient $\alpha=0.1$ for Teleport and $\alpha=0.03$ for the remaining environments. We set the reward temperature to $0$ instead of $10.0$. All other hyperparameters follow \cite{zheng2026can}.

\emph{FB:} We follow \cite{zheng2026can} but introduce several modifications for consistency in our setting. We remove actions from the backward representation, remove the $(1-\gamma)$ term in the diagonal loss to match the original FB formulation, set the reward temperature to $0$, use an orthonormality coefficient of $1.0$, set BC coefficient $\alpha=0.03$, and normalize the Q-loss used in actor training. 

\textbf{Marginalization loss of FB $+\pi^h$.}
As discussed in Section \ref{subsec:post_tuning_FB}, the standard FB formulation learns a forward representation of the form $F(s,a,z)$, which explicitly depends on actions. To enable the use of the switching advantage approximation in \eqref{eq:def_A_FB}, we introduce a marginalized forward representation $F'(s,z)$ by averaging over actions. This is enforced via the following objective:
\begin{align*}
    \mathcal{L}_{\text{marg}}(F') =
    \mathbb{E}_{s,z} \left[
    \left(
    F'(s,z) - \frac{1}{N}\sum_{i=1}^N F(s,a_i,z)
    \right)^2
    \right],
    \quad a_i \sim \pi_z(\cdot \mid s)
\end{align*}
where N=8 denotes the number of actions sampled from the policy. For constructing the latent variables used in this loss, we set \texttt{latent\_mix\_prob}=0.5, corresponding to an equal probability of sampling either embeddings of permuted observations or random vectors drawn uniformly from the surface of a sphere with radius $\sqrt{d}$. The resulting marginalized representation is then used to train the high-level policy according to the objective in \eqref{eq:L_plan}.

\textbf{Latent space sampling.}
For the high-level actor loss in FB $\pi$-Switch and FB $+\pi^h$, subgoals (high-actor targets) are obtained via geometric sampling over indices of future states in the trajectory conditioned on the current state, and then mapped into the latent space via backward networks. High-actor goals are sampled uniformly from future states in the trajectory (from the next state until episode termination), mixed with isotropic Gaussian samples that are normalized to lie on a sphere, with mixing ratio \texttt{critic\_latent\_mix\_prob}=0.5.

The policy is trained using a clipped exponential advantage weighting scheme, where advantages are passed through an exponential function with inverse temperature \texttt{high\_alpha}, and clipped at 5.0 for numerical stability. The corresponding hyperparameter values are reported in Tables \ref{tab:hyperparameters_hilswitch} and \ref{tab:hyperparameters_hfb}.

\subsection*{C.4 Hyperparameters}
In Table \ref{tab:hyperparameters}, we report the hyperparameters shared across all methods. In Tables \ref{tab:hyperparameters_hilswitch} and \ref{tab:hyperparameters_hfb}, we report all hyperparameters of our hierarchical methods, FB $\pi$-Switch and FB $+\pi^h$, to facilitate reproducibility.

\begin{table}[ht]
    \centering
    \small
    \begin{tabular}{l c l}
    \toprule
    \textbf{Hyperparameter} & \textbf{Value} & \textbf{Description} \\
    \midrule
    \texttt{num\_timesteps} & 1M & total training steps \\
    \texttt{episode\_length} & 1,000  & maximum steps per episode \\
    \texttt{discount} & 0.995 (giant), 0.99 (other) &  discount factor $\gamma$\\
    \texttt{batch\_size} & 1024 & minibatch size \\
    \texttt{latent\_dim} & 128 (other), 10 (HIQL)  &  latent representation size \\
    \texttt{gc\_negative} & False &
        \begin{tabular}[t]{@{}l@{}}
        choice of goal reward parameterization:\\
        $\{0 \text{ at goal}, -1 \text{ otherwise}\}$ (True) or \\
        $\{1 \text{ at goal}, 0 \text{ otherwise}\}$ (False)
        \end{tabular} \\
    \bottomrule\\
    \end{tabular}
    \caption{General/shared hyperparameters across agents used in our experiments.}
    \label{tab:hyperparameters}
    \vspace{-10pt}
\end{table}

\begin{table}[ht]
\centering
\small
\begin{tabular}{l c l}
\toprule
\textbf{Hyperparameter} & \textbf{Value} & \textbf{Description} \\
\midrule
    \texttt{lr} & $3 \times 10^{-4}$ & learning rate \\
    \texttt{actor\_hidden\_dims} & (512, 512, 512) & actor network hidden dimensions \\
    \texttt{forward\_repr\_hidden\_dims} & (512, 512, 512) & F-representation network hidden dim. \\
    \texttt{backward\_repr\_hidden\_dims} & (512, 512, 512) & B-representation network hidden dim. \\
    \texttt{actor\_layer\_norm} & False & use layer normalization for actors \\
    \texttt{forward\_repr\_layer\_norm} & True & use layer normalization for F-representations \\
    \texttt{backward\_repr\_layer\_norm} & True & use layer normalization for B-representations \\
    \texttt{activation} & gelu & activation function \\
    \texttt{tau} & 0.005 & target network update rate \\
    \texttt{alpha} & 3.0 & AWR temperature for low-level policy \\
    \texttt{const\_std} & True & use constant std in actor \\
    \texttt{reward\_temperature} & 0.0 & reward weighting temperature \\
    \texttt{num\_zero\_shot\_samples} & 100,000 & number of zero-shot samples for latent inference \\
    \texttt{expectile} & 0.7 & expectile parameter  \\
    \midrule
    \texttt{value\_p\_curgoal} & 0.2 & prob. of using current state as value goal \\
    \texttt{value\_p\_trajgoal} & 0.5 & prob. of using future state in traj. as value goal \\
    \texttt{value\_p\_randomgoal} & 0.3 & prob. of using random state as value goal \\
    \texttt{value\_geom\_sample} & True & use geometric sampling for value goals \\
    \texttt{low\_actor\_p\_curgoal} & 0.2 & prob. of using current state as low-actor goal \\
    \texttt{low\_actor\_p\_trajgoal} & 0.5 & prob. of using future state in traj. as low-actor goal \\
    \texttt{low\_actor\_p\_randomgoal} & 0.3 & prob. of using random state as low-actor goal \\
    \texttt{high\_actor\_p\_curgoal} & 0.0 & prob. of using current state as high-actor goal \\
    \texttt{high\_actor\_p\_trajgoal} & 1.0 & prob. of using future state in traj. as high-actor goal \\
    \texttt{high\_actor\_p\_randomgoal} & 0.0 & prob. of using random state as high-actor goal \\
    \texttt{actor\_geom\_sample} & False & use geometric sampling for actor goals \\
    \midrule
    \texttt{normalize\_latent} & True & normalize backward representations \\
    \texttt{orthonorm\_coeff} & $1 \times 10^{-4}$ & orthonormalization coefficient for B-representations \\
    \texttt{actor\_latent\_mix\_prob} & 0.5 & prob. of sampling unif. random actor latent \\
    \texttt{critic\_latent\_mix\_prob} & 0.5 & prob. of sampling unif. random critic latent \\
    \texttt{high\_alpha} & 0.1 & AWR temperature for high-level policy  \\
    \texttt{high\_adv\_threshold} & 5.0 & threshold for adv. clipping before exponentiation \\
    \bottomrule\\
    \end{tabular}
    \caption{Hyperparameters of the FB $\pi$-Switch algorithm. Abbreviations: prob. = probability, traj. = trajectory, std. = standard deviation, unif. = uniformly.}
    \vspace{-25pt}
    \label{tab:hyperparameters_hilswitch}
\end{table}

\begin{table}[ht]
\centering
\small
\begin{tabular}{l c l}
\toprule
\textbf{Hyperparameter} & \textbf{Value} & \textbf{Description} \\
\midrule
    \texttt{lr} & $1 \times 10^{-4}$ & learning rate \\
    \texttt{actor\_hidden\_dims} & (512, 512, 512, 512) & actor network hidden dimensions \\
    \texttt{forward\_repr\_hidden\_dims} & (512, 512, 512, 512) & F-representation network hidden dim. \\
    \texttt{backward\_repr\_hidden\_dims} & (512, 512, 512, 512) & B-representation network hidden dim. \\
    \texttt{actor\_layer\_norm} & False & use layer normalization for actors \\
    \texttt{forward\_repr\_layer\_norm} & False & use layer normalization for F-representations \\
    \texttt{backward\_repr\_layer\_norm} & False & use layer normalization for B-representations \\
    \texttt{activation} & gelu & activation function \\
    \texttt{tau} & 0.005 & target network update rate \\
    \texttt{alpha} & 0.03 & behavioral cloning coefficient \\
    \texttt{const\_std} & True & use constant std in actor \\
    \texttt{reward\_temperature} & 0.0 & reward weighting temperature \\
    \texttt{num\_zero\_shot\_samples} & 100,000 & number of zero-shot samples for latent inference \\
    \texttt{tanh\_squash} & True & use tanh squashing for actor outputs \\
    \texttt{normalize\_q\_loss} & True & normalize Q loss \\
    \texttt{repr\_agg} & mean & aggregation method for representations \\
    \midrule
    \texttt{high\_actor\_p\_curgoal} & 0.0 & prob. of using current state as high-actor goal \\
    \texttt{high\_actor\_p\_trajgoal} & 1.0 & prob. of using future state in traj. as high-actor goal \\
    \texttt{high\_actor\_p\_randomgoal} & 0.0 & prob. of using random state as high-actor goal \\
    \texttt{actor\_geom\_sample} & False & use geometric sampling for actor goals \\
    \midrule
    \texttt{normalize\_latent} & True & normalize backward representations \\
    \texttt{orthonorm\_coeff} & 1.0 & orthonormalization coefficient for B-representations \\
    \texttt{margin\_latent\_mix\_prob} & 0.5 & prob. of sampling unif. random latent for margin. \\
    \texttt{critic\_latent\_mix\_prob} & 0.5 & prob. of sampling unif. random critic latent \\
    \texttt{high\_alpha} & $1 \times 10^{-3}$ & AWR temperature for high-level policy \\
    \texttt{high\_adv\_threshold} & 5.0 & threshold for adv. clipping before exponentiation\\
    \texttt{num\_actions\_sampled} & 8 & number of sampled actions for marginalization \\
    \bottomrule\\
    \end{tabular}
    \caption{Hyperparameters of the FB $+\pi^h$ algorithm. Abbreviations: prob. = probability, traj. = trajectory, std. = standard deviation, unif. = uniformly, margin. = marginalization.}
    \vspace{-25pt}
    \label{tab:hyperparameters_hfb}
\end{table}

\subsection*{C.5 Evaluation protocol}
For each agent and environment, we evaluate performance over 50 episodes per task, across 5 random seeds and 5 tasks. At evaluation time, we use a deterministic policy. For goal-reaching evaluation, goals are drawn from a fixed evaluation set shared across all methods.

For each task, we precompute a fixed set of 100,000 transitions, which are relabeled according to the corresponding reward function. The agents then use these relabeled state–reward pairs to infer the latent representations. We set the reward temperature to $T=0$, since nonzero temperatures would otherwise distort the treatment of negative rewards, as negative rewards are effectively mapped to zero under the temperature-based transformation.

We remove goal noise at environment reset, since latents are computed only once and are not updated upon resets. In contrast, standard GCRL methods (e.g., HIQL) condition directly on the true (noisy) goal at each reset, which can give them an advantage over methods that operate in a fixed latent representation. This design choice ensures a fair comparison by providing access to the same goal information across methods, putting goal- and latent-conditioned approaches on equal footing.

\textbf{Performance in goal-reaching tasks.} As is standard in goal-conditioned reinforcement learning, we evaluate performance using average success rate over a set of test goals. A trajectory is considered successful if the agent reaches the goal before episode termination, where success is defined as the agent entering a 0.5-radius ball in Euclidean distance in the (x, y) plane around the goal location (for AntMaze environments). We compute success as a binary indicator per episode and report the mean success rate across all evaluated goals. Final results are averaged over multiple random seeds, and we report standard deviation across seeds to quantify variability.

\textbf{Aggregation over tasks with general rewards.} We report performance using the interquartile mean (IQM), following recommendations in reinforcement learning evaluation \citep{agarwal2021deep}. Specifically, we first compute normalized returns for each task by linearly scaling (undiscounted) episode returns using the minimum and maximum returns observed across all methods and seeds within that task, ensuring comparability across tasks with different reward scales. This normalization is performed per task before any aggregation. We then compute the interquartile mean over the pooled set of normalized returns from all tasks and seeds. In our setting, we evaluate on 5 tasks, so the IQM is computed over the combined distribution of normalized returns across these tasks.

To quantify uncertainty, we compute 95\% confidence intervals using a stratified bootstrap procedure: we resample seeds within each task independently, preserving the task structure, and then recompute the IQM over the resampled data. Repeating this procedure yields a bootstrap distribution of IQM estimates, from which we extract the 2.5th and 97.5th percentiles as the confidence interval bounds. This ensures that uncertainty reflects both variability across seeds and heterogeneity across tasks.

\subsection*{C.6 Compute details} 
All experiments are run on NVIDIA A40 GPUs, with independent runs executed in parallel across multiple devices. We use JAX-based implementations for all models and optimization. Each model is trained for up to 1M gradient steps, taking at most 2–3 hours per run depending on the method. Evaluation is performed using 5 random seeds per method, with all seeds executed independently in parallel. We use fixed hyperparameters across runs to ensure reproducibility.

%% file: sections/11_appD.tex
\section*{D Further experimental results}

\subsection*{D.1 Discrete maze environment}
In this section, we provide additional results for the discrete maze experiments introduced in Section~\ref{sec:experiments}. We begin by recalling the switching advantage defined in Equation~\eqref{eq:def_A_FB}. One of its components, corresponding to the contribution before reaching an intermediate subgoal $w$, is visualized in Figure~\ref{fig:bigfigure_learning_based}(d) and denoted $A_{\mathrm{FB}}(s,w,r) \vert_{< w}$. This term can be written as:
\begin{align*}
    A_{\mathrm{FB}}(s,w,r) \vert_{< w} = F(s,z_w)^\top z_r -  \frac{F(s,z_w)^\top z_w }{F(w,z_w)^\top z_w} F(w,z_w)^\top z_r
\end{align*}
The second component can be interpreted as the expected return obtained by following the policy associated with $z_w$ starting from state $w$. When the subgoal policy is sufficiently strong, it tends to remain near $w$, so this term effectively captures the return of staying in the subgoal region. As a result, it becomes highly concentrated on states with nonzero reward, making it difficult to estimate reliably with function approximation.

To mitigate this issue, we introduce a simplified proxy for the switching advantage:
\begin{align*}
    \widehat{A}_{\mathrm{FB}} (s,w,z) := F(s,z_w)^\top z +  \frac{F(s,z_w)^\top z_w }{F(w,z_w)^\top z_w}  F(w, z)^\top z - F(s,z)^\top z
\end{align*}

\begin{figure}[!b]
    \centering
    \begin{subfigure}[t]{0.32\textwidth}
        \centering
        \includegraphics[width=\linewidth]{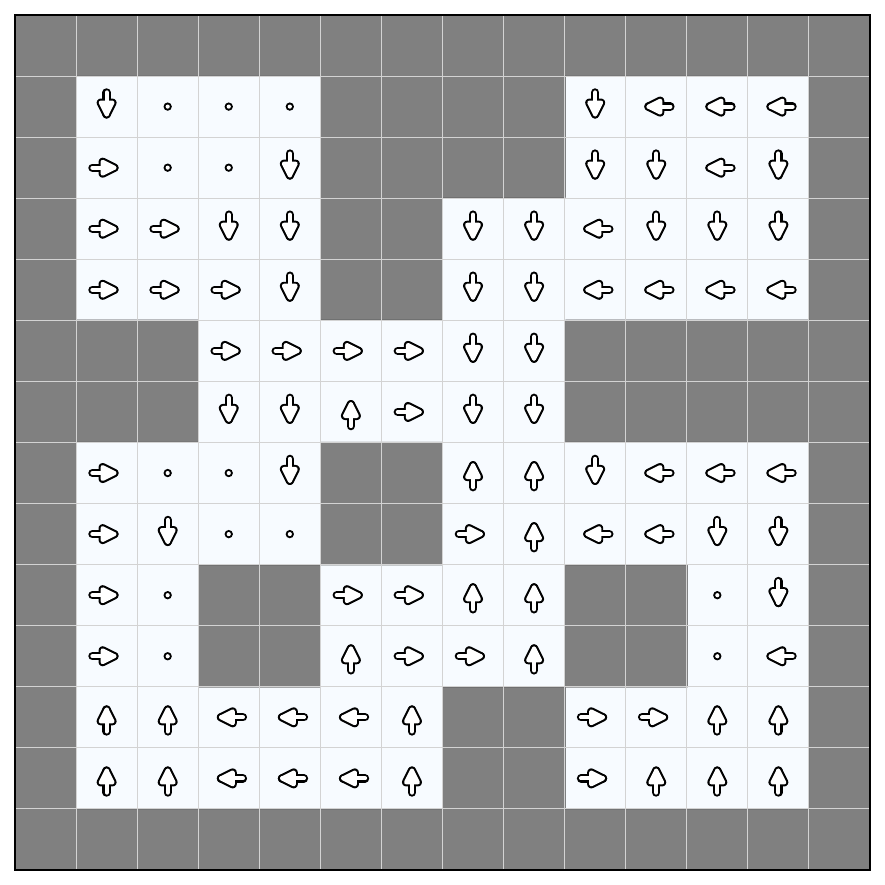}
        \caption{Actions induced by $A_{\mathrm{FB}}$}
    \end{subfigure}
    \begin{subfigure}[t]{0.32\textwidth}
        \centering
        \includegraphics[width=\linewidth]{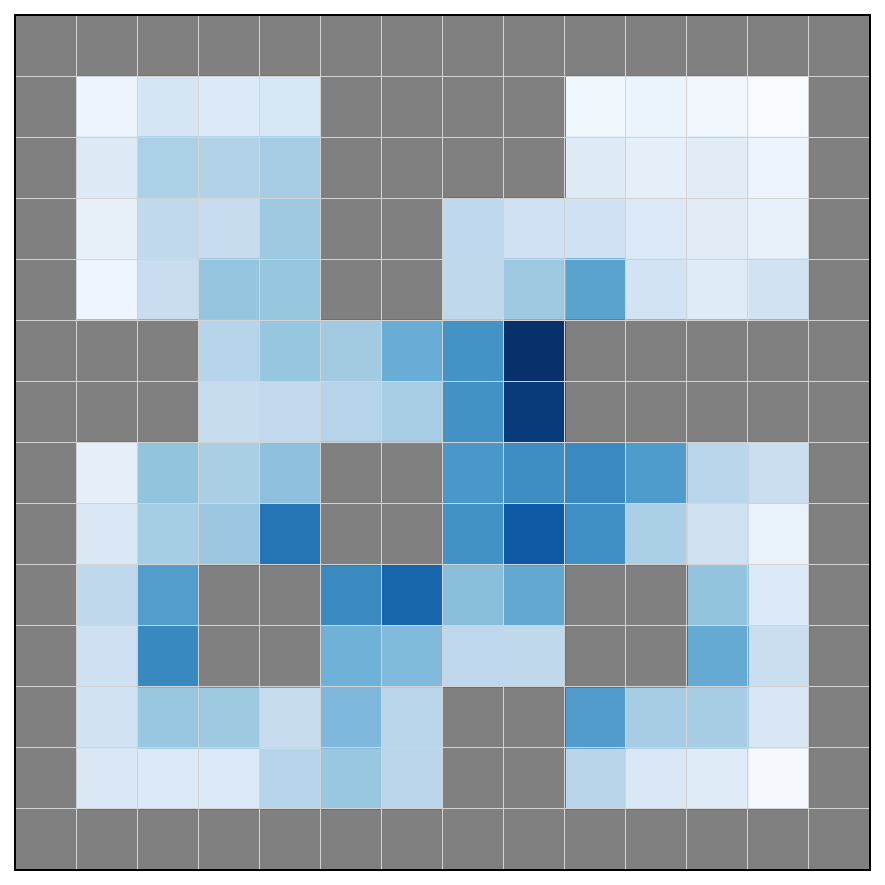}
        \caption{Subgoal distribution under $A_{\mathrm{FB}}$.}
    \end{subfigure}
    \begin{subfigure}[t]{0.32\textwidth}
        \centering
        \includegraphics[width=\linewidth]{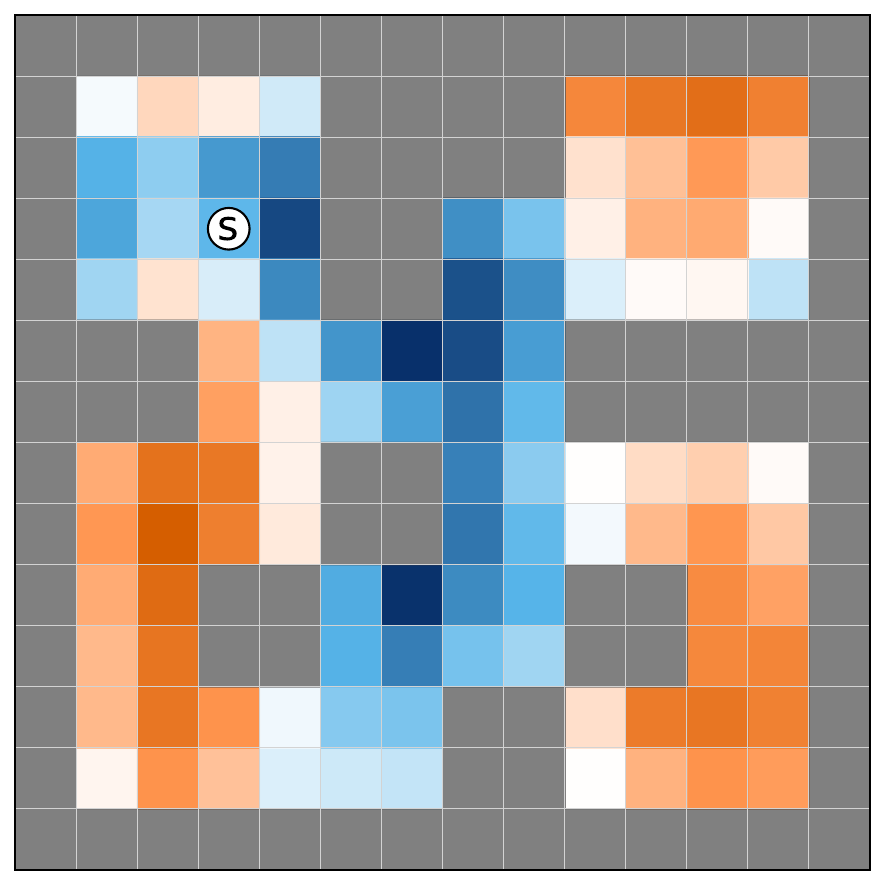}
        \caption{Switching advantage $\widehat{A}_{\mathrm{FB}}(s, \textcolor{magenta}{w}, r)$.}
    \end{subfigure}

    \vspace{0.2cm}

    \begin{subfigure}[t]{0.32\textwidth}
        \centering
        \includegraphics[width=\linewidth]{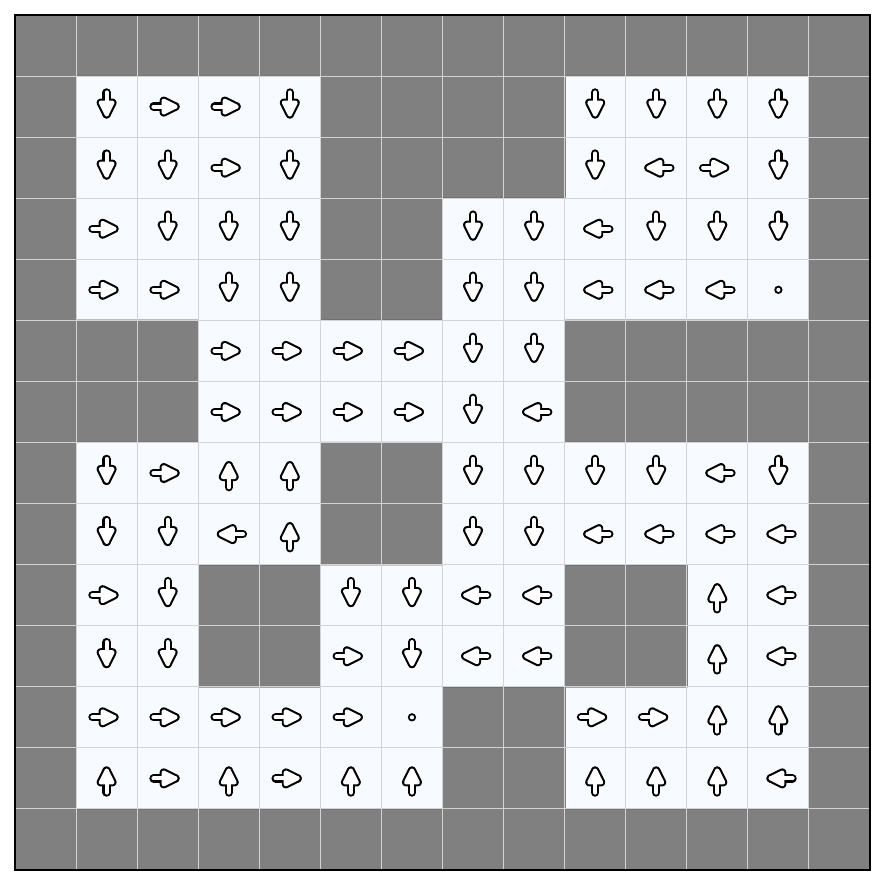}
        \caption{Actions induced by $\widehat{A}_{\mathrm{FB}}$ }
    \end{subfigure}
    \begin{subfigure}[t]{0.32\textwidth}
        \centering
        \includegraphics[width=\linewidth]{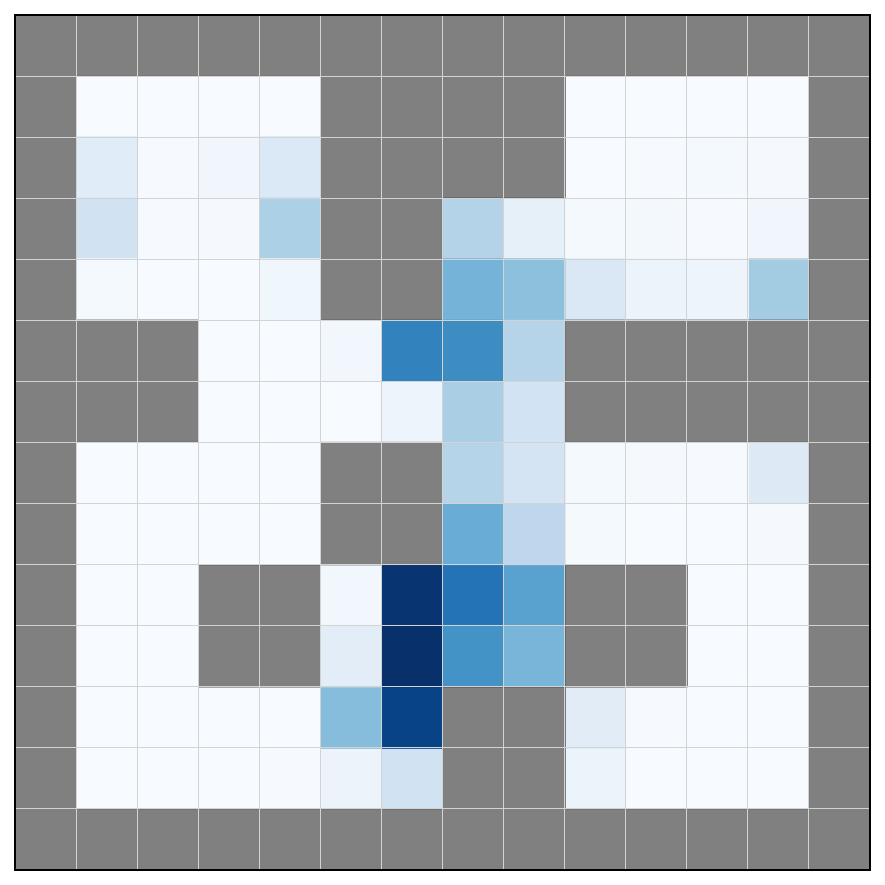}
        \caption{Subgoal distribution under $\widehat{A}_{\mathrm{FB}}$.}
    \end{subfigure}
    \begin{subfigure}[t]{0.32\textwidth}
        \centering
        \includegraphics[width=1.155\linewidth]{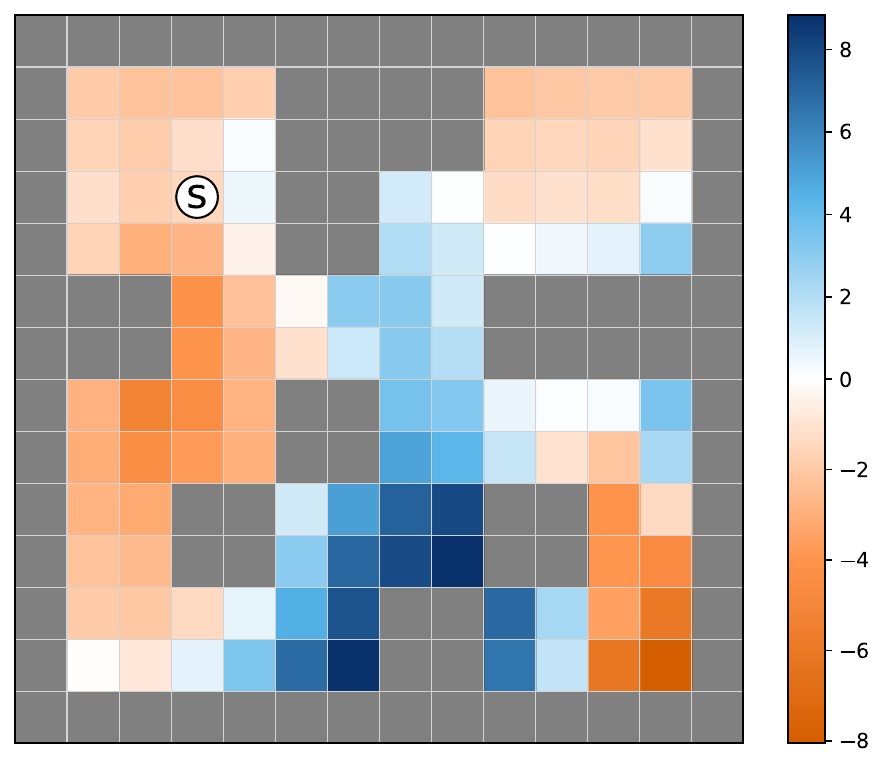}
        \caption{Difference $(\widehat{A}_{\mathrm{FB}} - A_{\mathrm{FB}})(s, \textcolor{magenta}{w}, r)$.}
    \end{subfigure}

    \caption{Effect of simplifying the switching advantage: $\widehat{A}_{\mathrm{FB}}$ (c) improves action selection (a,d) and produces more structured subgoal distributions (b,e) compared to $A_{\mathrm{FB}}$.}
    \label{fig:bigfigure_supp}
    \vspace{-0.3cm}
\end{figure}

Figure \ref{fig:bigfigure_supp} compares the original switching advantage $A_{\mathrm{FB}}$ with the proxy $\widehat{A}_{\mathrm{FB}}$. 

In Figures (a) and (d), actions are extracted by selecting, for each state, the subgoal that maximizes the corresponding advantage and then mapping it to the action that moves the agent closer to that subgoal. We observe that actions extracted from $\widehat{A}_{\mathrm{FB}}$ are more effective.

The middle column shows the induced subgoal distributions. Under $A_{\mathrm{FB}}(s,w,r)$, subgoals are diffuse and often assigned to irrelevant regions of the state space. In contrast, $\widehat{A}_{\mathrm{FB}}$ concentrates subgoals along paths leading to reward.

Finally, Figure (c) visualizes the proxy $\widehat{A}_{\mathrm{FB}}(s,w,r)$, which can be compared to $A_{\mathrm{FB}}(s,w,r)$  in Figure \ref{fig:bigfigure_learning_based}(c), while (f) shows their difference. Overall, the proxy appears to improve the alignment between subgoals and optimal behavior, while potentially simplifying learning.

\subsection*{D.2 Goal-reaching experiments}
This section provides additional results for the continuous control for goal-reaching tasks described in Section \ref{sec:experiments}. We first report the full version of Table \ref{tab:results_antmaze}, including results from the HIQL ablation study referenced in the main text.

For clarity, we briefly summarize the variants. FB refers to the standard FB algorithm. FB $+\pi^h$ augments this model with a learned high-level policy, as described in Section \ref{subsec:post_tuning_FB} and Appendix C.3. Vanilla FB $\pi$-Switch denotes our proposed method, while FB $\pi$-Switch $-\pi^h$ corresponds to its non-hierarchical version. Similarly, HIQL $-\pi^h$ removes the high-level policy from HIQL. Subsequent rows introduce additional modifications in an additive manner: adding a bilinear representation into value function, and then using intention-based value learning, where reward and policy are not necessarily aligned.

The full results are shown in Table \ref{tab:full_appendix_results_pointmaze}.

\begin{table}[ht]
    \centering
    \begin{tabular}{llcccc}
    \toprule
    \textbf{Algorithm} & \textbf{Variant}
    & \textbf{Medium} 
    & \textbf{Large} 
    & \textbf{Giant} 
    & \textbf{Teleport} \\
    \midrule

    \multirow{4}{*}{HIQL}
    & vanilla 
    & {93 \scriptsize $\pm 1$} 
    & {73 \scriptsize $\pm 6$} 
    & {5 \scriptsize $\pm 3$} 
    & {42 \scriptsize $\pm 4$} \\

    & \ \ \ \  $-\pi^h$   
    & {73 \scriptsize $\pm 3$} 
    & {23 \scriptsize $\pm 4$} 
    & {0 \scriptsize $\pm 0$} 
    & {32 \scriptsize $\pm 6$} \\

    &  \ \ \ \   \ \ \ \  $+ $ bilinear
    & {45 \scriptsize $\pm 6$} 
    & {20 \scriptsize $\pm 3$} 
    & {0 \scriptsize $\pm 0$} 
    & {14 \scriptsize $\pm 6$} \\

    &  \ \ \ \   \ \ \ \   \ \ \ \   $+$ intentions 
    & {66 \scriptsize $\pm 6$} 
    & {20 \scriptsize $\pm 4$} 
    & {1 \scriptsize $\pm 1$} 
    & {11 \scriptsize $\pm 3$} \\

    \midrule

    ICVF 
    &  
    &  {46 \scriptsize $\pm 37$} &  {28 \scriptsize $\pm 3$} &  {0 \scriptsize $\pm 0$} &  {23 \scriptsize $\pm 6$} \\

    \midrule

    One-Step FB
    &  
    &  {35 \scriptsize $\pm 7$} &  {23 \scriptsize $\pm 18$} &  {0 \scriptsize $\pm 0$} &  {8 \scriptsize $\pm 2$} \\
    
    \midrule

    \multirow{2}{*}{FB}
    & vanilla 
    & {47 \scriptsize $\pm 1$} 
    & {21 \scriptsize $\pm 5$} 
    & {0 \scriptsize $\pm 0$} 
    & {25 \scriptsize $\pm 7$} \\

    & \ \ \ \  $+\pi^h$ 
    & {50 \scriptsize $\pm 12$} 
    & {20 \scriptsize $\pm 5$} 
    & {0 \scriptsize $\pm 0$} 
    & {13 \scriptsize $\pm 7$} \\

    \midrule

    \multirow{2}{*}{FB $\pi$- Switch}
    & vanilla 
    & { 87\scriptsize $\pm 6$} 
    & { 66 \scriptsize $\pm 9$} 
    & { 1 \scriptsize $\pm 1$} 
    & { 40 \scriptsize $\pm 6$} \\

    & \ \ \ \  $-\pi^h$ 
    & {70 \scriptsize $\pm 4$} 
    & {37 \scriptsize $\pm 7$} 
    & {1 \scriptsize $\pm 1$} 
    & {29 \scriptsize $\pm 4$} \\
    \bottomrule\\
    \end{tabular}
    \caption{Average (binary) success rate ($\%$) across the five test-time goals with \texttt{navigate-v0} datasets from \cite{park2024ogbench}. The results are averaged over 50 episodes and 5 seeds, and we report the standard deviations after the $\pm$ sign. Note that operations in column Variant are additive - for example, entries in row +intentions correspond to HIQL algorithm without high-level policy, with bilinear critic and with intentions.}
    \label{tab:full_appendix_results_pointmaze}
    \vspace{-10pt}
\end{table}

A useful way to interpret these results is to note that several HIQL variants progressively move closer to the non-hierarchical version of FB $\pi$-Switch in terms of architectural choices. The remaining differences are primarily due to representation design: HIQL still operates partially in the state space and enforces strict unit normalization of goal representations, whereas FB $\pi$-Switch operates fully in latent space and uses the orthonormalization objective from the FB framework. Under these more aligned settings, FB $\pi$-Switch tends to achieve stronger performance than the corresponding HIQL-based variants.

The effect of hierarchy is also consistent across methods. Adding a high-level policy leads to clear performance improvements in both HIQL and FB $\pi$-Switch. That said, some differences remain. For example, the non-hierarchical HIQL variant degrades more noticeably on the Large maze, while HIQL slightly outperforms FB $\pi$-Switch on average in the same environment. One possible explanation is that success rates above roughly 70\% on the Large AntMaze tasks are difficult to exceed, creating a performance ceiling that affects both approaches.

Figure \ref{fig:traj_all_methods_GCRL} shows how high-level subgoal selection evolves during an episode for HIQL, FB $\pi$-Switch, and FB $+\pi^h$, highlighting differences in the learned high-level policy mentioned around Figure \ref{fig:hierarchies_qualitative_difference} in Section \ref{sec:experiments}. For a fixed goal state $g$, we compute $z_g = B(g)$, and for all discretized states $x$ in the maze we compute $z_x=B(x)$. At each time step, we sample $z\sim \pi^h(\cdot \vert s,z_g)$, and visualize $\langle z_x, z \rangle $ at the corresponding spatial location $x$.

Figures \ref{fig:zero_shot_methods_GCRL_value_fcns_appendix} and \ref{fig:HIQL_variants_GCRL_value_fcns_appendix} further visualize the learned successor measures for all models reported in Table \ref{tab:full_appendix_results_pointmaze}.

\newpage




\begin{figure}[ht]
    \centering

    \begin{subfigure}[t]{\linewidth}
        \centering
        \includegraphics[width=\linewidth]{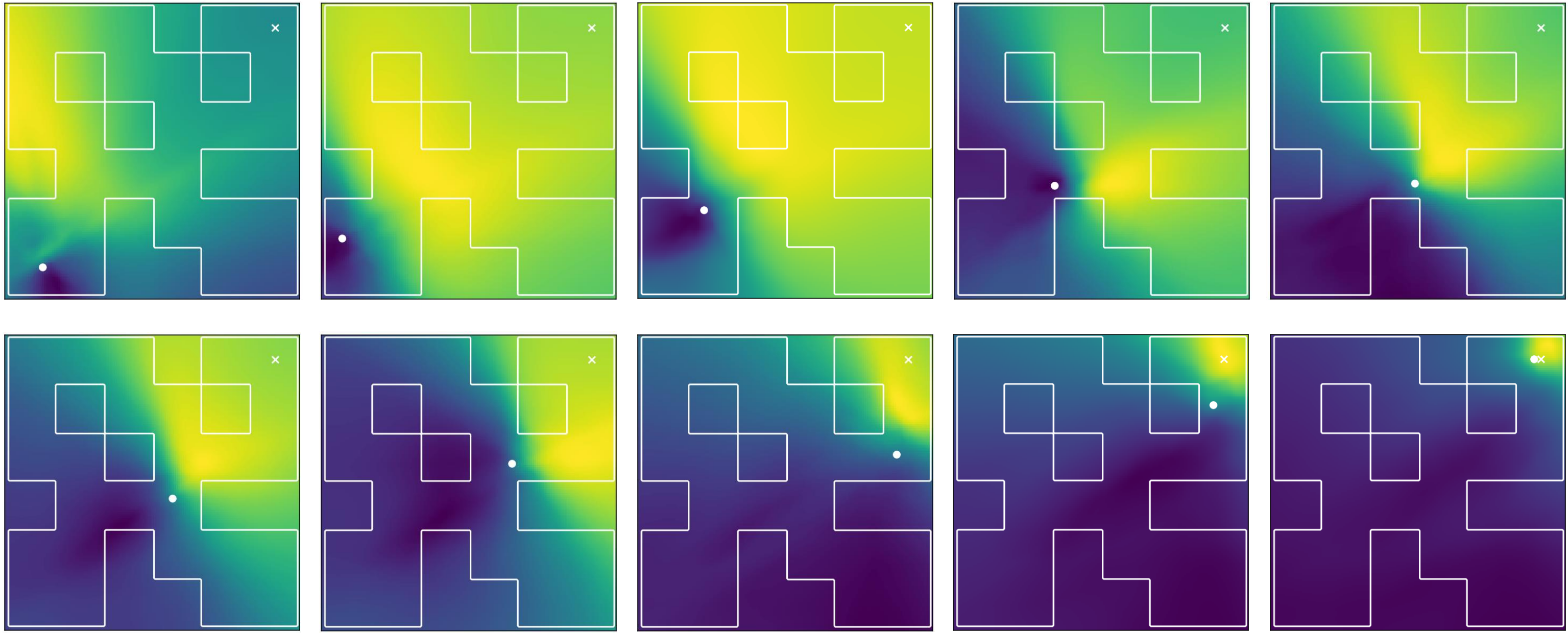}
        \caption{HIQL}
        \label{fig:traj_HIQL_GCRL_sub}
    \end{subfigure}

    \vspace{0.2cm}

    \begin{subfigure}[t]{\linewidth}
        \centering
        \includegraphics[width=\linewidth]{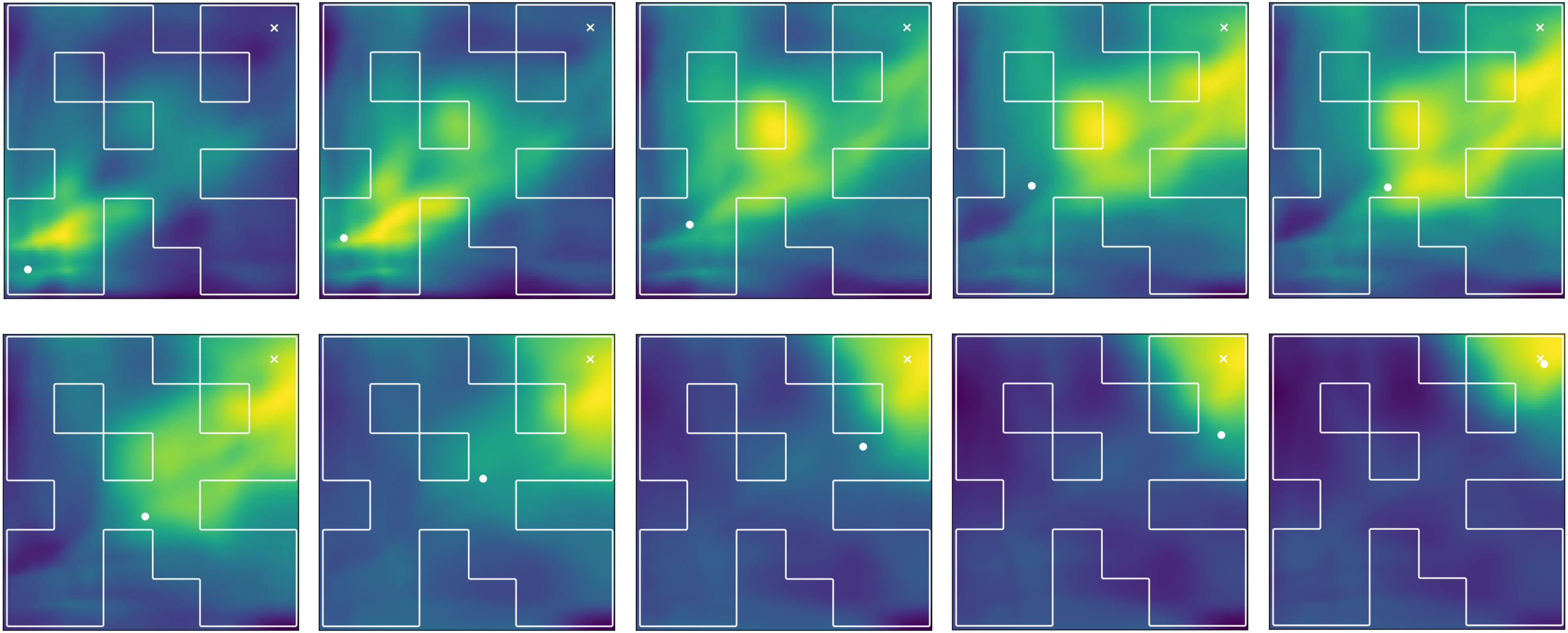}
        \caption{FB $\pi$-Switch}
        \label{fig:traj_HilSwitch_GCRL_sub}
    \end{subfigure}

    \vspace{0.2cm}

    \begin{subfigure}[t]{\linewidth}
        \centering
        \includegraphics[width=\linewidth]{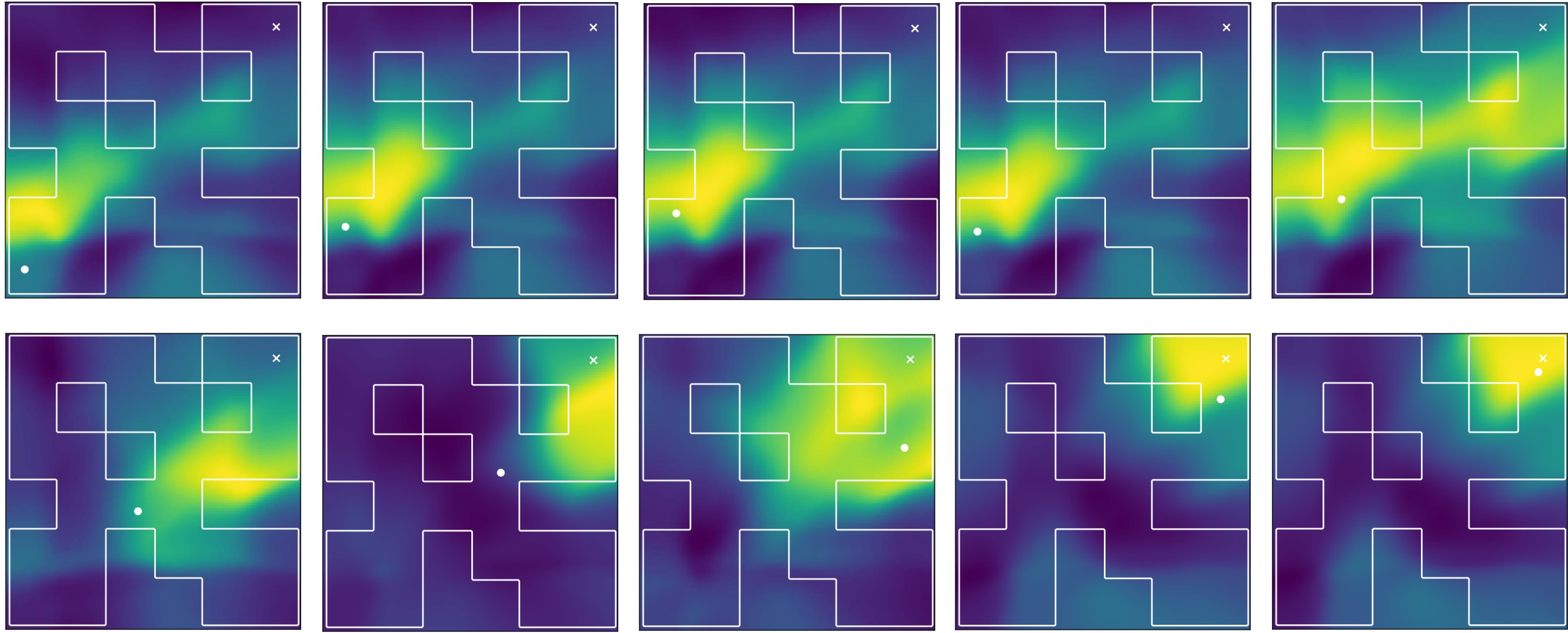}
        \caption{FB $+\pi^h$}
        \label{fig:traj_FB_pih_GCRL_sub}
    \end{subfigure}

    \caption{Evolution of high-level policies over a single trajectory for different methods in Medium-Antmaze environment. Within each trajectory, visualizations are temporally equally spaced. The dot indicates the agent’s current state, while the cross denotes the goal.}
    \label{fig:traj_all_methods_GCRL}
\end{figure}

\clearpage
\newpage

\begin{figure}[t]
    \centering
    \begin{subfigure}[t]{\linewidth}
        \centering
        \includegraphics[width=\linewidth]{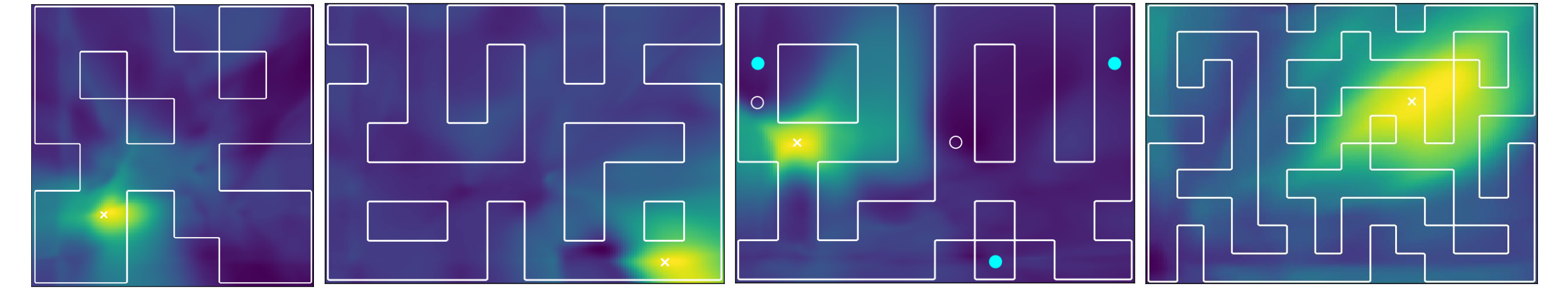}
        \caption{ICVF: state-value successor measures $M^{\pi_g}_{ \textcolor{magenta}{s}}(g)$}
    \end{subfigure}

    \vspace{4pt}

    \begin{subfigure}[t]{\linewidth}
        \centering
        \includegraphics[width=\linewidth]{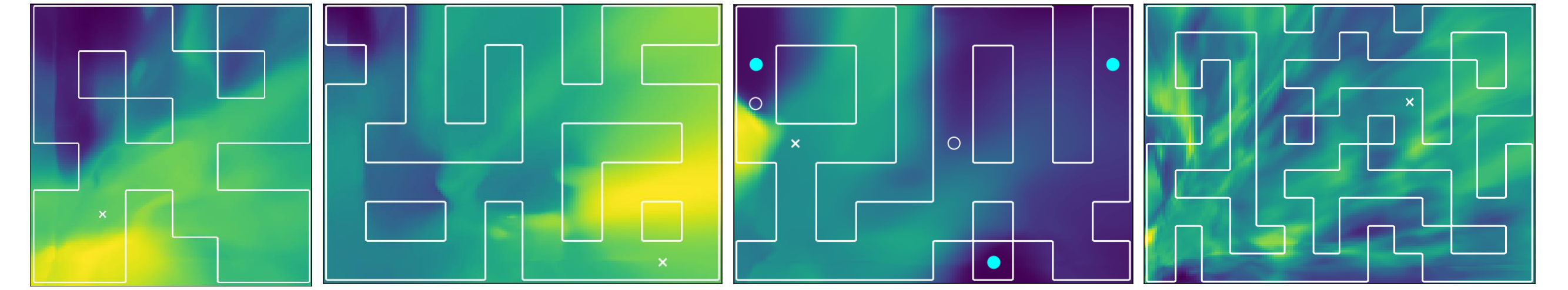}
        \caption{One-Step FB: state-value successor measures $M^{\pi_g}_{ \textcolor{magenta}{s}}(g)$}
    \end{subfigure}

    \vspace{4pt}

    \begin{subfigure}[t]{\linewidth}
        \centering
        \includegraphics[width=\linewidth]{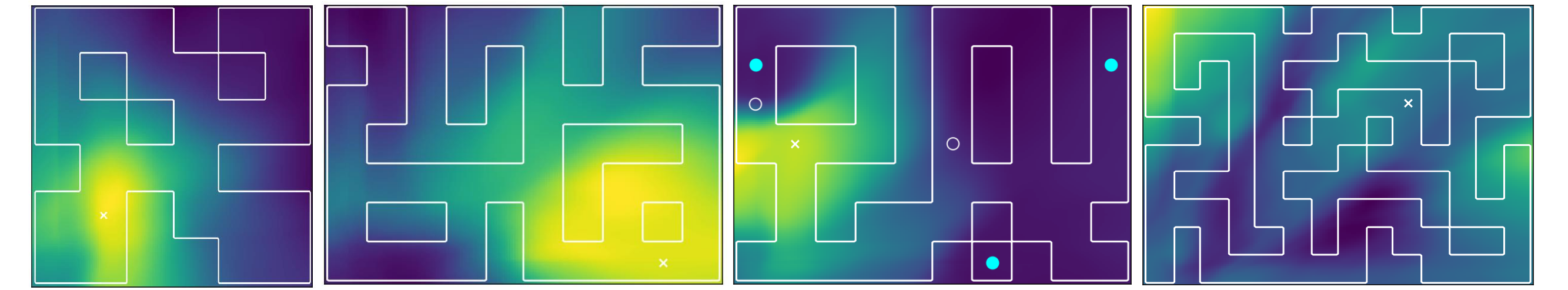}
        \caption{FB: state-value successor measures $M^{\pi_g}_{ \textcolor{magenta}{s}}(g)$}
    \end{subfigure}

    \vspace{4pt}

    \begin{subfigure}[t]{\linewidth}
        \centering
        \includegraphics[width=\linewidth]{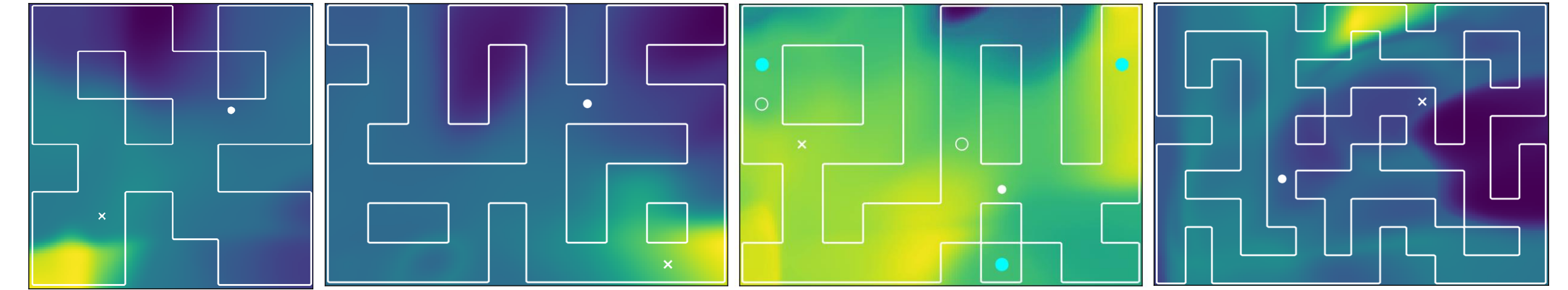}
        \caption{FB: intention-conditioned state-value successor measures $M^{\pi_{ \textcolor{magenta}{w}}}_s(g)$}
    \end{subfigure}

        \vspace{4pt}

    \begin{subfigure}[t]{\linewidth}
        \centering
        \includegraphics[width=\linewidth]{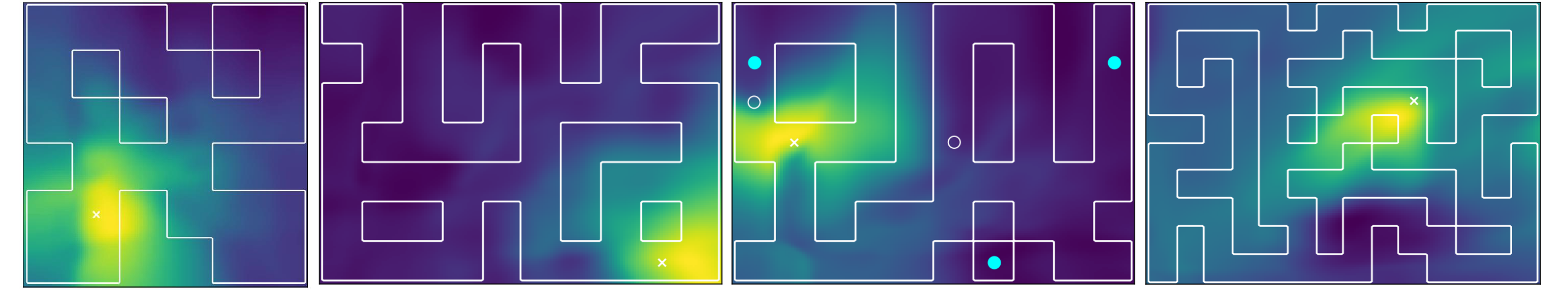}
        \caption{FB $\pi$-Switch: state-value successor measures $M^{\pi_g}_{ \textcolor{magenta}{s}}(g)$}
    \end{subfigure}

        \vspace{4pt}

    \begin{subfigure}[t]{\linewidth}
        \centering
        \includegraphics[width=\linewidth]{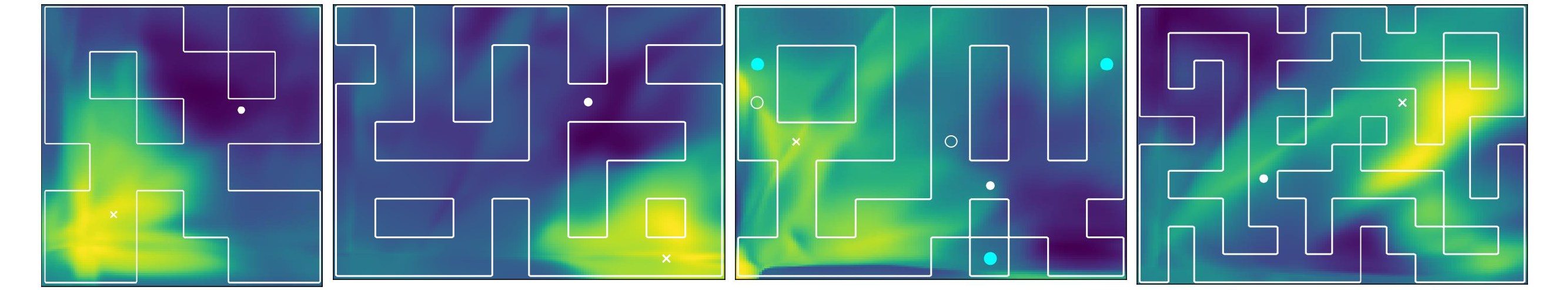}
        \caption{FB $\pi$-Switch: intention-conditioned state-value successor measures $M^{\pi_{ \textcolor{magenta}{w}}}_s(g)$}
    \end{subfigure}

    \caption{We report results for a single trained model per method (one seed), without aggregation over multiple seeds. All quantities are plotted as functions of the \textcolor{magenta}{magenta} variables. Fixed states $s$ are indicated by dots, while goals $g$ are indicated by crosses. Columns correspond to the Medium, Large, Teleport, and Giant AntMaze environments.}
    \label{fig:zero_shot_methods_GCRL_value_fcns_appendix}
\end{figure}

\begin{figure}[t]
    \centering
    \begin{subfigure}[t]{\linewidth}
        \centering
        \includegraphics[width=\linewidth]{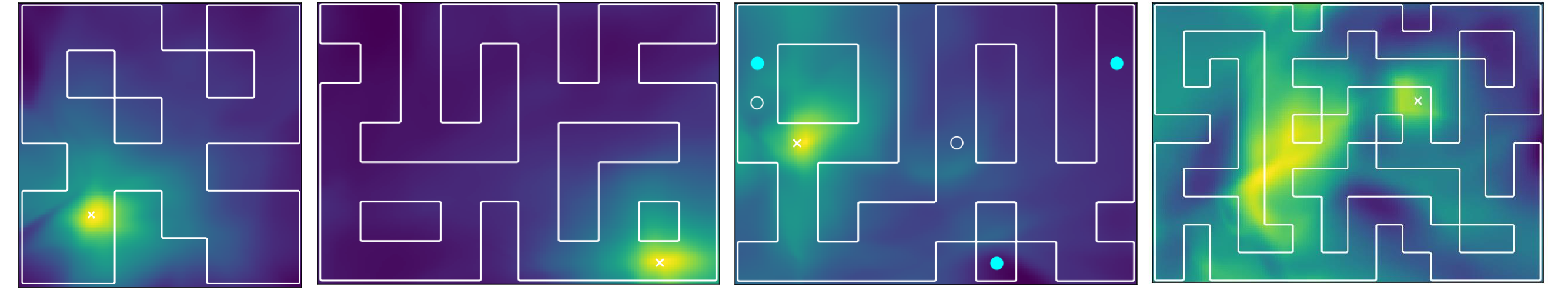}
        \caption{HIQL: state-value successor measures $M^{\pi_g}_{ \textcolor{magenta}{s}}(g)$}
    \end{subfigure}

    \vspace{4pt}

    \begin{subfigure}[t]{\linewidth}
        \centering
        \includegraphics[width=\linewidth]{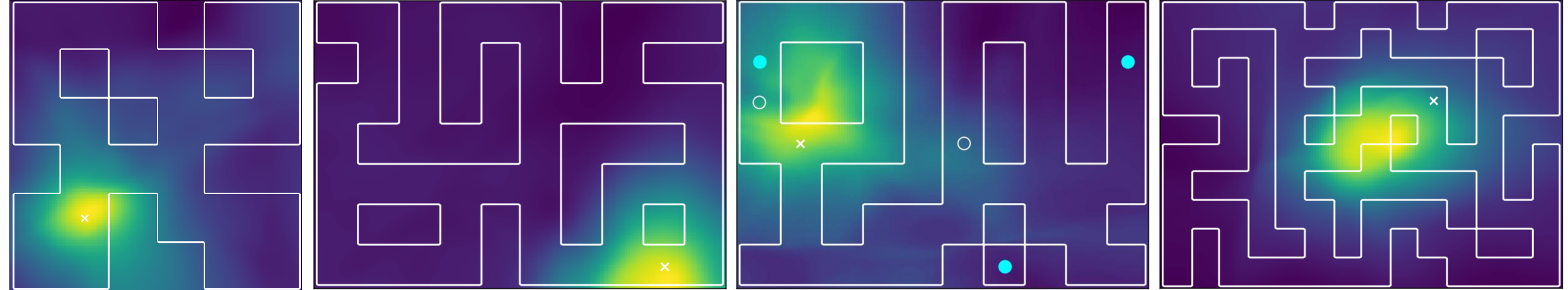}
        \caption{HIQL + bilinear: state-value successor measures $M^{\pi_g}_{ \textcolor{magenta}{s}}(g)$}
    \end{subfigure}

    \vspace{4pt}

    \begin{subfigure}[t]{\linewidth}
        \centering
        \includegraphics[width=\linewidth]{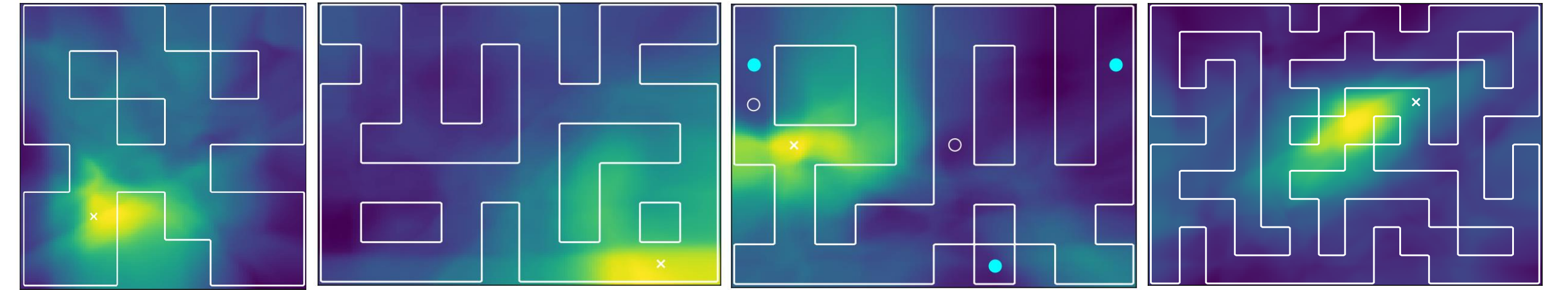}
        \caption{HIQL + bilinear + intentions: state-value successor measures $M^{\pi_g}_{ \textcolor{magenta}{s}}(g)$}
    \end{subfigure}

    \vspace{4pt}

    \begin{subfigure}[t]{\linewidth}
        \centering
        \includegraphics[width=\linewidth]{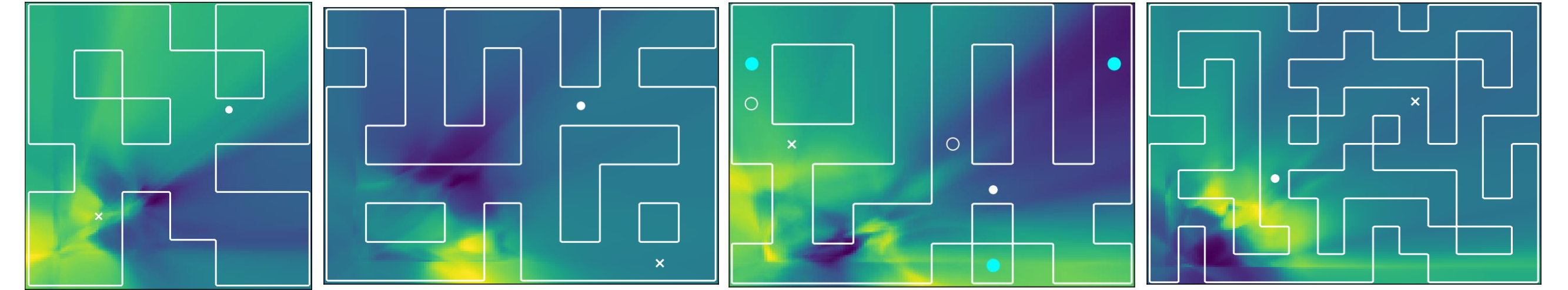}
        \caption{HIQL + bilinear + intentions: intention-conditioned state-value successor measures $M^{\pi_{ \textcolor{magenta}{w}}}_s(g)$}
    \end{subfigure}

    \caption{We report results for a single trained model per method (one seed), without aggregation over multiple seeds. All quantities are plotted as functions of the \textcolor{magenta}{magenta} variables. Fixed states $s$ are indicated by dots, while goals $g$ are indicated by crosses. Columns correspond to the Medium, Large, Teleport, and Giant AntMaze environments.}
    \label{fig:HIQL_variants_GCRL_value_fcns_appendix}
\end{figure}

\clearpage
\newpage
\subsection*{D.3 General rewards experiments.}
This section provides additional results for the continuous control tasks with general reward functions described in Section \ref{sec:experiments}. We first report the full results underlying Figure \ref{fig:regions_main}. For each of the five tasks (shown in Figure \ref{fig:rewards_antmaze}), we report the average undiscounted episodic return over episodes of length $1000$. 

It is important to emphasize that the tasks differ substantially in their reward structure. For example, task 3 in both environments contains no negative rewards, whereas tasks 5 include large negative regions combined sparse positive rewards. As a result, raw returns are not directly comparable across tasks. Instead, we normalize performance per task and aggregate using the interquartile mean (IQM), as described in Appendix C.5.  

The results are reported in Table \ref{tab:results_pointmaze}.

\begin{table}[ht]
    \centering
    \small
    \begin{tabular}{p{1.2cm}lrrrrrr}
    \toprule
    \textbf{Env.} & \textbf{Task} 
    & \textbf{ICVF} 
    & \textbf{One-Step FB} 
    & \multicolumn{2}{c}{\textbf{FB}} 
    & \multicolumn{2}{c}{\textbf{FB $\pi$-Switch}} \\
    \cmidrule(lr){5-6} \cmidrule(lr){7-8}
    & 
    & 
    & 
    & {vanilla} & {$+\pi^h$} 
    & {vanilla} & {$-\pi^h$} \\
    \midrule

    \multirow{6}{*}{\begin{tabular}{l} antmaze \\ large \end{tabular}}
    & task 1 
    & 282 \scriptsize $\pm 43 $
    & 207 \scriptsize $\pm 56$
    & 938 \scriptsize $\pm 308 $
    & 523 \scriptsize $\pm 95 $
    & 756 \scriptsize $\pm 185 $
    & 451 \scriptsize $\pm 99 $
    \\
    & task 2 
    & 360 \scriptsize $\pm 179 $
    & -171 \scriptsize $\pm 34 $
    & 4 \scriptsize $\pm 281 $
    & -45 \scriptsize $\pm 75$
    & 456 \scriptsize $\pm 305 $
    & 131 \scriptsize $\pm 107 $
    \\
    & task 3
    & 1183 \scriptsize $\pm 262 $
    & 192 \scriptsize $\pm 232$
    & 115 \scriptsize $\pm 107 $
    & 58 \scriptsize $\pm 75 $
    & 697 \scriptsize $\pm 354 $
    & 209 \scriptsize $\pm 133 $
    \\
    & task 4
    & 518 \scriptsize $\pm 206 $
    & 184 \scriptsize $\pm 352$
    & 906 \scriptsize $\pm 32 $
    & 916 \scriptsize $\pm 41 $
    & 664 \scriptsize $\pm 97 $
    & 602 \scriptsize $\pm 223 $
    \\
    & task 5
    & 109 \scriptsize $\pm 44 $
    & -79 \scriptsize $\pm 96 $
    & 36 \scriptsize $\pm 67 $
    & -5 \scriptsize $\pm 11 $
    & 479 \scriptsize $\pm 89$
    & 337\scriptsize $\pm 238 $
    \\
    & IQM score
    & 0.48				
    & 0.09
    & 0.37
    & 0.24	
    & 0.67
    & 0.35
    \\
    & LCI (95\%)
    & 0.42		
    & 0.03
    & 0.28	
    & 0.21	
    & 0.56
    & 0.30	
    \\
    & UCI (95\%)
    & 0.57				
    & 0.15	
    & 0.50
    & 0.27
    & 0.74
    & 0.43
    \\
    
    \midrule

    \multirow{6}{*}{\begin{tabular}{l} antmaze \\ giant \end{tabular}}
    & task 1 
    & -19 \scriptsize $\pm 28$
    & 36 \scriptsize $\pm 30$
    & 1 \scriptsize $\pm 55$
    & 2 \scriptsize $\pm 12$
    & 101 \scriptsize $\pm 87$
    & 125 \scriptsize $\pm 167$
    \\
    & task 2 
    & 266 \scriptsize $\pm 50$
    & 30 \scriptsize $\pm 30$
    & 158 \scriptsize $\pm 128$
    & 169 \scriptsize $\pm 117$
    & 292 \scriptsize $\pm 92$
    & 326 \scriptsize $\pm 60$
    \\
    & task 3 
    & 185 \scriptsize $\pm 76$
    & 173 \scriptsize $\pm 107$
    & 292 \scriptsize $\pm 175$
    & 292 \scriptsize $\pm 170$
    & 470 \scriptsize $\pm 268$
    & 437 \scriptsize $\pm 187$
    \\
    & task 4 
    & 192 \scriptsize $\pm 65$
    & -176 \scriptsize $\pm 152$
    & 138 \scriptsize $\pm 158$
    & -3 \scriptsize $\pm 134$
    & 321 \scriptsize $\pm 61$
    & 362 \scriptsize $\pm 99$
    \\
    & task 5 
    & -40 \scriptsize $\pm 80$
    & -292 \scriptsize $\pm 138$
    & -391 \scriptsize $\pm 283$
    & -369 \scriptsize $\pm 210$
    & 257 \scriptsize $\pm 305$
    & 42 \scriptsize $\pm 150$
    \\
    & IQM score
    & 0.43		
    & 0.21	
    & 0.33	
    & 0.30
    & 0.65
    & 0.63	
    \\
    & LCI (95\%)
    & 0.39		
    & 0.17	
    & 0.23	
    & 0.24
    & 0.54
    & 0.54	
    \\
    & UCI (95\%)
    & 0.47		
    & 0.28	
    & 0.43	
    & 0.37	
    & 0.74
    & 0.73
    \\
    \bottomrule\\
    \end{tabular}
    \caption{Averaged undiscounted returns across 50 evaluations and five test-time rewards from Figure \ref{fig:rewards_antmaze} averaged over 5 seeds. We report standard deviations after the $\pm$ sign. We also report interquartile mean scores (IQM) and lower and upper 95\% confidence intervals - see Appendix C.5 for more details on the evaluation protocol.}
    \label{tab:results_pointmaze}
\end{table}

Overall, the results suggest that while several methods achieve strong performance on individual tasks, FB $\pi$-Switch is more consistent across tasks. In Antmaze-Large, the lower 95\% confidence interval of FB $\pi$-Switch is close to or exceeds the upper confidence interval of the best performing agent. A similar trend is observed in Antmaze-Giant, where even the non-hierarchical variant of FB $\pi$-Switch remains competitive.   

To better understand these differences, we further decompose returns into two components: (i) return accumulated before reaching the highest-rewarding state for the first time, and (ii) return accumulated after reaching it. These results are shown in Figure \ref{fig:full_results_regions}. This decomposition allows us to assess whether agents follow reward-aligned paths and how efficiently they reach the highest-rewarding state.

For Antmaze-Large, FB $\pi$-Switch improves over its non-hierarchical variant primarily by reaching the highest-reward state faster, while maintaining similar reward along the way. An interesting observation in Antmaze-Giant is that the non-hierarchical FB $\pi$-Switch achieves higher average performance on task 1 by reaching the final state more quickly, but does so at the cost of traversing regions with negative reward more frequently. This highlights the importance of reward design when evaluating behavioral quality beyond aggregate return.

\begin{figure*}[t]
    \centering
    \includegraphics[width=\textwidth]{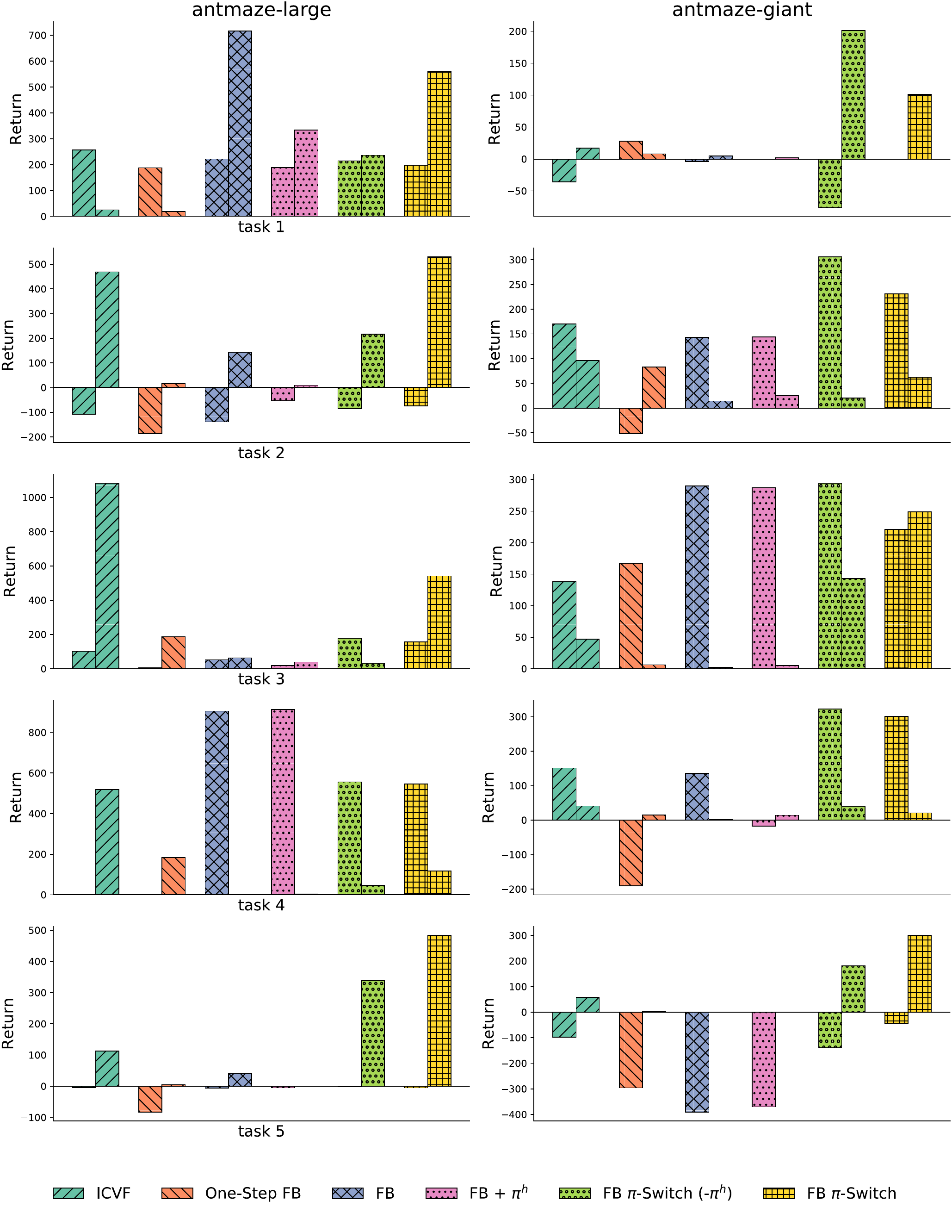}
    \caption{Average episodic return before (left bars) and after (right bars) reaching the highest-reward state for the first time. Results are shown separately for AntMaze-Large (left) and AntMaze-Giant (right).}
    \label{fig:full_results_regions}
\end{figure*}